\documentclass[a4paper,12pt,reqno]{amsart}
\usepackage{hyperref}
\hypersetup{backref,colorlinks=true,allcolors=black}
\usepackage{mathrsfs}
\usepackage{enumitem}
\usepackage{dsfont}
\usepackage{mathtools}
\usepackage{xcolor}
\usepackage[final]{graphicx}
\usepackage{upgreek}
\usepackage{amsmath,amssymb,bbm}
\usepackage{nameref}
\usepackage{comment}  
\usepackage{subfig}
\usepackage{booktabs}
\usepackage{multirow}

\usepackage{dsfont}

\usepackage{float}

\usepackage{cancel}

\newtheorem{theorem}{Theorem}[section]
\newtheorem{proposition}[theorem]{Proposition}
\newtheorem{lemma}[theorem]{Lemma}
\newtheorem{corollary}[theorem]{Corollary}

\theoremstyle{definition}

\newtheorem{example}[theorem]{Example}
\newtheorem{remark}[theorem]{Remark}

\newtheorem{innercustomgeneric}{\customgenericname}
\providecommand{\customgenericname}{}
\newcommand{\newcustomtheorem}[2]{%
  \newenvironment{#1}[1]
  {%
   \renewcommand\customgenericname{#2}%
   \renewcommand\theinnercustomgeneric{##1}%
   \innercustomgeneric
  }
  {\endinnercustomgeneric}
}

\newcustomtheorem{customthm}{Theorem}
\newcustomtheorem{customlemma}{Lemma}
\newcustomtheorem{customassump}{Condition}
\newcustomtheorem{customproof}{Proof}

\newcommand{\ignore}[1]{}
\newcommand{\R}{\mathbb{R}}
\newcommand{\N}{\mathbb{N}}

\newcommand{\norm}[1]{\left\lVert #1 \right\rVert}
\newcommand{\abs}[1]{\left\vert #1 \right\rvert}
\newcommand{\E}[1]{\mathbb{E}{\left[ #1\right]}}

\newcommand{\mb}[1]{\mathbf{#1}}
\newcommand{\bs}[1]{\boldsymbol{#1}}

\newcommand{\opnorm}[1]{\norm{#1}_{\operatorname{op}}}

\DeclarePairedDelimiter\autobracket{(}{)}
\newcommand{\brac}[1]{\autobracket*{#1}}
\newcommand{\inner}[1]{\left\langle #1 \right\rangle}

\theoremstyle{definition}

\makeatletter
\@addtoreset{proofcase}{theorem}
\makeatother

\allowdisplaybreaks

\theoremstyle{definition}

\makeatletter
\@addtoreset{proofstep}{theorem}
\makeatother

\usepackage{microtype}
\usepackage{parskip}

\definecolor{crimson}{rgb}{0.7294,0.0666,0.0470}
\title{High-Order Langevin Monte Carlo Algorithms}
\author{Thanh Dang, Mert G\"{u}rb\"{u}zbalaban, Mohammad Rafiqul Islam, Nian Yao and Lingjiong Zhu} 
\date{August 24, 2025}

\begin{document}

\begin{abstract}
Langevin algorithms are popular Markov chain Monte Carlo (MCMC) methods
for large-scale sampling problems that often arise in data science. 
We propose Monte Carlo algorithms based
on the discretizations of $P$-th order Langevin dynamics for any $P\geq 3$. Our design of $P$-th order Langevin Monte Carlo (LMC) algorithms is by combining splitting and accurate integration methods. 
We obtain Wasserstein convergence guarantees for sampling from distributions with log-concave and smooth densities. Specifically, the mixing time of the $P$-th order LMC algorithm scales as $O\brac{d^{\frac{1}{\mathcal{R}}}/\epsilon^{\frac{1}{2\mathcal{R}}} }$ for $ \mathcal{R}=4\cdot\mathds{1}_{\{ P=3\}}+ (2P-1)\cdot\mathds{1}_{\{ P\geq 4\}}$, which has a better dependence on the dimension $d$ and the accuracy level $\epsilon$ as $P$ grows. Numerical experiments illustrate the efficiency of
our proposed algorithms.
\end{abstract}

\maketitle

\section{Introduction}\label{sec:intro}

\textit{Langevin algorithms} are popular Markov chain Monte Carlo (MCMC) methods to sample from a given density $\mu(\theta)\propto e^{-U(\theta)}$ of interest
where $\theta\in\mathbb{R}^{d}$, and these sampling problems appear
in many applications such as Bayesian statistical inference, Bayesian formulations of inverse problems, and Bayesian classification and regression tasks in machine learning \cite{gelman1995bayesian,stuart2010inverse,andrieu2003introduction,teh2016consistency,DistMCMC19,GIWZ2024}.
The classical Langevin Monte Carlo algorithm is based on the discretization of
{\it overdamped (or first-order) Langevin dynamics} \cite{Dalalyan,DM2017,DK2017,Raginsky,Barkhagen2021,Chau2019,EH2021,Zhang2019,BCESZ2022} that follows
the stochastic differential equation (SDE):
\begin{equation}\label{eq:overdamped-2}
d\theta_{t}=-\nabla U(\theta_{t})dt+\sqrt{2}dB_{t},
\end{equation}
where $U:\mathbb{R}^{d}\rightarrow\mathbb{R}$ is often known as the \textit{potential function},
and $B_{t}$ is a standard $d$-dimensional Brownian motion with $\theta_{0}\in\mathbb{R}^{d}$. Under some mild assumptions on $U(\cdot)$, the diffusion \eqref{eq:overdamped-2} admits a unique stationary distribution with the density $\mu(\theta) \propto e^{-U(\theta)}$,
also known as the \emph{Gibbs distribution} \cite{chiang1987diffusion,stroock-langevin-spectrum}. In computing practice, this diffusion is simulated by considering its discretization,
and one of the most commonly used discretization schemes is the Euler–Maruyama discretization of
\eqref{eq:overdamped-2}, often known as the \textit{unadjusted Langevin algorithm} in the literature; see e.g. \cite{DM2017}:
\begin{equation}\label{discrete:overdamped}
\theta_{k+1}=\theta_{k}-\eta\nabla U(\theta_{k})+\sqrt{2\eta}\xi_{k+1},
\end{equation}
where $\xi_{k}$ are i.i.d. $\mathcal{N}(0,I_{d})$ Gaussian vectors.

In a seminal paper, \cite{Dalalyan} obtained the first non-asymptotic result of the discretized Langevin dynamics \eqref{discrete:overdamped}; later, \cite{DM2017} improved the dependence on the dimension $d$.
Both works consider the total variation (TV) as the distance to measure the convergence.
In contrast, \cite{DM2016} studied the convergence in the 2-Wasserstein distance,
and \cite{DMP2016} studied variants of \eqref{discrete:overdamped} when $U$ is not smooth.
\cite{CB2018} studied the convergence in the Kullback-Leibler distance.
\cite{EHZ2022} obtained the convergence 
in chi-squared and R\'{e}nyi divergence.
\cite{DK2017,Raginsky,Barkhagen2021,Chau2019,Zhang2019} studied the convergence when only stochastic gradients are available.

In the literature, many variants
of the overdamped Langevin dynamics and the discretization schemes have been studied.
One popular Langevin dynamics is the {\it underdamped Langevin dynamics}, also known as the \textit{second-order} or kinetic Langevin dynamics, see e.g.
\cite{mattingly2002ergodicity,Villani2009,cheng2018underdamped,cheng-nonconvex,CLW2020,JianfengLu,dalalyan2018kinetic,GGZ2,Ma2019,GGZ}:
\begin{equation}\label{eqn:underdamped}
\begin{cases}
dr_{t}=-\gamma r_{t}dt-\nabla U(\theta_{t})dt+\sqrt{2\gamma}dB_{t},\\
d\theta_{t}=r_{t}dt,
\end{cases}
\end{equation}
where $B_{t}$ is a standard $d$-dimensional Brownian motion
with $r_{0},\theta_{0}\in\mathbb{R}^{d}$.
Under some mild assumptions on $U$, the SDE \eqref{eqn:underdamped} admits a unique stationary distribution with the density $\mu(\theta,r) \propto e^{-U(\theta)-\frac{1}{2}|r|^{2}}$ \cite{Eberle}, whose $\theta$-marginal distribution coincides 
with the stationary distribution of \eqref{eq:overdamped-2}.
It is known that the second-order (underdamped) Langevin dynamics \eqref{eqn:underdamped} might converge to the Gibbs distribution faster than the first-order (overdamped) Langevin dynamics \cite{Eberle,JianfengLu}, 
and the discretization based on the second-order Langevin dynamics might have better iteration complexity, 
in particular, with a better dependence on the dimension $d$
and the accuracy level $\epsilon$ \cite{cheng2018underdamped,GGZ}.

In the recent literature, higher-order, in particular, 
the \textit{third-order} Langevin dynamics
and its discretization have been proposed and studied in \cite{mou2021high}:
\begin{equation}\label{JI}
\begin{cases}
d\theta_t=p_t\ dt,\\
dp_t=-\frac{1}{L} U(\theta_t)\ dt + \gamma r_t\ dt,\\
dr_t=-\gamma p_t\ dt - 2\gamma r_t\ dt + \sqrt{\frac{4\gamma}{L}}\ dB_t,
\end{cases}
\end{equation}
where $\gamma>0$ is the friction parameter, $L$ is the smoothness parameter of $U$ and $B_t$ is a standard Brownian motion in $\mathbb{R}^d$.
Under some mild assumptions on $U$, the SDE \eqref{JI} admits a unique stationary distribution with the density $\mu(\theta,r) \propto e^{-U(\theta)-\frac{L}{2}|p|^{2}-\frac{L}{2}|r|^{2}}$ \cite{mou2021high}.
They showed that a Langevin Monte Carlo algorithm based on the discretization of the third-order Langevin SDE \eqref{JI} can have
even better iteration complexity in terms of dependence on the dimension $d$ and the accuracy
level $\epsilon$ compared to the algorithm based on the second-order Langevin SDE \eqref{eqn:underdamped} \cite{mou2021high}.

It is thus very natural to ask if one can propose and study
a more general $P$-th order Langevin dynamics, and whether
its discretization can lead to better iteration complexity.
In the very recent probability literature, a generalized Langevin dynamics is studied in \cite{monmarche2023almost}. Their result is in continuous time only. The focus of our paper is to propose and study
the iteration complexity of an algorithm based on the discretization of the continuous-time $P$-th order Langevin dynamics, which we name $P$-th order Langevin Monte Carlo (LMC) algorithm. In the context of log-concave sampling via the Langevin equation and its variants, Table~\ref{table:1} compares the mixing time of our $P$-th order LMC algorithm in Theorem~\ref{theorem_mixingtime_Pthorder} with the mixing time of other algorithms from the references in the literature\footnote{For comparison purpose, we focus solely on the dependence on $d$ and $\epsilon$ of the rates in the cited references in the table. These references improve other aspects of log-concave sampling which we do not cover here.}. Also note that the convex-smooth condition is our Condition~\ref{cond_mainpaper}. 

\begin{table}
\begin{center}
\begin{tabular}{ |c|c|c| } 
\hline
References & Assumptions on potential $U$ & Mixing time in $\mathrm{Wass}_2$ \\
\hline
\cite{cheng2018underdamped,dalalyan2018kinetic,Ma2019} & convex-smooth & $O\brac{\frac{d^{1/2}}{\epsilon}}$\\
\hline
\cite{shen2019randomized} & convex-smooth & $O\brac{\frac{d^{1/3}}{\epsilon^{2/3}}}$ \\
\hline
\cite{mou2021high} & ridge-separable, convex-smooth& $O\brac{\frac{d^{1/4}}{\epsilon^{1/2}}}$ \\
\hline
\multirow{2}{4em}{\cite{mou2021high}} & strongly convex & $O\brac{\frac{d^{1/4}}{\epsilon^{1/2}}}$ \\ 
& and smooth up to order $\alpha$  &  $+O\brac{\frac{d^{1/2}}{\epsilon^{1/\alpha-1}}}$ \\ 
\hline
\multirow{3}{4em}{Our Theorem \ref{theorem_mixingtime_Pthorder}} & convex-smooth& $O\brac{d^{\frac{1}{\mathcal{R}}}/\epsilon^{\frac{1}{2\mathcal{R}}} }$, \text{where } \\ 
& and Condition~\ref{cond_derivativegrowthrate}  &  $ \mathcal{R}=4\cdot\mathds{1}_{\{ P=3\}}$ \\
& & $\qquad+ (2P-1)\cdot\mathds{1}_{\{ P\geq 4\}}$\\
\hline
\end{tabular}
\end{center}
\caption{Summary of assumptions and iteration complexities in our paper compared with the literature.}\label{table:1}
\end{table}

Our contributions can be summarized as follows.

\begin{itemize}
\item We construct $P$-th order LMC algorithms that are based on discretizations of $P$-th order Langevin dynamics for $P\geq 3$. Under the condition that the potential function $U$ is convex, sufficiently smooth and the operator norm of the derivatives of $U$ do not grow too quickly, we show that the iteration complexity of our $P$-th order LMC algorithm scales as $O\brac{d^{\frac{1}{\mathcal{R}}}/\epsilon^{\frac{1}{2\mathcal{R}}} }$ for $ \mathcal{R}=4\cdot\mathds{1}_{\{ P=3\}}+ (2P-1)\cdot\mathds{1}_{\{ P\geq 4\}}$. Our iteration complexity result therefore has a better dependence on the dimension $d$ and the accuracy level $\epsilon$ as $P$ grows. We therefore provide a positive answer to a conjecture in \cite[Section 5]{mou2021high} that one can construct LMC algorithms based on high-order Langevin dynamics that reduce the dependence of the iteration complexity on the dimension and the accuracy level.

\item Inspired by existing work on second- and third-order Langevin Monte Carlo algorithms \cite{cheng2018underdamped,mou2021high}, we propose and rigorously study novel discretization schemes for high-order Langevin dynamics that contain several stages of refinement, and each stage adopts a splitting scheme to ensure that the conditional expectation of the vector formed by the variables in each stage (conditioned on the last stage) follows a multivariate normal distribution. A natural question is what the maximum number of refinement stages one can design, which affects how much improvement one can obtain in iteration complexity. We discover that the maximum number of stages is $P-1$ (see Remark~\ref{remark_nomorestages} and Remark~\ref{remark_Pthorder}). 

\item We perform numerical experiments and compare the performance of the third- and fourth-order LMC algorithms. In particular, we study sampling from the posterior distribution of the model parameters in Bayesian regression using real data, where the loss function is quadratic. Our numerical results show better performance for the fourth-order LMC algorithm. 
In addition, we consider a sigmoid loss function for sampling from the posterior distribution of the model parameters in Bayesian classification problems, using real data, which demonstrates the efficiency of our proposed algorithm. 
\end{itemize}

The rest of the paper can be summarized as follows.
In Section~\ref{section_frommonmarche}, main results are stated. We first provide
some preliminaries for the continuous-time $P$-th order Langevin dynamics and state our main assumptions. 
In Section~\ref{section_4thorder}, for pedagogical purpose, we introduce and study the fourth-order LMC algorithm, and then in Section~\ref{section_Pthorder}, we extend our results to any $P$-th order LMC algorithm
for $P\geq 3$. We conduct numerical experiments
to show the efficiency of our algorithms in Section~\ref{sec:numerical}. In Section~\ref{sec:conclusion}, we conclude. 
Further technical details will be provided
in the Appendix.


\section{Main Results}
\label{section_frommonmarche}

In this section, we first present an important result regarding convergence toward equilibrium of (continuous-time) $P$-th order Langevin dynamics, that is established by Monmarch\'{e} in \cite{monmarche2023almost}. Let us start with some definitions.

Let $P,d\geq 1$. A $P$-th order Langevin dynamics has the form 
\begin{align}
\label{Pthorderlangevinequation_mainpaper}
dX_t &=  A Y_t d t, \nonumber\\
dY_t &=  - A^{\top}\nabla U(X_t) d t - \gamma B Y_t dt + \sqrt{\gamma}d W_t,
\end{align}
where $W$ is a standard $(P-1)d$-dimensional Brownian motion; $U\in\mathcal C^2(\R^d)$; while the $d\times (P-1)d$ matrix $A$ and the $(P-1)d\times (P-1)d$ matrix $B$ are given by:
\begin{equation*}
A = \begin{pmatrix}
I_d & 0 & \dots   & 0
\end{pmatrix}\qquad \text{and}\qquad
B = \begin{pmatrix}
0  & -I_d & 0 & \dots & 0 \\
 I_d & 0 & -I_d & \ddots & \vdots \\
0 & \ddots & \ddots & \ddots &   0 \\
\vdots & \ddots &   I_d & 0 & - I_d\\
0 &\dots  & 0 &  I_d &   I_d
\end{pmatrix}\,.
\end{equation*}

Regarding notations, set $b$ as the drift coefficient of \eqref{Pthorderlangevinequation_mainpaper}, that is 
\begin{align}
\label{def_driftb}
    b(x,y)=\begin{pmatrix}
        Ay\\-A^{\top}\nabla U(x)-\gamma By
    \end{pmatrix},
\end{align}
and denote $J_b$ as its Jacobian matrix. 

Next, we set
\begin{align*}
     \hat{\lambda}=\min \{Re(\lambda), \text{$\lambda$ is an eigenvalue of $B_{\operatorname{sim}}$}\},\,\,\text{ where }  B_{\operatorname{sim}} :=\begin{pmatrix}
0  & -1 & 0 & \dots & 0 \\
 1 & 0 & -1 & \ddots & \vdots \\
0 & \ddots & \ddots & \ddots &   0 \\
\vdots & \ddots &   1 & 0 & - 1\\
0 &\dots  & 0 &  1 &   1
\end{pmatrix},
\end{align*}
noting that $B=B_{\operatorname{sim}}\otimes I_d$ where $\otimes$ denotes the Kronecker product. 
Also, 
\begin{align*}
  \kappa=  \begin{cases} 
      \hat{\lambda}, & \text{when $B$ is diagonalizable}, \\
      \hat{\lambda}-\epsilon, & \epsilon\in \brac{0,\hat{\lambda}} \text{ when $B$ is not diagonalizable.}
   \end{cases}
\end{align*}
We show in the proof of Corollary~\ref{coro_monmarcheforPthorderlangevin} that $\hat{\lambda}>0$, which implies $\kappa>0$.

Denote the Jordan blocks of $B_{\operatorname{sim}}$ by $B_n,1\leq n\leq N$. Each block $B_n$ of length $\ell_n$ is associated with the eigenvalue $\lambda_n$ and the set of generalized eigenvectors $v_n^{(k)}; 1\leq k\leq \ell_n$. In particular, $v_n^{(1)}$ is the (standard) eigenvector of $B_n$. Notice here we slightly abuse notations as we are using the same notations $B_n, \ell_n,v_n^{(k)}$ for the matrix $B$ in Appendix~\ref{appendix_frommonmarche}. For a Jordan block $B_n$ with $\mathrm{Re}(\lambda_n)>\kappa$, we set 
\begin{align*}
H_n=\sum_{i=1}^{\ell_n} b^i_n v_n^{(i)}\brac{{\bar{v}_n^{(i)}}}^{\top},
\end{align*}
where $\bar{v}^{\top}$ denotes the conjugate transpose of a vector $v$ and 
\begin{align}
\label{def_coefficients}
    &b_{n}^1=1;\quad b^j_n=c_j (t_n)^{2(1-j)},\quad 2\leq j\leq \ell_n;\\
    &c_1=1;\quad c_{j+1}=1+c_{j}^2,\quad 2\leq j\leq \ell_n;\qquad t_n=2\brac{\mathrm{Re}(\lambda_n)-\hat{\lambda}}\nonumber. 
\end{align}
Meanwhile, for a Jordan block $B_m$ with $\mathrm{Re}(\lambda_m)=\kappa$, we define 
\begin{align*}
\widetilde{H}_m(\epsilon)=\sum_{i=1}^{\ell_m} b^i_m(\epsilon) v_m^{(i)}\brac{{\bar{v}_m^{(i)}}}^{\top},
\end{align*}
where $b^i_m$'s are the same as in \eqref{def_coefficients}, except that we replace the above $t_n$ with $t_m=2(\mathrm{Re}(\lambda_m)-\hat{\lambda}+\epsilon)$ for any $\epsilon\in (0,\hat{\lambda})$ and write $b^i_m=b^i_m(\epsilon)$ to emphasize the dependence on $\epsilon$. Now, assume $I=\{n\in\{1,\cdots,N \}:\ell_n\geq 2, \mathrm{Re}(\lambda_n)=\hat{\lambda} \}$ and set 
\begin{align*}
    H(\epsilon)=\sum_{n\in \{1,\ldots,N\}\setminus I } H_n+\sum_{m\in I}\widetilde{H}_m(\epsilon).
\end{align*}

Next, we define
\begin{align*}
h_1&=\opnorm{H(\epsilon)\begin{pmatrix}
         1 & \cdots &0\\
         \vdots&\ddots&\vdots\\
         0&\cdots&0
     \end{pmatrix} }, \qquad\qquad h_2=h_3=1,\\
       h_4&= \brac{1+\frac{P-1}{2}}\opnorm{H(\epsilon)^{-1}}, \quad h_5=(1+P)\opnorm{H(\epsilon)^{-1}}.
\end{align*}

Regarding the potential function $U$, we assume that 
\begin{customassump}{H1}~
		\label{cond_mainpaper}
$U$ is $m$-strongly convex and $L$-smooth: $m I_d\leq \nabla^2 U(x)\leq L I_d$ for any $x\in \R^d$.
\end{customassump}

The following important result is established by Monmarch\'{e} in \cite{monmarche2023almost}. Further details are provided in Theorem~\ref{theorem_monmarche_appendix} and Corollary~\ref{coro_monmarcheforPthorderlangevin} in Appendix~\ref{appendix_frommonmarche}.

\begin{theorem}
\label{theorem_frommonmarche_mainpaper}
(a shortened version of Theorem~\ref{theorem_monmarche_appendix} and Corollary~\ref{coro_monmarcheforPthorderlangevin}) Assume the setup above for the $P$-th order Langevin dynamics~\eqref{Pthorderlangevinequation_mainpaper}, including Condition~\ref{cond_mainpaper} on the potential function $U$. If the friction $\gamma$ is sufficiently large: 
\begin{align*}
  \gamma\geq   \gamma_0:=2\sqrt{\frac{h_1L}{\kappa}}\max \left\{ \sqrt{h_2h_5},\sqrt{\frac{h_4}{\kappa}} \right\},
\end{align*}
then the $Pd\times Pd$ matrix  $M:=\begin{pmatrix}
 1 & \frac{1}{\gamma}\begin{pmatrix}
    1 &\ldots &1
\end{pmatrix}\\
\frac{1}{\gamma}\begin{pmatrix}
    1 &\ldots &1
\end{pmatrix}^{\top} & \frac{\kappa}{Lh_1} H(\epsilon)
\end{pmatrix}\otimes I_d$ is symmetric, positive definite and satisfies
\begin{align}
\label{contraction}
    M J_b+J_b^{\top} M\leq -2\rho M, \qquad\quad \rho=\min \left\{ \frac{m}{3h_3\gamma},\frac{\gamma \kappa}{6} \right\}. 
\end{align} 
In particular, $\rho=\rho(\gamma,L,P)$ and $\gamma_0=\gamma_0(\gamma,L,P)$ depend on $\gamma,L,P$ but not on the dimension parameter $d$. Furthermore, $\lambda_{\min,M}=\lambda_{\min,M}(P)$ and $\lambda_{\max,M}=\lambda_{\max,M}(P)$ are respectively the smallest and largest eigenvalues of the positive definite matrix $M$, and they depend on $P$ but not on $d$. 
\end{theorem}

\begin{remark}
    Condition~\ref{cond_mainpaper} is exactly Condition~\ref{cond_frommonmarche} in Appendix~\ref{appendix_frommonmarche} specified for the $P$-th order Langevin dynamics~\eqref{Pthorderlangevinequation_mainpaper}.  
\end{remark}

\begin{example}
\label{example_P4}
Here we demonstrate how to find $\gamma_0$ and $M$ in Theorem~\ref{theorem_frommonmarche_mainpaper} in the case $P=4$. The matrix $B=\begin{pmatrix}
         0 &-1 &0\\
         1&0&-1\\
         0&1&1
     \end{pmatrix}$ is diagonalizable. It has eigenvector $v_1\approx (-0.877-0.745i,-0.785+1.307i,1)$ corresponding to eigenvalue $0.215+1.307i$, $v_2\approx (-0.877+0.745i,-0.785-1.307i,1)$ corresponding to eigenvalue $0.215-1.307i$ and $v_3\approx(0.755,-0.430,1)$ corresponding to eigenvalue $0.570$. Then the matrix $H$ is approximately $\begin{pmatrix}
         3.341 &-1.004 &-0.999\\
         -1.004&4.850&-2.000\\
         -0.999&-2.000&3.000
     \end{pmatrix}$ with eigenvalues (approximately) $ 6.168,4.098$ and $0.924$. From there, one deduces $h_1\approx 3.341, h_4=2.705, h_5=5.410,\kappa=0.924$ and we know beforehand that $h_2=h_3=1$. Thus, we have $\gamma_0\approx 2\sqrt{\frac{3.341 L}{0.215}}\cdot (3.547)$ and 
     \begin{align*}
    M\approx\begin{pmatrix}
     1 &1/\gamma &1/\gamma&1/\gamma\\
      1/\gamma&   (1/L)0.924 &-(1/L)0.278 &-(1/L)0.276\\
       1/\gamma&  -(1/L)0.278&(1/L)1.341&-(1/L)0.553\\
      1/\gamma&   -(1/L)0.276&-(1/L)0.553&(1/L)0.830
     \end{pmatrix}\otimes I_d. 
     \end{align*}
\end{example}


In the upcoming part, we will assume a strengthened version of Assumption 2 in \cite{mou2021high} about the potential function $U$. This strengthened assumption will ensure that we can approximate the nested integrals in Lemma~\ref{lemma_explicitformofxbar} with reasonable accuracy, and ultimately allow us to construct an MCMC algorithm with a better discretization error (with respect to the dimension $d$ and the accuracy level $\epsilon$) than \cite{mou2021high}. We note that in the case where $U$ is not a polynomial or a piece-wise polynomial function, the upcoming condition basically asks that $\sup_{x\in \R^d}\opnorm{\nabla^\alpha U(x)}$ does not grow too fast as $\alpha$ increases. 

\begin{customassump}{H2}~
\label{cond_derivativegrowthrate}
Let the stepsize $\eta$ and the dimension $d$ be fixed.    There exists a positive real number $c$ that does not depend on the dimension $d$ and a positive integer $\alpha$ large enough such that 
    $U$ is in $\mathcal{C}^{\alpha}$ and 
    \begin{align*}
          \brac{\frac{L_{\alpha}}{\alpha!}}^2\brac{\widetilde{C}_1}^\alpha(d+2\alpha)^\alpha \leq c \cdot d\cdot\brac{\mathds{1}_{\{P=3\}}\eta^{4}+\mathds{1}_{\{P\geq 4\}}\eta^{2P-1}}, 
    \end{align*}
where $L_{\alpha}:=    \sup_{x\in\R^d}\opnorm{\nabla^{\alpha}U(x)}$. $\widetilde{C}_1$ given in Lemma~\ref{lemma_momentbound_Pthorder} is a positive constant that depends only on the friction parameter $\gamma$ and the smoothness parameter $L$, but not on the dimension $d$ or the stepsize $\eta$.  
\end{customassump}

\begin{remark}
\label{remark_someeasierconditions}
We observe that Condition~\ref{cond_derivativegrowthrate} is satisfied whenever $U$ is a polynomial of some degree $k$, since we can take $\alpha=k+1$ so that $\nabla^\alpha U\equiv 0$. This is the case with quadratic loss function in our numerical experiments for Bayesian linear regression (our Section~\ref{section_linearegression}). More generally, one can consider a polynomial regression problem \cite[Section 3.2]{jung2022machinebook}.
\end{remark}

\begin{remark}
  In the case where $U$ is not a polynomial, an example is the regularized Huber loss function that is $U(x):=U_{0}(x)+\frac{\lambda}{2}|x|^{2}$ for some $\lambda>0$, where 
  \begin{align*}
  U_{0}(x):=\begin{cases}
    \frac{\abs{x}^{2}}{2} & \text{if} \abs{x}\leq \alpha, \\ \alpha\abs{x}-\frac{\alpha^2}{2} &\text{otherwise},
\end{cases} 
\end{align*}
for some positive parameter $\alpha$ (\cite[Page 44]{steinwartbook2008support}). In fact, for this example, we do not need to verify Condition~\ref{cond_derivativegrowthrate}  since the latter is to ensure we can approximate the nested integrals in Lemma~\ref{lemma_explicitformofxbar} (a fact pointed out in the paragraph before Condition~\ref{cond_derivativegrowthrate}). 
\end{remark}

\begin{remark}
In the case where $U$ is not a polynomial or a piece-wise polynomial function, Condition~\ref{cond_derivativegrowthrate} basically asks that $\sup_{x\in \R^d}\opnorm{\nabla^\alpha U(x)}$ does not grow too fast as $\alpha$ increases. This Condition as stated is quite hard to verify however. Hence, an example of a condition that implies Condition~\ref{cond_derivativegrowthrate} and is easier to check than the latter is: there exists an integer $K\in \N$ and real numbers $c,\beta>1$ such that for every $k\geq K$, 
\begin{align}
\label{easierconditiontocheck}
\sup_{x\in \R^d}\opnorm{\nabla^k U(x)}\leq \sqrt{c\Gamma(k/\beta+1)}d^k,
\end{align}
where $\Gamma(\cdot)$ is the gamma function. Then, since 
\begin{align*}
\lim_{k\to \infty} \frac{c\Gamma(k/\beta+1)d^{2k}\widetilde{C}_1^k(d+2k)^k}{(k!)^2}
=\lim_{k\to \infty} 
\frac{c\sqrt{\frac{2\pi k}{\beta}}(\frac{k}{\beta e})^{k/\beta}d^{2k}\widetilde{C}_1^k(d+2k)^k}{2\pi k(\frac{k}{e})^{2k}}
=0,
\end{align*}
for any fixed $\beta>1$ and $d$, 
where we applied the Stirling's formula $\Gamma(x+1)\sim\sqrt{2\pi x}(\frac{x}{e})^{x}$ as $x\rightarrow\infty$, 
the parameter $\alpha$ in Condition~\ref{cond_derivativegrowthrate} is guaranteed to exist. Finally, we note that \eqref{easierconditiontocheck} is similar to the assumption in \cite[Theorem 3.3]{wibisono2016variational} in the context of accelerated gradient methods in optimization.  
\end{remark}

\begin{remark}
Our Condition~\ref{cond_derivativegrowthrate} is much stronger than Assumption 2 in~\cite{mou2021high}, even though both are roughly about the smoothness of the loss function $U$. The reason is as follows. The mixing time of our $P$-th order LMC algorithm is determined by the error in our discretization scheme of a $P$-th order Langevin dynamics. As it will be clear from our proofs, the discretization error is a sum of two parts: the first part being the error of a splitting scheme, and the second part being the error of a polynomial approximation. As $P$ increases, we can show that the former gets smaller; however, we cannot do the same for the latter. Thus, in order to obtain an improvement of the discretization error as $P$ increases, one must assume some condition for the polynomial approximation error to be dominated by the splitting scheme error. Condition~\ref{cond_derivativegrowthrate} ensures this outcome.  
\end{remark}

\begin{remark}
In practice, even when condition~\eqref{easierconditiontocheck} or Condition~\ref{cond_derivativegrowthrate} is not satisfied, our $P$-th order LMC algorithm might still work well; see, for example, our numerical experiments for Bayesian logistic regression (Section~\ref{section_logisticregression}).
\end{remark}


\subsection{Fourth-order Langevin Monte Carlo Algorithm}
\label{section_4thorder}

\subsubsection{Fourth-order Langevin Monte Carlo algorithm} 
Given the iterate $x^{(k)}$, the next iterate $x^{(k+1)}$ is obtained by drawing from a multivariate normal distribution with mean $\textbf{M}(x^{(k)})$ and covariance $\boldsymbol{\Sigma}$, both of which are stated in Lemma~\ref{lemma_meanandcovariance}.

The proof of the next result is presented at the end of Section~\ref{section_proof_fouthorder}. 

\begin{theorem}
\label{theorem_mixingtime_4thorder}
Assume Equation~\eqref{fourthorderlangevin} satisfies Conditions~\ref{cond_mainpaper} and~\ref{cond_derivativegrowthrate}. Let $a$ be any positive constant satisfying 
\begin{align*}
    a\leq \min \left\{\frac{m}{3\gamma}, \frac{\gamma \kappa}{6} \right\}\lambda_{\min,M}, 
\end{align*}
where the positive definite matrix $M$ and the constants $\gamma,m,L,\kappa$ are from Section~\ref{section_frommonmarche}. Denote $\mu$ the invariant measure associated with the fourth-order Langevin dynamics~\eqref{fourthorderlangevin}.

Choose a $2$-Wasserstein accuracy of $\epsilon$ small enough such that $\eta_0:=\brac{\frac{\epsilon^2}{2C_1d}}^{1/7}<\min \{\eta^*,\frac{1}{h}\}$ where $h$ is defined in Proposition~\ref{prop_discretizationerror_fourthorder} and $\eta^*,C_1$ are from Lemma~\ref{lemma_momentbound}. Suppose we run our fourth-order Langevin Monte Carlo algorithm with stepsize $\eta_0$, then 
$\operatorname{Wass}_2\brac{\operatorname{Law}(x^{(k^*)}),\mu}\leq\epsilon$,
where $k^*$ is the mixing time of the fourth-order Langevin Monte Carlo algorithm with respect to $\mu$ that is given by
 \begin{align*}
k^*=\log\brac{\frac{2C_4\mathbb{E}_{Z\sim\mu}\left[\left|Z-x^{(0)}\right|^2\right]}{\epsilon^2}}\frac{(2C_3)^{1/7}}{h}\frac{d^{1/7}}{\epsilon^{2/7}}-1,
 \end{align*}
where $C_3,C_4$ are positive constants that depend on $\gamma,L,c$ but do not depend on the dimension parameter $d$.
\end{theorem}

\begin{remark}
Our mixing time rate of $O\brac{\frac{d^{1/7}}{\epsilon^{2/7}} }$ improves upon the rates in \cite{mou2021high} in terms of both $d$ and $\epsilon$ dependencies. For instance, \cite[Theorem 1]{mou2021high} has a mixing time rate of $O\brac{\frac{d^{1/4}}{\epsilon^{1/2}} }$.  
\end{remark}

\subsubsection{Derivation of the discretization scheme}
\label{section_derivation_fourthorder}

Consider the fourth-order Langevin dynamics:
\begin{align}
\label{fourthorderlangevin}
    d\theta(t)  &=v_1(t)dt,\nonumber\\
    dv_1(t)  &=\left(-\nabla U(\theta_t)+\gamma v_2(t)
    \right)dt,\nonumber\\
    dv_2(t)  &=\left(-\gamma v_1(t)+\gamma v_3(t)\right)dt,\nonumber\\
  dv_3(t)  &=(-\gamma v_2(t)-\gamma v_3(t))dt+\sqrt{2\gamma}dB_t. 
\end{align}

\begin{remark}
\label{remark_extraparamaters}
The equation \eqref{JI} in our introduction (Section~\ref{sec:intro}) is studied in \cite{mou2021high} and contains two parameters (namely $\gamma$ and $L$ in their paper) compared to our equation \eqref{Pthorderlangevinequation_mainpaper} that contains only a single parameter $\gamma$. We make such an assumption out of convenience and our paper is able to handle extra parameters as in \cite{mou2021high}, and this is explained in Appendix~\ref{appendix_frommonmarche}. Specifically in Equation~\eqref{originalequation_appendix} in Appendix~\ref{appendix_frommonmarche}, we can take $A=-\frac{1}{L}(I_d,0,\ldots, 0)$ and $\Sigma=\sqrt{\frac{4}{L}}I_{p}$. 
\end{remark}

Below we will write $\abs{\cdot}$ for the Euclidean norm and $\abs{\cdot}_M$ for the $M$-norm $\abs{x}_M=\sqrt{x^{\top}Mx}$. The numerical scheme for fourth-order Langevin dynamics consists of three stages: updating $x^{(k)}$ to $\hat{x}(t)$, then updating $\hat{x}(t)$ to $\tilde{x}(t)$, then updating $\tilde{x}(t)$ to $\bar{x}(t)$ (for $t\in [k\eta,(k+1)\eta]$). Each stage adopts a splitting scheme. 

\textbf{Stage 1:} Set the initial value
\begin{align*}
    \hat{x}(k\eta):=\left(\hat{\theta}(k\eta),\hat{v}_1(k\eta),\hat{v}_2(k\eta),\hat{v}_3(k\eta)\right)=x^{(k)}. 
\end{align*}
For $t\in (k\eta,(k+1)\eta]$, let
\begin{align*}    \hat{v}_1(t)=v_1^{(k)},
\end{align*}
and
\begin{align*}
    d\hat{\theta}(t)&=\hat{v}_1(t)dt,\\
    d\hat{v}_2(t)&=\brac{-\gamma \hat{v}_1(t)+\gamma v_3^{(k)}}dt,\\
    d\hat{v}_3(t)&=\brac{-\gamma \hat{v}_2(t)-\gamma \hat{v}_3(t)}dt+\sqrt{2\gamma}dB_t.
\end{align*}

\textbf{Stage 2:} Set the initial value $\tilde{x}(k\eta)=x^{(k)}$. For $t\in (k\eta,(k+1)\eta]$, let
\begin{align*}
    d\tilde{v}_1(t)=\left(-\tilde{g}(t)+\gamma \hat{v}_2(t)\right)dt,
\end{align*}
and
\begin{align*}
    d\tilde{\theta}(t)&=\tilde{v}_1(t)dt,\\
    d\tilde{v}_2(t)&=\brac{-\gamma \tilde{v}_1(t)+\gamma\hat{v}_3(t)}dt,\\
    d\tilde{v}_3(t)&=\brac{-\gamma \tilde{v}_2(t)-\gamma\tilde{v}_3(t)}dt+\sqrt{2\gamma}dB_t,
\end{align*}
where $\tilde{g}(t)$ is a polynomial (in $t$) of degree $\alpha-1$ and approximates $\nabla U(\hat{\theta}(t))$, and $\tilde{g}(t)$ will be defined in \eqref{def_polynomialsg} below. 

\textbf{Stage 3:} Set $\bar{x}(k\eta)=x^{(k)}$. For $t\in (k\eta,(k+1)\eta]$, let
\begin{align*}
    d\bar{v}_1(t)=\left(-\bar{g}(t)+\gamma \tilde{v}_2(t)\right)dt,
\end{align*}
and
\begin{align*}
    d\bar{\theta}(t)&=\bar{v}_1(t)dt,\\
    d\bar{v}_2(t)&=\brac{-\gamma \bar{v}_1(t)+\gamma\tilde{v}_3(t)}dt,\\
    d\bar{v}_3(t)&=\brac{-\gamma \bar{v}_2(t)-\gamma\bar{v}_3(t)}dt+\sqrt{2\gamma}dB_t, 
\end{align*}
where $\bar{g}(t)$ is a polynomial (in $t$) of degree $\alpha-1$ and approximates $\nabla U(\tilde{\theta}(t))$, and $\bar{g}(t)$ will be defined in \eqref{def_polynomialsg} below. 

Finally, set
\begin{align}
\label{scheme_4thorder}
    x^{(k+1)}= \bar{x}((k+1)\eta). 
\end{align}

\textbf{Definitions of $\tilde{g}(t)$ and $\bar{g}(t)$:} Recall $U$ is a map from $\R^d$ to $\R$, so that $\nabla U$ is a map from $\R^d$ to $L(\R^d,\R^d)$ where $L(\R^d,\R^d)$ is the space consisting of bounded linear maps from $\R^d$ to $\R^d$. Per \cite[Page 70]{cartan1971differentialbook}, the Taylor polynomial of degree $\alpha-1$ which is associated with $\nabla U$ and centers at the origin is
    \begin{align}
    \label{def_taylorpoly}
        P_{\alpha-1}(x)=\sum_{k=0}^{\alpha-1}\frac{\nabla^k U(0)}{(k-1)!} x^{k-1},
    \end{align}
where per \cite{taylorfacenda1989note}, $ \frac{\nabla^k U(0)}{(k-1)!} x^{k-1}=\sum_{i_1+\ldots+i_d=k-1}\frac{1}{i_1!\ldots i_d!}\frac{\partial^k U}{\partial x_1^{i_1}\cdots \partial x_d^{i_d}}(0) x_1^{i_1}\ldots x_d^{i_d}$.

This allows us to define 
\begin{align}
\label{def_polynomialsg}
    \tilde{g}(t):=P_{\alpha-1}(\hat{\theta}(t)); \qquad \bar{g}(t):=P_{\alpha-1}(\tilde{\theta}(t)),
\end{align}
where 
\begin{align*}
\hat{\theta}(t)
=\theta^{(k)}+(t-k\eta)v_1^{(k)},
\end{align*}
and
\begin{align*}
&\tilde{\theta}(t)=\theta^{(k)}+v_1^{(k)}(t-k\eta)-\int_{k\eta}^{t}\int_{k\eta}^s \tilde{g}(r)drds
\\
&\qquad\qquad+\gamma v_2^{(k)}\frac{(t-k\eta)^2}{2!}+\gamma^2\brac{v_3^{(k)}-v_1^{(k)} }\frac{(t-k\eta)^3}{3!}. 
\end{align*}

\begin{remark}
\label{remark_softwarefortaylorpoly}
The definitions of $\bar{g}$ and $\tilde{g}$ in \eqref{def_polynomialsg} require finding multivariate Taylor polynomials, which is a challenging task in itself. One can use numerical software to help with this, for example, by using Maple$\textsuperscript{TM}$ (\cite{redfern2012maple}) or the calculus package in R (\cite{rmultivariatetaylor}). 
\end{remark}

The next result is a consequence of Lemma~\ref{lemma_explicitformofxbar} and Lemma~\ref{lemma_meanandcovariance} from Appendix~\ref{appendix_4thorder}. 

\begin{proposition}
\label{prop_multivariatenormal}
$\E{x^{(k+1)}|x^{(k)}}=\E{\bar{x}((k+1)\eta)|\bar{x}(k\eta)}$ follows a multivariate normal distribution with mean $\textbf{M}(x^{(k)})\in \R^4$ and covariance $\boldsymbol{\Sigma}\in \R^{4\times 4}$. The explicit forms of $\textbf{M}(x^{(k)})$ and $\boldsymbol{\Sigma}$ are stated in Lemma~\ref{lemma_meanandcovariance}. 
\end{proposition}

\begin{remark}
\label{remark_polyapprox}
The authors of \cite{mou2021high} propose an MCMC algorithm based on third-order Langevin dynamics. In the case where $U$ is a general potential function and not ridge separable, an important step in their algorithm is the Lagrange polynomial interpolation step (\cite[Section 3.3]{mou2021high}) to approximate the path $s\mapsto \nabla U\brac{\theta^{(k)}+(s-k\eta)p^{(k)}}$ for given vectors $\theta^{(k)},p^{(k)}$ in $\R^d$ and $s\in [k\eta,(k+1)\eta]$. There seems to be some major difficulty in applying this Lagrange polynomial interpolation step to our MCMC algorithm based on fourth-order Langevin dynamics, which pushes us to use Taylor approximation of $\nabla U$ instead. We further explain the difficulty of using Lagrange polynomial interpolation for our algorithm in Appendix~\ref{section_polyappro}. 
\end{remark}

\begin{remark}
    \label{remark_nomorestages}
One cannot add another stage to the above discretization procedure of the fourth-order Langevin dynamics, since it is unclear how to implement the resulting algorithm in that case. The reason the current algorithm which is based on a three-stage discretization procedure can be easily implemented is that per Proposition~\ref{prop_multivariatenormal}, $\E{x^{(k+1)}|x^{(k)}}=\E{\bar{x}((k+1)
\eta) |x^{(k)}}$ is a multivariate normal distribution. 
Now suppose that we add another stage of the discretization procedure:

\textbf{Stage 4:} Set $\check{x}(k\eta)=x^{(k)}$. For $t\in (k\eta,(k+1)\eta]$, let
\begin{align*}
    d\check{v}_1(t)=\left(-\check{g}(t)+\gamma \bar{v}_2(t)\right)dt,
\end{align*}
and
\begin{align*}
    d\check{\theta}(t)&=\check{v}_1(t)dt,\\
    d\check{v}_2(t)&=\brac{-\gamma \check{v}_1(t)+\gamma\bar{v}_3(t)}dt,\\
    d\check{v}_3(t)&=\brac{-\gamma \check{v}_2(t)-\gamma\check{v}_3(t)}dt+\sqrt{2\gamma}dB_t, 
\end{align*}
where $\check{g}(t):=P_{\alpha-1}(\bar{\theta}_t)$ is a polynomial (in $t$) of degree $\alpha-1$ and approximates $\nabla U(\bar{\theta}(t))$, noting that $P_{\alpha-1}$ is the multivariate Taylor polynomial given in \eqref{def_taylorpoly}. 
Per Lemma~\ref{lemma_explicitformofxbar}, $\bar{\theta}(t)$ has the general form $F(k,\eta,\gamma,t) +\int_{k\eta}^t G(k,\eta,\gamma,s)dB_s$, so that 
$\check{\theta}(t)$ is approximately
\begin{align*}
   \theta^{(k)}-\int_{k\eta}^t\nabla U\brac{F(k,\eta,\gamma,s) +\int_{k\eta}^s G(k,\eta,\gamma,r)dB_r }  ds+\gamma\int_{k\eta}^t \bar{v}_2(s)ds. 
\end{align*}
In the case where $U$ is not a quadratic potential function, the presence of the term $\nabla U\brac{F(k,\eta,\gamma,s) +\int_{k\eta}^s G(k,\eta,\gamma,r)dB_r } $ on the right hand side suggests that $\E{\check{\theta}(t)|x^{(k)}}$ may not be multivariate normal, which makes the algorithm difficult to implement. Consequently, we do not have more than three stages in our discretization procedure. 
\end{remark}

\subsubsection{Proofs}
\label{section_proof_fouthorder}
We need a few technical lemmas whose proofs are placed near the end of Appendix~\ref{appendix_4thorder}. First, we quantify how well $\tilde{g}(t)$ and $\bar{g}(t)$ respectively approximate $\nabla U(\hat{\theta}(t))$ and $\nabla U(\tilde{\theta}(t))$. 
\begin{lemma}
\label{lemma_polyapprox}
Under Conditions~\ref{cond_mainpaper}, it holds that 
    \begin{align*}
      &  \sup_{t\in [k\eta,(k+1)\eta]}  \E{\abs{\nabla U(\hat{\theta}(t))-\tilde{g}(t)}^2} \leq \brac{\frac{L_\alpha}{\alpha!}}^2\sup_{t\in [k\eta,(k+1)\eta]} \E{\abs{\hat{\theta}(t)^{2\alpha}}},\\
      & \sup_{t\in [k\eta,(k+1)\eta]} \E{\abs{\nabla U(\tilde{\theta}(t))-\bar{g}(t)}^2}\leq \brac{\frac{L_\alpha}{\alpha!}}^2\sup_{t\in [k\eta,(k+1)\eta]} \E{\abs{\tilde{\theta}(t)^{2\alpha}}}. 
    \end{align*}
    where $L_{\alpha}:=    \sup_{x\in\R^d}\opnorm{\nabla^{\alpha}U(x)}$. 
\end{lemma}

Next, we bound the differences in $L^2$-norm of variables of two consecutive stages. 
\begin{lemma}
    \label{lemma_boundupdatedifference_4thorder}
Under Conditions~\ref{cond_mainpaper} and~\ref{cond_derivativegrowthrate}, it holds for $t\in (k\eta,(k+1)\eta]$ that
\begin{align*}
   \E{\abs{\bar{v}_1(t)-\tilde{v}_1(t)}^2}&\leq C_2(d+1)\brac{\gamma^2\brac{(\gamma+1)^2+ (2\gamma+\sqrt{2\gamma})^2}+c }\eta^5;\\
   \E{\abs{\tilde{\theta}(t)-\bar{\theta}(t)}^2}&\leq C_2(d+1)\brac{\gamma^2\brac{(\gamma+1)^2+ (2\gamma+\sqrt{2\gamma})^2}+c }\eta^7;\\
  \E{\abs{\bar{v}_2(t)-\tilde{v}_2(t)}^2}&\leq C_2(d+1)\gamma^2\brac{\gamma^2\brac{(\gamma+1)^2+ (2\gamma+\sqrt{2\gamma})^2}+c }\eta^7\\
  &\qquad\qquad+C_2(d+1)\gamma^2\brac{(\gamma+1)^2+ (2\gamma+\sqrt{2\gamma})^2}\eta^7;\\
     \E{\abs{\bar{v}_3(t)-\tilde{v}_3(t)}^2}&\leq C_2(d+1)\gamma^4\brac{\gamma^2\brac{(\gamma+1)^2+ (2\gamma+\sqrt{2\gamma})^2}+c }\eta^9\\
  &\qquad\qquad+C_2(d+1)\gamma^4\brac{(\gamma+1)^2+ (2\gamma+\sqrt{2\gamma})^2}\eta^9. 
\end{align*}
\end{lemma}

The upcoming result bounds the discretization error of the numerical scheme \eqref{scheme_4thorder}.

\begin{proposition}
\label{prop_discretizationerror_fourthorder}
Assume Equation \eqref{fourthorderlangevin} satisfies Conditions~\ref{cond_mainpaper} and~\ref{cond_derivativegrowthrate}. Let $a$ be any positive constant satisfying $a\leq \min \left\{\frac{m}{3\gamma}, \frac{\gamma \kappa}{6} \right\}\lambda_{\min,M}$, 
where the positive definite matrix $M$ and the constants $\gamma,m,L,\kappa$ are from Section~\ref{section_frommonmarche}. Denote $\mu$ the invariant measure of the fourth-order Langevin dynamics \eqref{fourthorderlangevin}.

Then regarding the discretization error, it holds when $\eta<\min \{\eta^*,\frac{1}{h}\}$ that 
    \begin{align*}
    \E{\abs{x((k+1)\eta)-x^{(k+1)}}^2}\leq C_3d\eta^8+C_4e^{-(k+1)h\eta}\E{\abs{Z-x^{(0)}}^2},\qquad Z\sim \mu. 
\end{align*}
In particular, $\eta^*$ is defined at~\eqref{def_eta*}, and $h:=2\rho-\frac{2a}{\lambda_{\min,M}}$ where $\rho,M$ are from Theorem \ref{theorem_frommonmarche_mainpaper} and $a$ is any positive constant equal to or less than $\min \left\{\frac{m}{3\gamma}, \frac{\gamma \kappa}{6} \right\}\lambda_{\min,M}$. Moreover, $C_3,C_4$ are positive constants that depend only on $\gamma,L,c$ but do not depend on the dimension parameter $d$. 
\end{proposition}

\begin{proof}

\textbf{Step 1:} Assume $t\in [k\eta,(k+1)\eta]$ and recall that $\bar{x}(t)=\brac{\bar{\theta}(t),\bar{v}_1,\bar{v}_2(t),\bar{v}_3(t)}$. Based on \eqref{scheme_4thorder}, we have
\begin{align*}
    d\bar{x}(t)=\bar{b}(t)dt+\sqrt{2\gamma}DdB_t,
\end{align*}
where
\begin{align*}
  D:=\begin{pmatrix}
  0_{d}&0_{d} &0_{d}&0_{d}\\
    0_{d}&0_{d}& 0_{d}&0_{d}\\
   0_{d}&0_{d}& 0_{d}&0_{d}\\
    0_{d}&   0_{d}&   0_{d}&I_{d}
\end{pmatrix};\qquad  \bar{b}(t):=\begin{pmatrix}
        \bar{v}_1(t)\\
        -\bar{g}(t)+\gamma \tilde{v}_2(t)\\
        -\gamma \bar{v}_1(t)+\gamma \tilde{v}_3(t)\\
        {-\gamma \bar{v}_2(t)-\gamma \bar{v}_3(t)}
    \end{pmatrix}. 
\end{align*}
Meanwhile, the fourth-order Langevin dynamics in \eqref{fourthorderlangevin}
can be written as 
\begin{align*}
dx(t)=b(x(t))dt+\sqrt{2\gamma}DdB_t, \qquad b(x)=\begin{pmatrix}
        v_1\\
        -\nabla U(\theta)+\gamma v_2\\
        -\gamma v_1+\gamma v_3\\
        -\gamma v_2-\gamma v_3
    \end{pmatrix}. 
\end{align*}
Then
\begin{align*}
    d\brac{x(t)-\bar{x}(t)}=\brac{b(x(t))-b\brac{\bar{x}(t)}}dt+\brac{b(\bar{x}(t))-\bar{b}(t)}dt. 
\end{align*}
This leads to
\begin{align}
\label{error_firststep}
    &\frac{d}{dt}\E{\brac{x(t)-\bar{x}(t)}^{\top} M \brac{x(t)-\bar{x}(t)}}\nonumber\\
    &=\E{\brac{x(t)-\bar{x}(t)}^{\top} M\brac{b(x(t))-b(\bar{x}(t))}}+\E{\brac{x(t)-\bar{x}(t)}^{\top} M\brac{b(\bar{x}(t))-\bar{b}(t)} }\nonumber\\
    &\quad+\E{\brac{b(x(t))-b(\bar{x}(t))}^{\top} M\brac{x(t)-\bar{x}(t)}}+\E{\brac{b(\bar{x}(t))-\bar{b}(t)}^{\top} M \brac{x(t)-\bar{x}(t)}}, 
\end{align}
where 
\begin{align}\label{equation_differencebandbbar}
  b(\bar{x}(t))-\bar{b}(t) =  \begin{pmatrix}
        0\\
        \bar{g}(t)-\nabla U(\bar{\theta}(t))-\gamma\brac{\tilde{v}_2(t)-\bar{v}_2(t)}\\
        \gamma\brac{\bar{v}_3(t)-\tilde{v}_3(t)}\\
        0
    \end{pmatrix}.
\end{align}

Recall the $M$-norm $ \abs{x}_M=\sqrt{x^{\top}Mx}$ and notice that 
\begin{align}
\label{error_secondstep}
    &\brac{x(t)-\bar{x}(t)}^{\top} M\brac{b(x(t))-b(\bar{x}(t))}+\brac{b(x(t))-b(\bar{x}(t))}^{\top} M\brac{x(t)-\bar{x}(t)}\nonumber\\
    &\leq \brac{x(t)-\bar{x}(t)}^{\top} M\int_0^1 J_b\brac{wx(t)+(1-w)\bar{x}(t)}\brac{x(t)-\bar{x}(t)}dw\nonumber\\
    &\qquad+\int_0^1 \brac{x(t)-\bar{x}(t)}^{\top} J_b\brac{wx(t)+(1-w)\bar{x}(t)}^{\top}dw M \brac{x(t)-\bar{x}(t)}\nonumber\\
    &\leq \brac{x(t)-\bar{x}(t)}^{\top} (-2\rho)M \brac{x(t)-\bar{x}(t)}=-2\rho\abs{x(t)-\bar{x}(t)}^2_M,
\end{align}
where the last line is due to the contraction property \eqref{contraction} in Theorem~\ref{theorem_frommonmarche_mainpaper}.

Moreover, choose any positive $a\leq \min \left\{\frac{m}{3\gamma}, \frac{\gamma \kappa}{6} \right\}\lambda_{\min,M}$ and notice that per Theorem~\ref{theorem_frommonmarche_mainpaper},
\begin{align}
\label{choiceofa}
    -\rho+\frac{a}{\lambda_{\min,M}}<0. 
\end{align}

From \eqref{error_firststep}, \eqref{error_secondstep}, \eqref{choiceofa} and Cauchy-Schwarz inequality, we can deduce that
\begin{align}
 &\frac{d}{dt} \E{ \abs{x(t)-\bar{x}(t)}^2_M}\\
 &\leq -2\rho \E{\abs{x(t)-\bar{x}(t)}^2_M+ 2a\abs{x(t)-\bar{x}(t)}^2}+\frac{2}{a}\norm{M}^2_{\operatorname{op}}\brac{\gamma^2 +1}\E{\abs{b(\bar{x}(t))-\bar{b}(t)}^2}. \label{previous:bound}
\end{align}
By combining the bound in \eqref{previous:bound} with 
\begin{align}
\label{equivalenceMnorm} \sqrt{\lambda_{\min,M}}\abs{x}\leq \abs{x}_M\leq \sqrt{\lambda_{\min,M}}\abs{x},
\end{align}
for any $x$, we get
\begin{align}
\label{error_beforeboundingtheextraterm}
 &\frac{d}{dt}  \E{\abs{x(t)-\bar{x}(t)}^2_M}\nonumber\\
 &\leq \brac{-2\rho+\frac{2a}{\lambda_{\min,M}} } \E{\abs{x(t)-\bar{x}(t)}^2_M}+\frac{2}{a}\norm{M}^2_{\operatorname{op}}\brac{\gamma^2+1} \E{\abs{b(\bar{x}(t))-\bar{b}(t)}^2}. 
\end{align}

\textbf{Step 2:} We will bound $\E{\abs{b(\bar{x}(t))-\bar{b}(t)}^2}$ as the second term on the right hand side of \eqref{error_beforeboundingtheextraterm}. Based on \eqref{equation_differencebandbbar}, we will need to bound the $L^2$ norm of 
\begin{align}
\label{needtobound}
    \bar{g}(t)-\nabla U(\bar{\theta}(t)), \quad\gamma\brac{\tilde{v}_2(t)-\bar{v}_2(t)},\quad \text{ and } \quad\gamma\brac{\bar{v}_3(t)-\tilde{v}_3(t)}. 
\end{align}

Let us start with 
\begin{align}
\label{estimate_approximationbytildeg}
    \E{\abs{\bar{g}(t)- \nabla U(\bar{\theta}(t))}^2}\leq 2\E{\abs{\bar{g}(t)-\nabla U(\tilde{\theta}(t))}^2} +2\E{\abs{\nabla U(\tilde{\theta}(t))-\nabla U(\bar{\theta}(t))}^2}. 
\end{align}

The first term on the right hand side in \eqref{estimate_approximationbytildeg} is bounded in \eqref{bound_polyapproxtildetheta} as $\E{\abs{\bar{g}(t)-\nabla U(\tilde{\theta}(t))}^2}\leq   cd\eta^7$. The second term on the right hand side in \eqref{estimate_approximationbytildeg} can be bounded by $L$-smoothness of $U$ in Condition~\ref{cond_mainpaper} and Lemma~\ref{lemma_boundupdatedifference_4thorder} as 
\begin{align*}
\E{\abs{\nabla U(\tilde{\theta}(t))-\nabla U(\bar{\theta}(t))}^2}\leq L^2\E{\abs{\bar{\theta}(t)-\tilde{\theta}(t)}^2}\leq C_1d\eta^7, 
\end{align*}
where $C_1$ denotes a generic constant that depends only on $\gamma,L$ and can change from line to line. Thus, 
\begin{align}
\label{estimate_approximationbytildeg_thebound}
     \E{\abs{\bar{g}(t)- \nabla U(\bar{\theta}(t))}^2}\leq (C_1+c)d\eta^7. 
\end{align}
We also know from Lemma~\ref{lemma_boundupdatedifference_4thorder} that 
\begin{align}
\label{anotherbound}
    \E{\abs{\tilde{v}_2(t)-\bar{v}_2(t)}^2}+\E{\abs{\tilde{v}_3(t)-\bar{v}_3(t)}^2}\leq C_1d\eta^7. 
\end{align}
Per \eqref{needtobound} and \eqref{estimate_approximationbytildeg_thebound}, \eqref{anotherbound}, we arrive at $\E{\abs{b(\bar{x}(t))-\bar{b}(t)}^2}\leq C_3d\eta^7$. Then per \eqref{error_beforeboundingtheextraterm}, we have 
\begin{align}
\label{prefinalstep}
     &\frac{d}{dt}  \E{\abs{x(t)-\bar{x}(t)}^2_M}\leq \brac{-2\rho+\frac{2a}{\lambda_{\min,M}} } \E{\abs{x(t)-\bar{x}(t)}^2_M}+C_3d\eta^7.
\end{align}

\textbf{Step 3:} Let us rewrite \eqref{prefinalstep} as
\begin{align}
\label{simplifiednotation}
    \frac{d\Delta}{dt}(t)\leq -h\Delta(t)+C_3d\eta^7, 
\end{align}
where $h:=2\rho-\frac{2a}{\lambda_{\min,M}}>0$ and $\Delta(t):=\E{\abs{x(t)-\bar{x}(t)}^2_M}$. We solve \eqref{simplifiednotation} by the integrating factor method. We integrate from $k\eta$ to $t$ to obtain for $t\in [k\eta,(k+1)\eta]$, 
\begin{align*}
    \Delta(t)&=C_3de^{-ht}\int_{k\eta}^t e^{hs}(s-k\eta)^7ds+e^{h(k\eta-t)}\Delta(k\eta)\\
    &\leq C_3d\int_{k\eta}^t (s-k\eta)^7ds+e^{h(k\eta-t)}\Delta(k\eta)= \frac{C_3(t-k\eta)^8}{8}+e^{h(k\eta-t)}\Delta(k\eta). 
\end{align*}
Therefore, we get 
\begin{align*}
\Delta((k+1)\eta)\leq \frac{C_3d}{8}\eta^8+e^{-h\eta}\Delta(k\eta),    
\end{align*} 
which leads to
\begin{align*}
    \Delta((k+1)\eta)\leq \frac{C_3d}{8}\eta^8 \sum_{j=0}^{k-1} e^{-jh\eta}+e^{-(k+1)h\eta}\Delta(0)\leq \frac{C_3d}{8}\eta^8 \frac{1}{1-e^{-h\eta}}+e^{-kh\eta}\Delta(0). 
\end{align*}
Observe that when $\eta\leq \frac{1}{h}$, we have $\frac{1}{1-\frac{h\eta}{2}}\leq 2$ and hence $\frac{1}{1-e^{-h\eta}}\leq \frac{1}{h\eta(1-\frac{h\eta}{2})}\leq \frac{2}{h\eta}$. This implies 
\begin{align*}
\Delta((k+1)\eta)\leq \frac{4C_3d}{8h}\eta^7+e^{-(k+1)h\eta}\Delta(0). 
\end{align*}
By the equivalence of norm relation~\eqref{equivalenceMnorm}, we further obtain 
\begin{align*}
 \lambda_{\min,M}   \E{\abs{x((k+1)\eta)-x^{(k)}}^2}\leq  \frac{4C_3d}{8h}\eta^7+e^{-(k+1)h\eta} \lambda_{\min,M}   \E{\abs{x(0)-x^{(0)}}^2}. 
\end{align*}
Now assume the continuous dynamics \eqref{fourthorderlangevin} is stationary and $x(0)$ is distributed as its invariant measure $\mu$. Then $\E{\abs{x(0)-x^{(0)}}^2}=\E{\abs{Z-x^{(0)}}^2_M}$ where $Z\sim \mu$. We also know $\lambda_{\min,M}, \lambda_{\max,M}$ do not depend on $d$ per Corollary~\ref{coro_monmarcheforPthorderlangevin}. Thus, we arrive at 
\begin{align*}
    \E{\abs{x((k+1)\eta)-x^{(k)}}^2}\leq C_3d\eta^7+C_4e^{-(k+1)h\eta}\E{\abs{Z-x^{(0)}}^2},
\end{align*}
where $C_3,C_4$ are positive constants that depend on $\gamma,L,c$ but do not depend on $d$. Here we abuse notations and reuse $C_3,C_4$. 

Finally, the fact that the constant $h:=2\rho-\frac{2a}{\lambda_{\min,M}}>0$ depends on $\gamma,L$ and does not depend on $d$ is due to Theorem~\ref{theorem_frommonmarche_mainpaper}. This completes the proof. 
\end{proof}

\begin{proof}[Proof of Theorem~\ref{theorem_mixingtime_4thorder}]
Recall the basic fact about the $2$-Wasserstein distance that 
\begin{align*}
\operatorname{Wass}_2\brac{\operatorname{Law}(X),\operatorname{Law}(Y)}\leq \E{\abs{X-Y}^2}^{1/2}.
\end{align*}
 In view of Proposition~\ref{prop_discretizationerror_fourthorder}, we can then derive the mixing time with respect to $\operatorname{Wass}_2$ by solving for $C_3d\eta^7\leq \epsilon^2/2$ and $C_4e^{-(k+1)h\eta}\mathbb{E}_{Z\sim\mu}\left[\left|Z-x^{(0)}\right|\right]^2\leq \epsilon^2/2$. Solving for $\eta$ in the first equation gives $\eta\leq \eta^*:=\brac{\frac{\epsilon^2}{2C_3d}}^{1/7}$ Solving for $k$ in the second equation gives $k\geq \log\brac{\frac{2C_4\mathbb{E}_{Z\sim\mu}\left[\left|Z-x^{(0)}\right|^2\right]}{\epsilon^2}}\frac{1}{h\eta}-1$. Plugging in the largest possible stepsize $\eta^*$ into the right hand side of the previous inequality leads to the mixing time as claimed. 
\end{proof}

\subsection{$P$-th order Langevin Monte Carlo Algorithm for $P\geq 3$}
\label{section_Pthorder}

\subsubsection{$P$-th order Langevin Monte Carlo algorithm} 

Given the iterate $x^{(k)}$, the next iterate $x^{(k+1)}$ is obtained by drawing from a multivariate normal distribution. The mean vector $\textbf{M}(x^{(k)})$ and the covariance matrix $\boldsymbol{\Sigma}$ of this multivariate normal distribution are not provided explicitly, but their derivations are explained in the proof of Lemma~\ref{lemma_meanandcovariance_Pthorder} for any order $P\geq 3$.

Below is the main result of this section. The proof is placed at the end of Section~\ref{section_proof_Pthorder}. 
\begin{theorem}
\label{theorem_mixingtime_Pthorder}
Assume $P\geq 3$ and Equation \eqref{equation_Pthorderlangevin} satisfies Conditions~\ref{cond_mainpaper} and~\ref{cond_derivativegrowthrate}. Let $a$ be any positive constant satisfying 
\begin{align*}
    a\leq \min \left\{\frac{m}{3\gamma}, \frac{\gamma \kappa}{6} \right\}\lambda_{\min,M}, 
\end{align*}
where the positive definite matrix $M$ and the constants $\gamma,m,L,\kappa$ are from Section~\ref{section_frommonmarche}. Denote $\mu$ the invariant measure associated with the $P$-th order Langevin dynamics~\eqref{Pthorderlangevinequation_mainpaper}. 

Let \begin{align*}
    \mathcal{R}=4\cdot\mathds{1}_{\{ P=3\}}+ (2P-1)\cdot\mathds{1}_{\{ P\geq 4\}}.
\end{align*}
Choose a $2$-Wasserstein accuracy of $\epsilon$ small enough such that $\eta_0:=\brac{\frac{\epsilon^2}{2\widetilde{C}_1d}}^{1/\mathcal{R}}<\min\{\eta^{**},\frac{1}{h}\}$ where $h$ is defined in Proposition~\ref{prop_discretizationerror_Pthorder} and $\eta^{**},\widetilde{C}_1$ are from Lemma~\ref{lemma_momentbound_Pthorder}. Then 
$\operatorname{Wass}_2\brac{\operatorname{Law}(x^{(k^*)}),\mu}\leq\epsilon$,
where $k^*$ is the mixing time of the $P$-th order Langevin Monte Carlo algorithm with respect to $\mu$ that is given by 
\begin{align*}
k^*=\log\brac{\frac{2\widetilde{C}_4\mathbb{E}_{Z\sim\mu}\left[\left|Z-x^{(0)}\right|^2\right]}{\epsilon^2}}\frac{\brac{2\widetilde{C}_3}^{1/\mathcal{R}}}{h}\frac{d^{1/\mathcal{R}}}{\epsilon^{1/(2\mathcal{R})}}-1,
 \end{align*}
where $\widetilde{C}_3,\widetilde{C}_4$ are positive constants from Proposition~\ref{prop_discretizationerror_Pthorder} that depend on $c,\gamma,L,P$ and do not depend on the dimension $d$.
\end{theorem}

\begin{remark}
    In the cases $P=3$ and $P=4$, the results in Theorem~\ref{theorem_mixingtime_Pthorder} match, respectively, the result in \cite{mou2021high} and the result in our Theorem~\ref{theorem_mixingtime_4thorder}. 
\end{remark}

\subsubsection{Derivation of the discretization scheme}
\label{section_derivation_Pthorder}

We generalize what was done in Section~\ref{section_4thorder} to the $P$-th order Langevin dynamics which is 
\begin{align}
\label{equation_Pthorderlangevin}
    d\theta(t)&=v_1(t)dt,\\
    dv_1(t)&=-\nabla U(\theta(t))dt+\gamma v_2(t)dt,\nonumber\\
    dv_n(t)&=-\gamma v_{n-1}(t)dt+\gamma v_{n+1}(t)dt, \quad 2\leq n\leq P-2,\nonumber\\
    dv_{P-1}(t)&=-\gamma v_{P-2}(t)dt-\gamma v_{P-1}(t)dt+\sqrt{2\gamma}dB_t\nonumber. 
\end{align}
Note that we can handle similar models with extra parameters as explained in Remark~\ref{remark_extraparamaters}. 

Let us describe the splitting scheme for any $P\geq 3$. We assume that we know $x^{(k)}=\brac{\theta^{(k)}, v_1^{(k)}, \ldots, v_{P-1}^{(k)}}$ and that is performing the $(k+1)$-th iterate of our algorithm.

\textbf{Stage $1$:}

Set the initial value
\begin{align*}
    x^{\operatorname{st}_1}(k\eta):=\brac{\theta^{\operatorname{st}_1}(k\eta), v_1^{\operatorname{st}_1}(k\eta), \ldots, v_{P-1}^{\operatorname{st}_1}(k\eta)}= x^{(k)}.
\end{align*}
For $t\in (k\eta,(k+1)\eta]$, let 
\begin{align*}
    v_1^{\operatorname{st}_1}(t)=v_1^{(k)},
\end{align*}
and 
\begin{align*}
    d\theta^{\operatorname{st}_1}(t)&=v_1^{\operatorname{st}_1}(t)dt,\\
    dv_n^{\operatorname{st}_1}(t)&=-\gamma v_{n-1}^{\operatorname{st}_1}(t)dt+\gamma v_{n+1}^{(k)}dt, \quad 2\leq n\leq P-2,\\
    dv_{P-1}^{\operatorname{st}_1}(t)&=-\gamma v_{P-2}^{\operatorname{st}_1}(t)dt-\gamma v_{P-1}^{\operatorname{st}_1}(t)dt+\sqrt{2\gamma}dB_t. 
\end{align*}

\textbf{Stage $j$ for $2\leq j\leq P-1$:} 

Set the initial value
\begin{align*}
    x^{\operatorname{st}_j}(k\eta):=\brac{\theta^{\operatorname{st}_j}(k\eta), v_1^{\operatorname{st}_j}(k\eta), \ldots, v_{P-1}^{\operatorname{st}_j}(k\eta)}= x^{(k)}.
\end{align*}
For $t\in (k\eta,(k+1)\eta]$, let 
\begin{align*}
      dv_1^{\operatorname{st}_j}(t)&=-g^{\operatorname{st}_j}(t)dt+\gamma v_2^{\operatorname{st}_{j-1}}(t)dt,
\end{align*}
and 
\begin{align*}
    d\theta^{\operatorname{st}_j}(t)&=v_1^{\operatorname{st}_j}(t)dt,\\
    dv_n^{\operatorname{st}_j}(t)&=-\gamma v_{n-1}^{\operatorname{st}_j}(t)dt+\gamma v_{n+1}^{\operatorname{st}_{j-1}}dt, \quad 2\leq n\leq P-2,\\
    dv_{P-1}^{\operatorname{st}_j}(t)&=-\gamma v_{P-2}^{\operatorname{st}_j}(t)dt-\gamma v_{P-1}^{\operatorname{st}_j}(t)dt+\sqrt{2\gamma}dB_t. 
\end{align*}
Note that $g^{\operatorname{st}_j}(t)$ is a polynomial (in $t$) of degree $\alpha-1$ and approximates $\nabla U(\theta^{\operatorname{st}_{j-1}}(t))$, and $g^{\operatorname{st}_j}(t)$ will be defined in \eqref{def_polynomialsg_Pthorder} below.

Finally, we set 
\begin{align}
\label{scheme_Pthorderlangevin}
x^{(k+1)}=x^{\operatorname{st}_{P-1}}((k+1)\eta). 
\end{align}

\textbf{Definitions of $g^{\operatorname{st}_j}(t),2\leq j\leq P-1$:} Here we define the polynomials $g^{\operatorname{st}_j}(t)$ which approximate $\nabla U(\theta^{\operatorname{st}_{j-1}}(t))$.

Recall the Taylor polynomial $P_{\alpha-1}$ of degree $\alpha-1$ centering at $0$ and associated with $\nabla U$ in \eqref{def_taylorpoly}. Inductively for $2\leq j\leq P-1$, let us set
\begin{align}
\label{def_polynomialsg_Pthorder}
    g^{\operatorname{st}_j}(t):=P_{\alpha-1}\brac{ \theta^{\operatorname{st}_{j-1}}(t)}. 
\end{align}

\begin{remark}
In the case $P=3$, the above splitting scheme is fairly similar to the splitting scheme in \cite{mou2021high}. The only notable difference is that we employ Taylor polynomial approximation while the authors of \cite{mou2021high} employ Lagrange polynomial interpolation. The necessity of this difference has been discussed in Remark~\ref{remark_polyapprox} and in Appendix~\ref{section_polyappro}. 
\end{remark}

\begin{remark}
\label{remark_Pthorder}
The general idea of our numerical scheme is that there are several stages of refinement. At every stage, the variable $v_1$ which contains the non-linear term $\nabla U(\theta)$ is split from the other variables and approximated first, while the vector formed by the remaining variables is  approximated by a multivariate Ornstein-Uhlenbeck process. As a result, the discretization procedure only works for $P\geq 3$. 

Moreover, there are $P-1$ stages in the above discretization procedure and we cannot add another one to it. The reason is similar to the one given in Remark~\ref{remark_nomorestages}. 
\end{remark}

\subsubsection{Proofs}
\label{section_proof_Pthorder}

The following result is similar to Lemma~\ref{lemma_polyapprox}. The proof is simple and is therefore omitted. 
\begin{lemma}
\label{lemma_polyapprox_Pthorder}
Denote $L_{\alpha}:=    \sup_{x\in\R^d}\opnorm{\nabla^{\alpha}U(x)}$. Under Conditions~\ref{cond_mainpaper}, it holds for $1\leq j\leq P-1$ that
    \begin{align*}
       \sup_{t\in [k\eta,(k+1)\eta]}  \E{\abs{\nabla U(\theta^{\operatorname{st}_{j-1}}(t))-g^{\operatorname{st}_j}(t)}^2} \leq \brac{\frac{L_\alpha}{\alpha!}}^2\sup_{t\in [k\eta,(k+1)\eta]} \E{\abs{\theta^{\operatorname{st}_{j-1}}(t)}^{2\alpha}}. 
    \end{align*}
\end{lemma}

The next two results bound the differences in $L^2$-norm of variables of two consecutive stages. Lemma~\ref{lemma_stagedifference_Pthorder} is a consequence of Lemma~\ref{lemma_basecases}. Their proofs are deferred to near the end of Appendix~\ref{appendix_Pthorder}. 

\begin{lemma}
\label{lemma_basecases}
     Assume $P\geq 3$ and consider the splitting scheme at the beginning of this section with $P-1$ stages. Denote the stages by $j, 1\leq j\leq P-1$. It holds that
       \begin{itemize}
      \item for $j=1$: $\sup_{t\in(k\eta,(k+1)\eta]} \E{\abs{v^{\operatorname{st}_{ 1}}_n(t)-v^{(k)}_n }^2}\leq C^{\operatorname{st}_{ 1}}_n d\eta^2 ,1 \leq n\leq P-2$; 
      \\$\sup_{t\in(k\eta,(k+1)\eta]} \E{\abs{v^{\operatorname{st}_{ 1}}_{P-1}(t)-v^{(k)}_{P-1} }^2}\leq C^{\operatorname{st}_{ 1}}_{P-1} d\eta$ and \\$\sup_{t\in(k\eta,(k+1)\eta]} \E{\abs{\theta^{\operatorname{st}_{ 1}}(t)-\theta^{(k)} }^2}\leq C^{\operatorname{st}_{ 1}}_P d\eta^2$. 
      \item for $j=2$

and $P=3$:  $\sup_{t\in(k\eta,(k+1)\eta]} \E{\abs{v^{\operatorname{st}_{ 2}}_1(t)-v^{\operatorname{st}_{1 }}_1(t) }^2}\leq C^{\operatorname{st}_{ 2}}_1 d\eta^2$, \\$\sup_{t\in(k\eta,(k+1)\eta]} \E{\abs{v^{\operatorname{st}_{ 2}}_2(t)-v^{\operatorname{st}_{1 }}_2(t) }^2}\leq C^{\operatorname{st}_{ 2}}_2 d\eta^2$ and \\$\sup_{t\in(k\eta,(k+1)\eta]} \E{\abs{\theta^{\operatorname{st}_{ 2}}(t)-\theta^{\operatorname{st}_{1 }}(t) }^2}\leq C^{\operatorname{st}_{ 2}}_3 d\eta^4$.

and $P\geq 4$: $\sup_{t\in(k\eta,(k+1)\eta]} \E{\abs{v^{\operatorname{st}_{ 2}}_1(t)-v^{\operatorname{st}_{1 }}_1(t) }^2}\leq C^{\operatorname{st}_{ 2}}_1 d\eta^2$;\\ $\sup_{t\in(k\eta,(k+1)\eta]} \E{\abs{v^{\operatorname{st}_{ 2}}_n(t)-v^{\operatorname{st}_{1 }}_n(t) }^2}\leq C^{\operatorname{st}_{ 2}}_n d\eta^4, 2\leq n\leq P-3$; \\$\sup_{t\in(k\eta,(k+1)\eta]} \E{\abs{v^{\operatorname{st}_{ 2}}_{P-2}(t)-v^{\operatorname{st}_{1 }}_{P-2}(t) }^2}\leq C^{\operatorname{st}_{ 2}}_{P-2} d\eta^3$,
\\$\sup_{t\in(k\eta,(k+1)\eta]} \E{\abs{v^{\operatorname{st}_{ 2}}_{P-1}(t)-v^{\operatorname{st}_{1 }}_{P-1}(t) }^2}\leq C^{\operatorname{st}_{ 2}}_{P-1} d\eta^5$ and \\$\sup_{t\in(k\eta,(k+1)\eta]} \E{\abs{\theta^{\operatorname{st}_{ 2}}(t)-\theta^{\operatorname{st}_{1 }}(t) }^2}\leq C^{\operatorname{st}_{ 2}}_P d\eta^4$. 
\end{itemize}

Here, $\left\{C^{\operatorname{st}_{ j}}_n:j=1 \text{ or } j=2,1\leq n\leq P \right\}$ are constants that depend on $c,\gamma,L,P$ but do not depend on the dimension $d$. 
\end{lemma}

\begin{lemma}
\label{lemma_stagedifference_Pthorder}
Assume $P\geq 3$ and consider the splitting scheme at the beginning of this section with $P-1$ stages. Denote the stages by $j, 1\leq j\leq P-1$.  It holds for $j\geq 3$
   \begin{enumerate}[label=\alph*)]
          \item and $1\leq n\leq P-j-1:\sup_{t\in(k\eta,(k+1)\eta]} \E{\abs{v^{\operatorname{st}_{ j}}_n(t)-v^{\operatorname{st}_{j-1 }}_n(t) }^2}\leq C^{\operatorname{st}_{ j}}_n d\eta^{2j}$;
          \item and $n=P-j:\sup_{t\in(k\eta,(k+1)\eta]} \E{\abs{v^{\operatorname{st}_{ j}}_{P-j}(t)-v^{\operatorname{st}_{j-1 }}_n(t) }^2}\leq C^{\operatorname{st}_{ j}}_{P-j} d\eta^{2j-1}$;
          \item and $P-j+1\leq n\leq P-1:\sup_{t\in(k\eta,(k+1)\eta]} \E{\abs{v^{\operatorname{st}_{ j}}_n(t)-v^{\operatorname{st}_{j-1 }}_n(t) }^2}\\\leq C^{\operatorname{st}_{ j}}_n d\eta^{4j+2n-2P-1}$;
          \item $\sup_{t\in(k\eta,(k+1)\eta]} \E{\abs{\theta^{\operatorname{st}_{ j}}(t)-\theta^{\operatorname{st}_{j-1 }}(t) }^2}\leq C^{\operatorname{st}_{ j}}_P d\eta^{2j+2}$.
      \end{enumerate}
$\left\{C^{\operatorname{st}_{ j}}_n:3\leq j\leq P-1,1\leq n\leq P \right\}$ are constants that depend on $c,\gamma,L,P$ but do not depend on dimension $d$. 

Consequently, it holds for the last Stage $P-1$ that 
\begin{align}
\label{difference_laststage_Pis3}
    \sup_{t\in(k\eta,(k+1)\eta]} \brac{\E{\abs{v^{\operatorname{st}_{ 2}}_2(t)-v^{\operatorname{st}_{1 }}_2(t) }^2}+ \E{\abs{\theta^{\operatorname{st}_{ 2}}(t)-\theta^{\operatorname{st}_{1 }}(t) }^2}}\leq \brac{C^{\operatorname{st}_{ 2}}_2+C^{\operatorname{st}_{ 2}}_3  }d\eta^4, \quad P=3;
\end{align}
and
\begin{align}
\label{difference_laststage_Pmorethan3}
    &\sup_{t\in(k\eta,(k+1)\eta]} \brac{\sum_{n=2}^{P-1}\E{\abs{v^{\operatorname{st}_{ P-1}}_n(t)-v^{\operatorname{st}_{P-2 }}_n(t) }^2}+\E{\abs{\theta^{\operatorname{st}_{ P-1}}(t)-\theta^{\operatorname{st}_{P-2 }}(t) }^2}}\nonumber
    \\
&\qquad\qquad\qquad\qquad\qquad\qquad\qquad\qquad\leq  \brac{\sum_{n=2}^P C^{\operatorname{st}_{ P-1}}_n}d\eta^{2P-1},\quad P\geq 4.
\end{align}
\end{lemma}

\begin{proposition}\label{prop_discretizationerror_Pthorder}    
Assume Equation~\eqref{equation_Pthorderlangevin} satisfies Conditions~\ref{cond_mainpaper} and~\ref{cond_derivativegrowthrate}. Let $a$ be any positive constant satisfying 
\begin{align*}
    a\leq \min \left\{\frac{m}{3\gamma}, \frac{\gamma \kappa}{6} \right\}\lambda_{\min,M}, 
\end{align*}
where the positive definite matrix $M$ and the constants $\gamma,m,L,\kappa$ are from Section~\ref{section_frommonmarche}. Denote $\mu$ the invariant measure of the $P$-th order Langevin dynamics \eqref{equation_Pthorderlangevin}.

Assume further that $\eta<\{\eta^{**},\frac{1}{h}\}$. Then regarding the discretization error, it holds when and $P\geq 4$ that 
    \begin{align*}
    \E{\abs{x((k+1)\eta)-x^{(k+1)}}^2}\leq \widetilde{C}_3d\eta^{2P-1}+\widetilde{C}_4e^{-(k+1)h\eta}\E{\abs{Z-x^{(0)}}^2},\qquad Z\sim \mu,
\end{align*}
and when $P=3$ that 
  \begin{align*}
    \E{\abs{x((k+1)\eta)-x^{(k+1)}}^2}\leq \widetilde{C}_3d\eta^{4}+\widetilde{C}_4e^{-(k+1)h\eta}\E{\abs{Z-x^{(0)}}^2},\qquad Z\sim \mu.
\end{align*}
In particular,  $\eta^{**}$ is defined at~\eqref{def_eta**}, $h:=2\rho-\frac{2a}{\lambda_{\min,M}}$ where $\rho,M$ are from Theorem \ref{theorem_frommonmarche_mainpaper} and $a$ is any positive constant equal to or less than  $\min \left\{\frac{m}{3\gamma}, \frac{\gamma \kappa}{6} \right\}\lambda_{\min,M}$ . Moreover, $\widetilde{C}_3$ and $\widetilde{C}_4$ are constants that depend on $\gamma,L,c$ but not on the dimension $d$. 
\end{proposition}

\begin{proof}
We will follow the argument in the proof of Proposition~\ref{prop_discretizationerror_fourthorder}. 

\textbf{Step 1:} Let us write 
\begin{align*}
    dx^{\operatorname{st}_{P-1 }}=\bar{b}(t)dt+\sqrt{2\gamma}DdB_t,
\end{align*}
where $\bar{b}(t)\in \mathcal{M}_{P\times 1}$ and $D\in \mathcal{M}_{Pd\times Pd}$ are respectively given by:
 \begin{align*}
        \bar{b}(t):=\begin{pmatrix}
        v^{\operatorname{st}_{P-1 }}_1(t)\\
        -g^{\operatorname{st}_{P-1 }}(t)+\gamma v^{\operatorname{st}_{P-2 }}_2(t)\\
        -\gamma v^{\operatorname{st}_{P-1 }}_1(t)+\gamma v^{\operatorname{st}_{P-2 }}_3(t)\\
        \vdots\\
        -\gamma v^{\operatorname{st}_{P-1 }}_{P-3}(t)+ \gamma v^{\operatorname{st}_{ P-2}}_{P-1}(t)\\
        -\gamma v^{\operatorname{st}_{P-1 }}_{P-2}(t)- \gamma v^{\operatorname{st}_{ P-1}}_{P-1}(t)
    \end{pmatrix},\quad 
  D:=\begin{pmatrix}
  0_{d}&\cdots &\cdots &0_{d}\\
    \vdots&\ddots& &\vdots\\
  \vdots& & 0_{d}&0_{d}\\
    0_{d}&   \cdots&   0_{d}&I_{d}
\end{pmatrix}. 
 \end{align*}
Meanwhile, the $P$-th order Langevin dynamics \eqref{equation_Pthorderlangevin}
can be written as 
\begin{align*}
dx(t)=b(x(t))dt+\sqrt{2\gamma}DdB_t,\qquad  b(x)= \begin{pmatrix}
        v_1(t)\\
        -\nabla U\brac{\theta(t)}+\gamma v_2(t)\\
        -\gamma v_1(t)+\gamma v_3(t)\\
        \vdots\\
        -\gamma v_{P-3}(t)+ \gamma v_{P-1}(t)\\
        -\gamma v_{P-2}(t)- \gamma v_{P-1}(t)
    \end{pmatrix}. 
\end{align*}


Then
\begin{align*}
    d\brac{x(t)-x^{\operatorname{st}_{P-1 }}(t)}=\brac{b(x(t))-b\brac{x^{\operatorname{st}_{P-1 }}}}dt+\brac{b\brac{x^{\operatorname{st}_{P-1 }}(t)}-\bar{b}(t)}dt, 
\end{align*}
where
\begin{align}
\label{equation_difference_Pthorder}
    b\brac{x^{\operatorname{st}_{P-1 }}(t)}-\bar{b}(t)= \begin{pmatrix}
        0\\
        \brac{g^{\operatorname{st}_{P-1 }}(t)- \nabla U\brac{\theta^{\operatorname{st}_{P-1 }}(t)}}+\gamma\brac{v^{\operatorname{st}_{P-1 }}_2(t)-v^{\operatorname{st}_{P-2 }}_2(t) }\\
        \gamma\brac{v^{\operatorname{st}_{P-1 }}_3(t)-v^{\operatorname{st}_{P-2 }}_3(t)} \\
        \vdots\\
        \gamma\brac{v^{\operatorname{st}_{P-1 }}_{P-1}(t)-v^{\operatorname{st}_{P-2 }}_{P-1}(t)}\\
       0
    \end{pmatrix}. 
\end{align}

This leads to
\begin{align}
\label{error_firststep_Pthorder}
    &\frac{d}{dt}\E{\brac{x(t)-x^{\operatorname{st}_{P-1 }}(t)}^{\top} M \brac{x(t)-x^{\operatorname{st}_{P-1 }}(t)}}\nonumber\\
    &=\E{\brac{x(t)-x^{\operatorname{st}_{P-1 }}(t)}^{\top} M\brac{b(x(t))-b(x^{\operatorname{st}_{P-1 }}(t)}}\nonumber\\
    &\qquad+\E{\brac{x(t)-x^{\operatorname{st}_{P-1 }}(t)}^{\top} M\brac{b(x^{\operatorname{st}_{P-1 }}(t))-\bar{b}(t)} }\nonumber\\
    &\qquad\quad+\E{\brac{b(x(t))-b(x^{\operatorname{st}_{P-1 }}(t))}^{\top} M\brac{x(t)-x^{\operatorname{st}_{P-1 }}(t)}}\nonumber\\
    &\qquad\qquad+\E{\brac{b(x^{\operatorname{st}_{P-1 }}(t))-\bar{b}(t)}^{\top} M \brac{x(t)-x^{\operatorname{st}_{P-1 }}(t)}}. 
\end{align}

Notice that 
\begin{align}
\label{error_secondstep_Pthorder}
    &\brac{x(t)-x^{\operatorname{st}_{P-1 }}(t)}^{\top} M\brac{b(x(t))-b(x^{\operatorname{st}_{P-1 }}(t)}+\brac{b(x(t))-b(x^{\operatorname{st}_{P-1 }}(t))}^{\top} M\brac{x(t)-\bar{x}(t)}\nonumber\\
    &\leq \brac{x(t)-\bar{x}(t)}^{\top} M\int_0^1 J_b\brac{wx(t)+(1-w)\bar{x}(t)}\brac{x(t)-x^{\operatorname{st}_{P-1 }}(t)}dw\nonumber\\
    &+\int_0^1 \brac{x(t)-x^{\operatorname{st}_{P-1 }}(t)}^{\top} J_b\brac{wx(t)+(1-w)\bar{x}(t)}^{\top}dw M \brac{x(t)-x^{\operatorname{st}_{P-1 }}(t)}\nonumber\\
    &\leq \brac{x(t)-x^{\operatorname{st}_{P-1 }}(t)}^{\top} (-2\rho)M \brac{x(t)-x^{\operatorname{st}_{P-1 }}(t)}=-2\rho\abs{x(t)-\bar{x}(t)}^2_M,
\end{align}
where the last line is due to the contraction property  \eqref{contraction} in Theorem~\ref{theorem_frommonmarche_mainpaper}.

Moreover, choose any positive $a\leq \min \left\{\frac{m}{3\gamma}, \frac{\gamma \kappa}{6} \right\}\lambda_{\min,M}$ and notice that per Theorem~\ref{theorem_frommonmarche_mainpaper},
\begin{align}
\label{choiceofa_Pthorder}
    -\rho+\frac{a}{\lambda_{\min,M}}<0. 
\end{align}

From \eqref{error_firststep_Pthorder}, \eqref{error_secondstep_Pthorder}, \eqref{choiceofa_Pthorder} and Cauchy-Schwarz inequality, we can deduce that
\begin{align}
 \frac{d}{dt} \E{ \abs{x(t)-x^{\operatorname{st}_{P-1 }}(t)}^2_M}
 &\leq -2\rho \E{\abs{x(t)-x^{\operatorname{st}_{P-1 }}(t)}^2_M+ 2a\abs{x(t)-x^{\operatorname{st}_{P-1 }}(t)}^2}\nonumber\\
 &\qquad\qquad+\frac{2}{a}\norm{M}^2_{\operatorname{op}}\brac{\gamma^2 +1}\E{\abs{b(x^{\operatorname{st}_{P-1 }}(t))-\bar{b}(t)}^2}.\label{previous:bound:2} 
\end{align}
By combining the bound in \eqref{previous:bound:2} with \eqref{equivalenceMnorm},  we get
\begin{align}
\label{error_beforeboundingtheextraterm_Pthorder}
 &\frac{d}{dt}  \E{\abs{x(t)-x^{\operatorname{st}_{P-1 }}(t)}^2_M}\nonumber\\
 &\leq \brac{-2\rho+\frac{2a}{\lambda_{\min,M}} } \E{\abs{x(t)-x^{\operatorname{st}_{P-1 }}(t)}^2_M}+\frac{2}{a}\norm{M}^2_{\operatorname{op}}\brac{\gamma^2+1} \E{\abs{b(x^{\operatorname{st}_{P-1 }}(t))-\bar{b}(t)}^2}. 
\end{align}

\textbf{Step 2:} We will bound $\E{\abs{b(x^{\operatorname{st}_{P-1 }}(t))-\bar{b}(t)}^2}$ as the second term on the right hand side of \eqref{error_beforeboundingtheextraterm_Pthorder}. Based on \eqref{equation_difference_Pthorder}, we will need to bound the $L^2$ norm of 
\begin{align*}
  g^{\operatorname{st}_{P-1 }}(t)- \nabla U\brac{\theta^{\operatorname{st}_{P-1 }}(t)},\,\, \gamma\brac{v^{\operatorname{st}_{P-1 }}_2(t)-v^{\operatorname{st}_{P-2 }}_2(t) },\,\,\gamma\brac{v^{\operatorname{st}_{P-1 }}_{P-1}(t)-v^{\operatorname{st}_{P-2 }}_{P-1}(t)}.
\end{align*}
Per estimate~\eqref{difference_laststage_Pmorethan3} Lemma~\ref{lemma_stagedifference_Pthorder} and in the case $P\geq 4$ (the case $P=3$ is similar and is handled at the end of this proof), we have 
\begin{align}
\label{1stterm}
  &  \E{\abs{\gamma\brac{v^{\operatorname{st}_{P-1 }}_2(t)-v^{\operatorname{st}_{P-2 }}_2(t) }}^2}+\E{\abs{\gamma\brac{v^{\operatorname{st}_{P-1 }}_{P-1}(t)-v^{\operatorname{st}_{P-2 }}_{P-1}(t)}}^2}\nonumber\\
   &\leq \gamma^2 \brac{C^{\operatorname{st}_{P-1 }}_2+C^{\operatorname{st}_{P-1 }}_{P-1} }d\eta^{2P-1}.  
\end{align}
Meanwhile, 
\begin{align*}
    &\E{\abs{ g^{\operatorname{st}_{P-1 }}(t)- \nabla U\brac{\theta^{\operatorname{st}_{P-1 }}(t)}}^2}\\
    &\leq 2\E{\abs{ g^{\operatorname{st}_{P-1 }}(t)- \nabla U\brac{\theta^{\operatorname{st}_{P-2 }}(t)}}^2}+2\E{\abs{\nabla U\brac{\theta^{\operatorname{st}_{P-2 }}(t)}-\nabla U\brac{\theta^{\operatorname{st}_{P-1 }}(t)}}^2}. 
\end{align*}
The same argument as the one in \eqref{bound_polyapproxtildetheta} yields
\begin{align*}
    \E{\abs{ g^{\operatorname{st}_{P-1 }}(t)- \nabla U\brac{\theta^{\operatorname{st}_{P-2 }}(t)}}^2}\leq c  d\eta^{2P-1},
\end{align*}
for some positive constant $c$. Moreover, $L$-smoothness of $U$ in Condition~\ref{cond_mainpaper} and estimate~\eqref{difference_laststage_Pmorethan3} of Lemma~\ref{lemma_stagedifference_Pthorder} imply
 \begin{align*}
     \E{\abs{\nabla U\brac{\theta^{\operatorname{st}_{P-2 }}(t)}-\nabla U\brac{\theta^{\operatorname{st}_{P-1 }}(t)}}^2}\leq L^2 D^{\operatorname{st}_{P-1 }}d\eta^{2P-1}. 
 \end{align*}
The last three bounds lead to 
\begin{align}
    \label{2ndterm}
     &\E{\abs{ g^{\operatorname{st}_{P-1 }}(t)- \nabla U\brac{\theta^{\operatorname{st}_{P-1 }}(t)}}^2}\leq  \brac{L^2 D^{\operatorname{st}_{P-1 }}+1}d\eta^{2P-1}. 
\end{align}
The combination of \eqref{equation_difference_Pthorder}, \eqref{1stterm} and \eqref{2ndterm} lead to 
\begin{align*}
    \E{\abs{b(x^{\operatorname{st}_{P-1 }}(t))-\bar{b}(t)}^2}\leq \widetilde{C}_1d\eta^{2P-1}, \quad P\geq 4,
\end{align*}
where $\widetilde{C}_1$ is a generic positive constant that depends only on $c,\gamma,L,P$ and can change from line to line. Then from \eqref{error_beforeboundingtheextraterm_Pthorder}, we get
\begin{align}
\label{estimate_beforediscretization}
  &  \frac{d}{dt}  \E{\abs{x(t)-x^{\operatorname{st}_{P-1 }}(t)}^2_M}\nonumber\\
 &\leq \brac{-2\rho+\frac{2a}{\lambda_{\min,M}} } \E{\abs{x(t)-x^{\operatorname{st}_{P-1 }}(t)}^2_M}+\frac{2}{a}\norm{M}^2_{\operatorname{op}}\brac{\gamma^2 +1}\widetilde{C}_1d\eta^{2P-1} ,\quad P\geq 4. 
\end{align}

\textbf{Step 3:} Solving the differential inequality \eqref{estimate_beforediscretization} by integrating factors as in the proof of Proposition~\ref{prop_discretizationerror_fourthorder}, we arrive at the desired discretization error in the case $P\geq 4$. 

\textbf{Step 4:} As the last part of this proof and in the case $P=3$, we follow a similar path and use estimate~\eqref{difference_laststage_Pis3} in Lemma~\ref{lemma_stagedifference_Pthorder} (instead of \eqref{difference_laststage_Pmorethan3}) to get an analogous inequality to \eqref{estimate_beforediscretization} that is
\begin{align}
\label{estimate_beforediscretization_thirdorder}
  &  \frac{d}{dt}  \E{\abs{x(t)-x^{\operatorname{st}_{P-1 }}(t)}^2_M}\nonumber\\
 &\leq \brac{-2\rho+\frac{2a}{\lambda_{\min,M}} } \E{\abs{x(t)-x^{\operatorname{st}_{2 }}(t)}^2_M}+\frac{2}{a}\norm{M}^2_{\operatorname{op}}\brac{\gamma^2 +1}\widetilde{C}_1d\eta^{4}. 
\end{align}
Then by solving the differential inequality~\eqref{estimate_beforediscretization_thirdorder} by integrating factors as in the proof of Proposition~\ref{prop_discretizationerror_fourthorder}, we arrive at the desired discretization error in the case $P=3$. The proof is complete.
\end{proof}

\begin{proof}[Proof of Theorem~\ref{theorem_mixingtime_Pthorder}]
The argument is the same as the proof of Theorem~\ref{theorem_mixingtime_4thorder} at the end of Section~\ref{section_proof_fouthorder}. We have $\operatorname{Wass}_2\brac{\operatorname{Law}(X),\operatorname{Law}(Y)}\leq \E{\abs{X-Y}^2}^{1/2}$. Then per Proposition~\ref{prop_discretizationerror_Pthorder}, we can solve for $\eta$ in  $\widetilde{C}_3d\eta^{4\mathds{1}_{\{P=3 \} }+(2P-1)\mathds{1}_{\{P\geq 4 \} } }\leq \epsilon^2/2$ and for $k$ in $\widetilde{C}_4e^{-(k+1)h\eta}\mathbb{E}_{Z\sim\mu}\left[\left|Z-x^{(0)}\right|^2\right]\leq \epsilon^2/2$ to obtain the desired mixing time. 
This completes the proof.
\end{proof}


\section{Numerical Experiments}\label{sec:numerical}


In this section, we will implement both third-order
and fourth-order LMC algorithms.
From Section~\ref{section_4thorder}, we recall the fourth-order LMC algorithm samples a multivariate normal distribution at every step, where mean and covariance are provided in Lemma~\ref{lemma_explicitformofxbar} and Lemma~\ref{lemma_meanandcovariance} in Appendix~\ref{appendix_4thorder}. The mean in particular contains several nested integrals that need to be exactly computed when the loss function $U$ is a polynomial, and approximated in the case where the loss function $U$ is not a polynomial. We provide the calculations related to these nested integrals for quadratic loss and logistic loss in Appendix~\ref{app:extra}, which allow us to perform the numerical experiments for our fourth-order LMC algorithm.  

In addition, we will provide some calculations necessary to perform the numerical experiments for the third-order LMC algorithm in~\cite{mou2021high} for quadratic loss.

When the potential function $U(\theta)$ satisfies Condition~\ref{cond_mainpaper}, then for a small stepsize $\eta>0$ and two arbitrary friction parameters $\gamma>0$ and $\xi>0$, the third-order Langevin Monte Carlo algorithm is given as follows:
\\
\begin{tabular}{l}
     \toprule
     \textbf{Algorithm 1}: Third-Order Langevin Monte Carlo Algorithm\label{alg:alg1}\\
     \midrule 
     \hspace{4mm} Let $x^{(0)}=\left(\theta^{(0)},v_1^{(0)},v_2^{(0)}\right)=(\theta^*,0,0)$\\
     \hspace{4mm} \textbf{for} $k=0,1,\cdots, N-1$ \textbf{do}\\
     \hspace{8mm} Sample $x^{(k+1)}\sim \mathcal{N}\left(\bs{\mu}(x^{(k)}),\bs{\Sigma}\right)$, where $\bs{\mu}$ and $\bs{\Sigma}$ are defined in the following\\  \hspace{8mm} equations\\
     \hspace{4mm} \textbf{end for} \\
     \bottomrule
\end{tabular}\hfill\\
\\
The update of the states $x$ from step $k$ to $k+1$ is obtained by drawing from the distribution with mean $\bs{\mu}\left(x^{(k)}\right)$ and covariance $\bs{\Sigma}$:
\begin{equation}
    \bs{\mu}(x):  = \begin{pmatrix}\theta-\frac{\eta}{2L}\Delta U(\theta,v_1)+\mu_{12}v_1+\mu_{13}v_2\\-\frac{1}{L}\Delta U(\theta,v_1) +\mu_{22}v_1+\mu_{23}v_2\\
    \frac{\mu_{31}}{L}\Delta U(\theta,v_1)+\mu_{32}v_1+\mu_{33}v_2
    \end{pmatrix},
    \bs{\Sigma}: = \begin{pmatrix}
        \sigma_{11}\cdot I_{d}&\sigma_{12}\cdot I_{d}&\sigma_{13}\cdot I_{d}\\
        \sigma_{12}\cdot I_{d}&\sigma_{22}\cdot I_{d}&\sigma_{23}\cdot I_{d}\\
        \sigma_{13}\cdot I_{d}&\sigma_{23}\cdot I_{d}&\sigma_{33}\cdot I_{d}
    \end{pmatrix},
\label{eq:o3}
\end{equation}
where all $\mu$'s and $\sigma$'s are defined in the article by \cite{mou2021high}. 
Now we present the fourth-order Langevin Monte Carlo algorithm as follows:\\
\\
\begin{tabular}{l}
     \toprule
     \textbf{Algorithm 2}: Fourth-Order Langevin Monte Carlo Algorithm\label{alg:alg2}\\
     \midrule 
     \hspace{4mm} Let $x^{(0)}=\left(\theta^{(0)},v_1^{(0)},v_2^{(0)},v_3^{(0)}\right)=(\theta^*,0,0,0)$\\
     \hspace{4mm} \textbf{for} $k=0,1,\cdots, N-1$ \textbf{do}\\
     \hspace{8mm} Sample $x^{(k+1)}\sim \mathcal{N}\left(\mb{m}(x^{(k)}),\bs{\Sigma}\right)$, where $\mb{m}$ and $\bs{\Sigma}$ are defined in the following\\ \hspace{8mm} equation (\ref{eq:o4})\\
     \hspace{4mm} \textbf{end for} \\
     \bottomrule
\end{tabular}\hfill\\
\\
\\
The update of the state $x$ from step $k$ to $k+1$ is obtained by drawing a sample from the multivariate Gaussian distribution with mean $\mb{m}(x)$
and covariance $\bs{\Sigma}$ given by:
\begin{equation}
        \mb{m}(x):= \begin{pmatrix}
                    m_0 \\ m_1 \\ m_2 \\ m_3
                \end{pmatrix},
    \qquad\bs{\Sigma}:= \begin{pmatrix}
                            \sigma_{00}\cdot I_{d} & \sigma_{01}\cdot I_{d} & \sigma_{02}\cdot I_{d} & \sigma_{03}\cdot I_{d} \\
                            \sigma_{01}\cdot I_{d} &
                            \sigma_{11}\cdot I_{d} & 
                            \sigma_{12}\cdot I_{d} &
                            \sigma_{13}\cdot I_{d} \\
                            \sigma_{02}\cdot I_{d} &
                            \sigma_{12}\cdot I_{d} &
                            \sigma_{22}\cdot I_{d} & 
                            \sigma_{23}\cdot I_{d} \\
                            \sigma_{03}\cdot I_{d} & 
                            \sigma_{13}\cdot I_{d} & 
                            \sigma_{23}\cdot I_{d} & 
                            \sigma_{33}\cdot I_{d}
                    \end{pmatrix},
\label{eq:o4}
\end{equation}
where the explicit formulas to compute $m_i\in \mathbb{R}^d$ and $\sigma_{ij}\in \mathbb{R}$ are given in Lemma~\ref{lemma_meanandcovariance} and the calculations for particular loss functions are given in Appendix~\ref{app:extra}. Note that both algorithms require the initialization of the model parameter $\theta^*$. \cite{mou2021high} recommended that $\theta^*$ can be chosen from the exact solution when $U$ is a polynomial. However, we initialize the sampling process randomly from the standard normal distribution, which leads to superior performance. 

\subsection{Bayesian linear regression}
\label{section_linearegression}

We conduct experiments using our algorithms for Bayesian linear regression-type problems using the \textbf{Air Quality} data from the \textit{UCI Machine Learning Repository} \cite{air_quality_360}. It contains sensor readings from an array of chemical sensors deployed in an Italian city to monitor air pollution. Collected between March 2004 and February 2005. The dataset includes hourly measurements of key pollutants such as carbon monoxide (CO), non-methane hydrocarbons (NMHC), benzene, nitrogen oxides ($\text{NO}_x$), and ozone ($\text{O}_3$), along with meteorological variables such as temperature and relative humidity. The dataset is often used for regression tasks to model air quality indicators, particularly predicting CO concentration based on other environmental variables. It presents challenges such as missing values and sensor drift, making it suitable for testing robust data pre-processing and modeling techniques.

In this experiment, our goal is to sample the posterior distribution of the model parameters that regress the concentration of CO present in the air. After pre-processing. The feature matrix has $d=16$ dimensions (including the intercept term) and a total of 7,674 observations.

We consider an arbitrary prior of $\theta$ from $\mathcal{N}(0,10 I)$. The known posterior for the linear regression problem is given as follows:
\begin{equation}
    \mu(\theta) \sim \mathcal{N}(\mathbf{m}, \mathbf{V}); \hspace{4mm} \mathbf{m} := \left(\Sigma^{-1}+\frac{X^{\top}X}{\xi^2}\right)^{-1}\left(\frac{X^{\top}y}{\xi^2}\right), \hspace{4mm} \mathbf{V} := \left(\bs{\Sigma}^{-1}+\frac{X^{\top}X}{\xi^2}\right)^{-1},
    \label{eq:post1}
\end{equation}
where $X$ and $y$ are input data-matrix and output vector, respectively, and $\bs{\Sigma} = \lambda I_d$ is the covariance matrix with the Ridge regularization ($L_2$) parameter $\lambda$, in this experiment, we choose a smaller penalty $\lambda=2$.

To draw a sample from the posterior, at each iteration we perform a Cholesky decomposition to factor the covariance matrix into a lower and upper triangular matrix $\bs{\Sigma} = LL^{\top}$ and use the formula \cite{wang2006generating}:
\begin{equation*}
    x^{(k+1)} = \bs{\mu}\left(x^{(k)}\right) + Lu;\hspace{4mm}\text{or }\hspace{4mm} x^{(k+1)} = \mb{m}\left(x^{(k)}\right) + Lu,
\end{equation*}
where $u\in\mathbb{R}^{3d}$ (or $\mathbb{R}^{4d}$) with $u\sim \mathcal{N}(0,I)$ and $\bs{\mu}\left(x^{(k)}\right)$ (or $\mb{m}\left(x^{(k)}\right)$) is the mean vector at $k$-th iterate of the respective algorithm. Note that the covariance matrix $\bs{\Sigma}$ needs to be symmetric positive definite (SPD) in order to factor it using Cholesky decomposition. To ensure that we get an SPD matrix for arbitrarily chosen $\gamma, \xi,$ and $\eta$ values, we add a small jitter ($10^{-6}$) to the covariance matrix. Then we perform a grid search to find the optimal hyperparameters based on the lowest mean 2-Wasserstein distance, computed using the formula from \cite{givens1984class}.

\begin{figure}[htbp]
  \centering

  \begin{minipage}[b]{0.3\textwidth}
    \includegraphics[width=\linewidth]{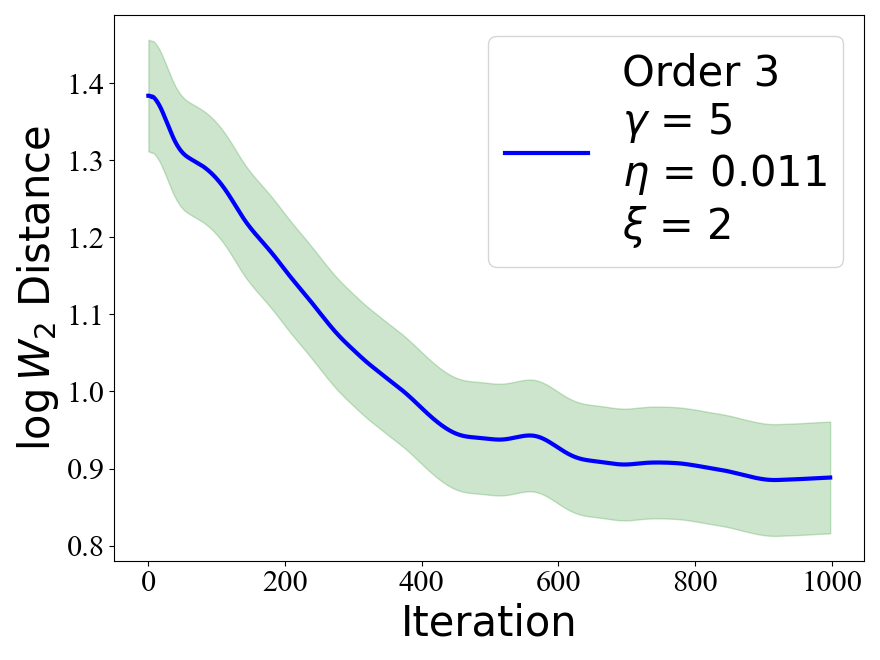}
    \caption*{a. $\mathcal{W}_2$ distance from the 3rd-order LMC}
  \end{minipage}
  \hfill
  \begin{minipage}[b]{0.3\textwidth}
    \includegraphics[width=\linewidth]{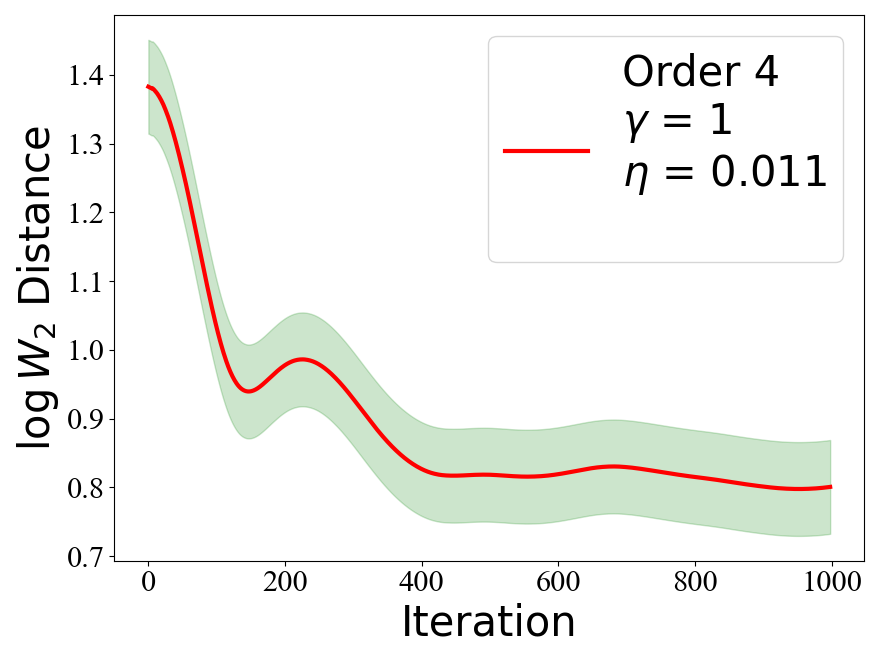}
    \caption*{b. $\mathcal{W}_2$ distance from the 4th-order LMC}
  \end{minipage}
  \hfill
  \begin{minipage}[b]{0.3\textwidth}
    \includegraphics[width=\linewidth]{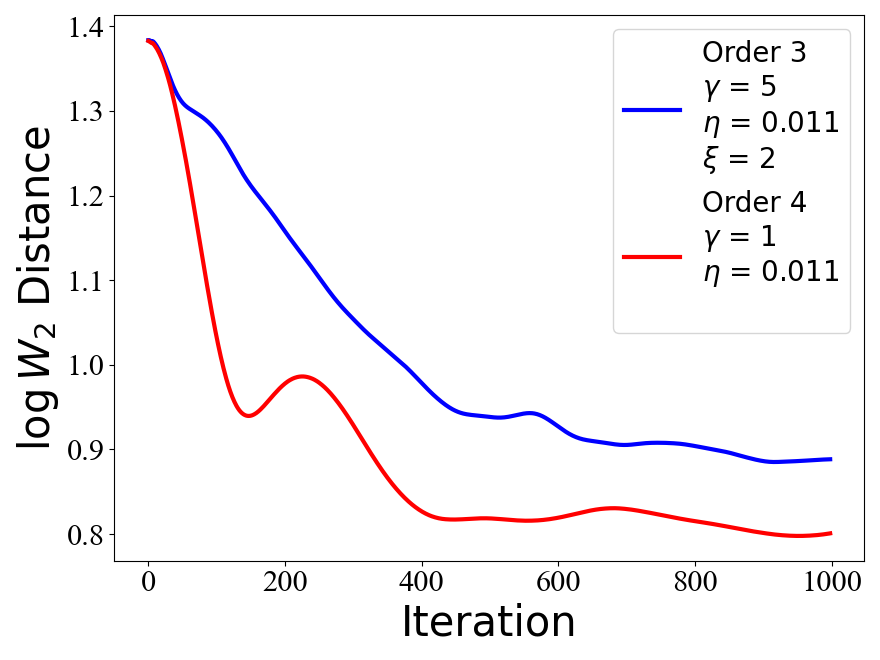}
    \caption*{c. $\mathcal{W}_2$ distances from the 3rd- and 4th- order LMC}
  \end{minipage}

  \caption{Comparative performance of the 3rd- and 4th-order Langevin Monte Carlo algorithms}
  \label{fig:reg1}
\end{figure}

The tuned hyperparameters for the third-order Langevin dynamics $\gamma = 5, \eta = 0.011,$ and $\xi=2$, and for the fourth-order Langevin dynamics $\gamma = 1$ and $\eta=0.011$. For both dynamics, we draw $N=1,000$ samples from the posterior distribution and compute the $\mathcal{W}_2$ (\textit{2-Wasserstein}) distance from the known posterior defined in (\ref{eq:post1}). The shaded region both in Figure~\ref{fig:reg1}a. and \ref{fig:reg1}b. represent half of the standard deviation in \textit{2-Wasserstein} distances. The relative performances of the third- and fourth-order LMC algorithms are presented in Figure~\ref{fig:reg1}c. From this set of experiments, we notice that the convergence to the posterior distribution is better for the 4th-order LMC algorithm than that of the 3rd-order LMC algorithm for a given stepsize $\eta$.\\
\begin{figure}[htbp]
  \centering

  \begin{minipage}[b]{0.48\textwidth}
    \includegraphics[width=\linewidth]{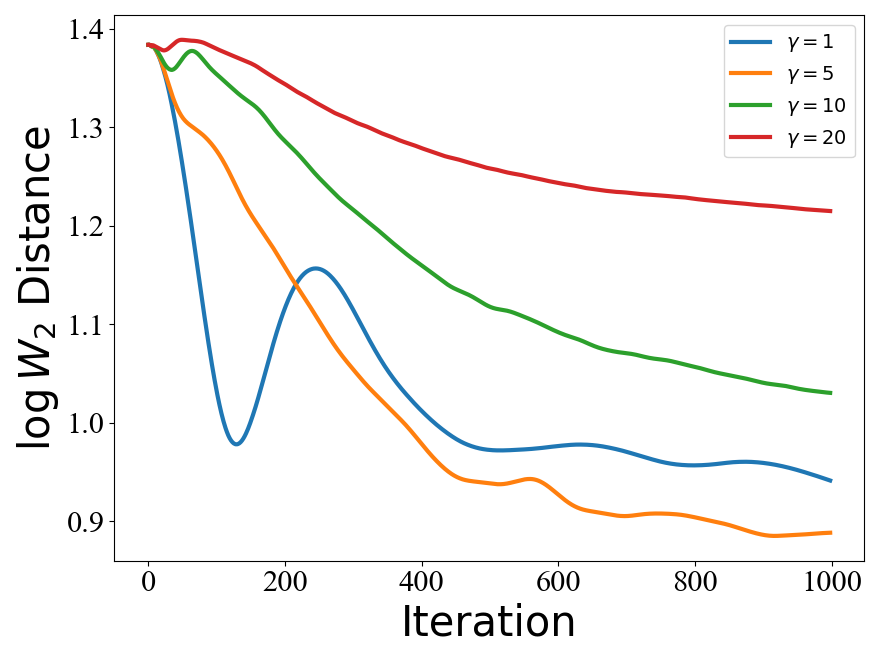}
    \caption*{a. Change in $\mathcal{W}_2$ distance from the 3rd-order LMC algorithm for varying $\gamma$}
  \end{minipage}
  \hfill
  \begin{minipage}[b]{0.48\textwidth}
    \includegraphics[width=\linewidth]{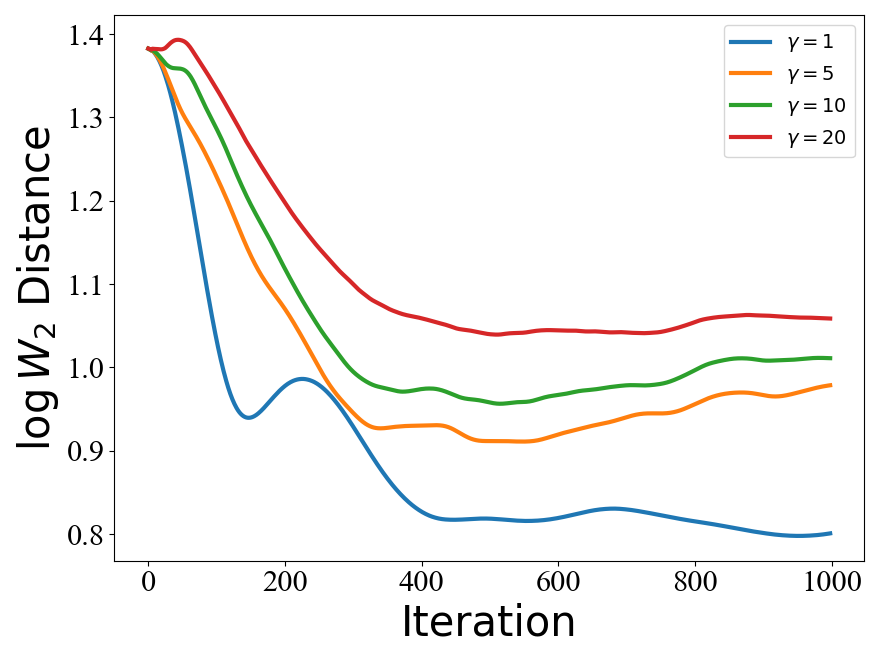}
    \caption*{b. Change in $\mathcal{W}_2$ distance from the 4th-order LMC algorithm for varying $\gamma$}
  \end{minipage}
  \caption{Comparative performance of the 3rd- and 4th-order Langevin Monte Carlo algorithms for the same stepsize $\eta$ and varying the friction parameter $\gamma$}
  \label{fig:reg2}
\end{figure}

Next, we present the effect of the variation of the friction parameter $\gamma$ in Figure~\ref{fig:reg2}. For the same stepsize $\eta$, we observe that the 4th-order LMC algorithm provides better convergence in terms of the \textit{2-Wasserstein} distance for smaller $\gamma$ values. However, this is not always the case for the 3rd-order LMC algorithm.

\subsection{Bayesian logistic regression}
\label{section_logisticregression}

In this section, we provide the implementation of the 4th-order LMC algorithm for sampling in a classification problem. To implement the 4th-order LMC algorithm efficiently, one needs to approximate the gradient of the potential function in higher-degree polynomials. The last step can be done via softwares per Remark~\ref{remark_softwarefortaylorpoly}; however this complicates the implementation of our algorithm. Therefore, we arbitrarily choose third-degree polynomials to approximate the gradient of the logistic loss function using a Taylor polynomial which is given in Appendix~\ref{dis:logloss}.


We choose the \textbf{Mushroom} dataset from the \textit{UCI Machine Learning Repository} \footnote{Mushroom. UCI Machine Learning Repository, 1981. DOI: https://doi.org/10.24432/C5959T}. The dataset contains 8,124 instances of gilled mushrooms, each described by 22 categorical features such as cap shape, odor, gill color, and habitat. After pre-processing (e.g., OneHotEncoding for categorical features). The final input dataset has a dimension $d=118$, and a total of 8,124 observations.

The typical objective with this data is to build a machine learning model that classifies whether a mushroom is edible or poisonous based on the given attributes. However, our goal in this experiment is not to find the optimal model; rather, we sample the model parameters and see how the accuracy measure varies as we increase the number of samples from the posterior distribution of the model parameters.

\begin{figure}[htbp]
  \centering

  \begin{minipage}[b]{0.48\textwidth}
    \includegraphics[width=\linewidth]{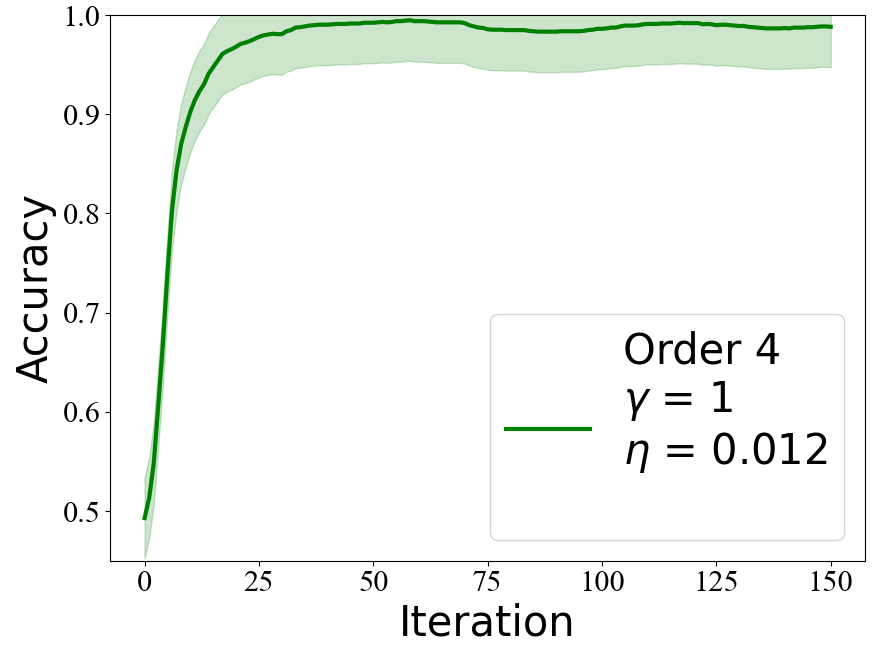}
    \caption*{a. Accuracy from the 4th-order LMC algorithm}
  \end{minipage}
  \hfill
  \begin{minipage}[b]{0.48\textwidth}
    \includegraphics[width=\linewidth]{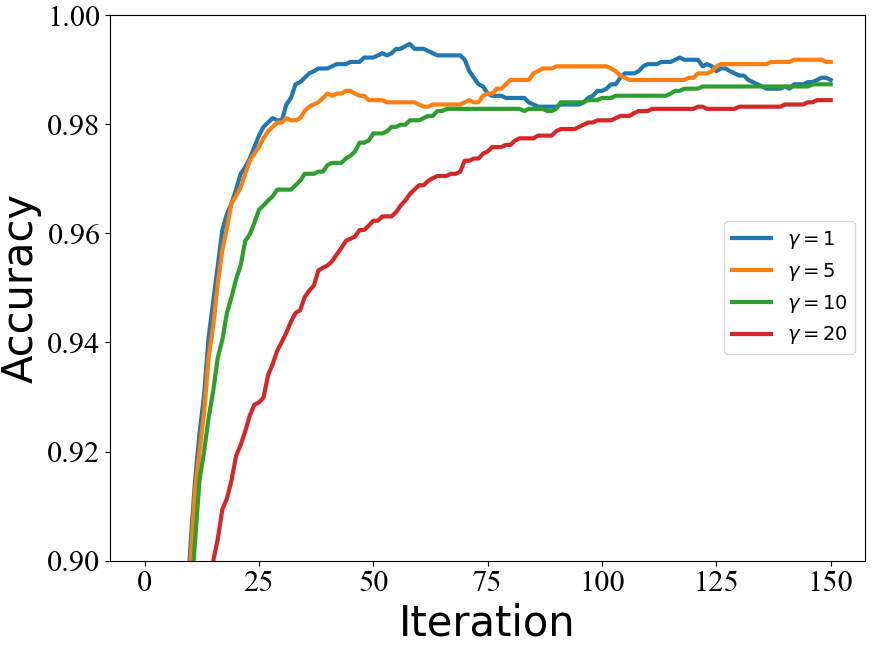}
    \caption*{b. Accuracy for different $\gamma$ values (for the a fixed stepsize $\eta=0.012$)}
  \end{minipage}
  \caption{Performance of the 4th-order LMC algorithm in sampling from a non-polynomial potential function}
  \label{fig:class1}
\end{figure}

Before starting the sampling process, we split the data into 70-30 training and testing ratio and check the sample quality on the test set. We generate $N=150$ samples of the model parameters and run a grid search for the hyperparameters $\eta$ and $\gamma$. To avoid overfitting, we use a larger penalty $\lambda=25$.

From Figure~\ref{fig:class1} (a), we see that the 4th-order LMC algorithm performs very well even for smaller degree polynomial approximation of the gradient of the potential function. We tune the model parameters $\eta=0.012$ and $\gamma=1$. Then we show the effect of the variation in the friction parameter for a chosen stepsize $\eta=0.012$ in Figure~\ref{fig:class1} (b). We see that smaller $\gamma$ values result in better performance in terms of higher accuracy. 

\section{Conclusion}\label{sec:conclusion}

In this paper, we proposed $P$-th order Langevin Monte Carlo algorithms based
on the discretizations of $P$-th order Langevin dynamics for any $P\geq 3$. We designed discretization schemes 
based on splitting and accurate integration methods. 
When the density of the target distribution is log-concave
and smooth, we obtained Wasserstein convergence guarantees
that lead to better iteration complexities. 
Specifically, the mixing time of the $P$-th order LMC algorithm scales as $O\brac{d^{\frac{1}{\mathcal{R}}}/\epsilon^{\frac{1}{2\mathcal{R}}} }$ for $ \mathcal{R}=4\cdot\mathds{1}_{\{ P=3\}}+ (2P-1)\cdot\mathds{1}_{\{ P\geq 4\}}$, which has a better dependence on the dimension $d$ and the accuracy level $\epsilon$ as $P$ grows.
Numerical experiments were conducted to illustrate the efficiency of our proposed algorithms.

\section*{Acknowledgments}

Mert G\"{u}rb\"{u}zbalaban's research is supported in part by the grants Office of Naval Research Award Number
N00014-21-1-2244, National Science Foundation (NSF)
CCF-1814888, NSF DMS-2053485.  
Mohammad Rafiqul Islam is partially supported by the grant NSF DMS-2053454.
Nian Yao was supported in part by the Natural Science Foundation of China under Grant 12071361, the Natural Science Foundation of Guangdong Province under Grant 2020A1515010822
and Shenzhen Natural Science Fund (the Stable Support Plan Program 20220810152104001).
Lingjiong Zhu is partially supported by the grants NSF DMS-2053454 and DMS-2208303.

\bibliographystyle{alpha}
\bibliography{refs}

\newcommand{\etalchar}[1]{$^{#1}$}
\begin{thebibliography}{MMW{\etalchar{+}}21}

\bibitem[ADFDJ03]{andrieu2003introduction}
Christophe Andrieu, Nando De~Freitas, Arnaud Doucet, and Michael~I Jordan.
\newblock An introduction to {MCMC} for machine learning.
\newblock {\em Machine Learning}, 50(1):5--43, 2003.

\bibitem[AE14]{arnolderb2014sharpentropy}
Anton Arnold and Jan Erb.
\newblock Sharp entropy decay for hypocoercive and non-symmetric
  {F}okker-{P}lanck equations with linear drift.
\newblock {\em arXiv preprint arXiv:1409.5425}, 2014.

\bibitem[AJW20]{arnold2020sharp}
Anton Arnold, Shi Jin, and Tobias W{\"o}hrer.
\newblock Sharp decay estimates in local sensitivity analysis for evolution
  equations with uncertainties: From {ODE}s to linear kinetic equations.
\newblock {\em Journal of Differential Equations}, 268(3):1156--1204, 2020.

\bibitem[BCE{\etalchar{+}}22]{BCESZ2022}
Krishna Balasubramanian, Sinho Chewi, Murat~A Erdogdu, Adil Salim, and Shunshi
  Zhang.
\newblock Towards a theory of non-log-concave sampling: First-order
  stationarity guarantees for {L}angevin {M}onte {C}arlo.
\newblock In {\em Proceedings of Thirty Fifth Conference on Learning Theory},
  volume 178, pages 2896--2923. PMLR, 2022.

\bibitem[BCM{\etalchar{+}}21]{Barkhagen2021}
Mathias Barkhagen, Ngoc~Huy Chau, \'{E}ric Moulines, Mikl\'{o}s R\'{a}sonyi,
  Sotirios Sabanis, and Ying Zhang.
\newblock On stochastic gradient {L}angevin dynamics with dependent data
  streams in the logconcave case.
\newblock {\em Bernoulli}, 27(1):1--33, 2021.

\bibitem[Car71]{cartan1971differentialbook}
Henri Cartan.
\newblock {\em Differential Calculus}.
\newblock Hermann, 1971.

\bibitem[CB18]{CB2018}
Xiang Cheng and Peter~L. Bartlett.
\newblock Convergence of {L}angevin {MCMC} in {KL}-divergence.
\newblock In {\em Proceedings of the 29th International Conference on
  Algorithmic Learning Theory (ALT)}, volume~83, pages 186--211. PMLR, 2018.

\bibitem[CCA{\etalchar{+}}18]{cheng-nonconvex}
Xiang {Cheng}, Niladri~S. {Chatterji}, Yasin {Abbasi-Yadkori}, Peter~L.
  {Bartlett}, and Michael~I. {Jordan}.
\newblock {Sharp Convergence Rates for {L}angevin Dynamics in the Nonconvex
  Setting}.
\newblock {\em arXiv:1805.01648}, 2018.

\bibitem[CCBJ18]{cheng2018underdamped}
Xiang Cheng, Niladri~S Chatterji, Peter~L Bartlett, and Michael~I Jordan.
\newblock Underdamped {L}angevin {MCMC}: A non-asymptotic analysis.
\newblock In {\em Conference on learning theory}, pages 300--323. PMLR, 2018.

\bibitem[CHS87]{chiang1987diffusion}
Tzuu-Shuh Chiang, Chii-Ruey Hwang, and Shuenn~Jyi Sheu.
\newblock Diffusion for global optimization in $\mathbb{R}^n$.
\newblock {\em SIAM Journal on Control and Optimization}, 25(3):737--753, 1987.

\bibitem[CLW21]{CLW2020}
Yu~Cao, Jianfeng Lu, and Lihan Wang.
\newblock Complexity of randomized algorithms for underdamped {L}angevin
  dynamics.
\newblock {\em Communications in Mathematical Sciences}, 19(7):1827--1853,
  2021.

\bibitem[CLW23]{JianfengLu}
Yu~Cao, Jianfeng Lu, and Lihan Wang.
\newblock On explicit {$L^{2}$}-convergence rate estimate for underdamped
  {L}angevin dynamics.
\newblock {\em Archive for Rational Mechanics and Analysis}, 247(90):1--34,
  2023.

\bibitem[CMR{\etalchar{+}}21]{Chau2019}
Ngoc~Huy Chau, \'{E}ric Moulines, Miklos R\'{a}sonyi, Sotirios Sabanis, and
  Ying Zhang.
\newblock On stochastic gradient {L}angevin dynamics with dependent data
  streams: the fully non-convex case.
\newblock {\em SIAM Journal of Mathematics of Data Science}, 3(3):959--986,
  2021.

\bibitem[Dal17]{Dalalyan}
Arnak~S Dalalyan.
\newblock Theoretical guarantees for approximate sampling from smooth and
  log-concave densities.
\newblock {\em Journal of the Royal Statistical Society: Series B (Statistical
  Methodology)}, 79(3):651--676, 2017.

\bibitem[DK19]{DK2017}
Arnak~S. Dalalyan and Avetik~G. Karagulyan.
\newblock User-friendly guarantees for the {L}angevin {M}onte {C}arlo with
  inaccurate gradient.
\newblock {\em Stochastic Processes and their Applications},
  129(12):5278--5311, 2019.

\bibitem[DM17]{DM2017}
Alain Durmus and Eric Moulines.
\newblock Non-asymptotic convergence analysis for the {U}nadjusted {L}angevin
  {A}lgorithm.
\newblock {\em Annals of Applied Probability}, 27(3):1551--1587, 2017.

\bibitem[DM19]{DM2016}
Alain Durmus and Eric Moulines.
\newblock High-dimensional {B}ayesian inference via the {U}nadjusted {L}angevin
  {A}lgorithm.
\newblock {\em Bernoulli}, 25(4A):2854--2882, 2019.

\bibitem[DMP18]{DMP2016}
Alain Durmus, Eric Moulines, and Marcelo Pereyra.
\newblock Efficient {B}ayesian computation by proximal {M}arkov {C}hain {M}onte
  {C}arlo: When {L}angevin meets {M}oreau.
\newblock {\em SIAM Journal on Imaging Sciences}, 11(1):473--506, 2018.

\bibitem[DRD20]{dalalyan2018kinetic}
Arnak~S Dalalyan and Lionel Riou-Durand.
\newblock On sampling from a log-concave density using kinetic {L}angevin
  diffusions.
\newblock {\em Bernoulli}, 26(3):1956--1988, 2020.

\bibitem[EGZ19]{Eberle}
Andreas Eberle, Arnaud Guillin, and Raphael Zimmer.
\newblock Couplings and quantitative contraction rates for {L}angevin dynamics.
\newblock {\em Annals of Probability}, 47(4):1982--2010, 2019.

\bibitem[EH21]{EH2021}
Murat~A. Erdogdu and Rasa Hosseinzadeh.
\newblock On the convergence of {L}angevin {M}onte {C}arlo: The interplay
  between tail growth and smoothness.
\newblock In {\em Proceedings of Thirty Fourth Conference on Learning Theory},
  volume 134, pages 1776--1822. PMLR, 2021.

\bibitem[EHZ22]{EHZ2022}
Murat~A Erdogdu, Rasa Hosseinzadeh, and Matthew~S. Zhang.
\newblock Convergence of {L}angevin {M}onte {C}arlo in chi-squared and
  {R}\'{e}nyi divergence.
\newblock In {\em Proceedings of the 25th International Conference on
  Artificial Intelligence and Statistics}, volume 151, pages 8151--8175. PMLR,
  2022.

\bibitem[FA89]{taylorfacenda1989note}
Jos{\'e}~A Facenda~Aguirre.
\newblock A note on {T}aylor's theorem.
\newblock {\em The American Mathematical Monthly}, 96(3):244--247, 1989.

\bibitem[GCSR95]{gelman1995bayesian}
Andrew Gelman, John~B Carlin, Hal~S Stern, and Donald~B Rubin.
\newblock {\em Bayesian Data Analysis}.
\newblock Chapman \& Hall/CRC Press, 1995.

\bibitem[GGHZ21]{DistMCMC19}
Mert G\"urb\"uzbalaban, Xuefeng Gao, Yunhan Hu, and Lingjiong Zhu.
\newblock Decentralized stochastic gradient {L}angevin dynamics and
  {H}amiltonian {M}onte {C}arlo.
\newblock {\em Journal of Machine Learning Research}, 22(239):1--69, 2021.

\bibitem[GGZ20]{GGZ2}
Xuefeng Gao, Mert G\"{u}rb\"{u}zbalaban, and Lingjiong Zhu.
\newblock Breaking reversibility accelerates {L}angevin dynamics for global
  non-convex optimization.
\newblock In {\em Advances in Neural Information Processing Systems (NeurIPS)},
  volume~33, 2020.

\bibitem[GGZ22]{GGZ}
Xuefeng Gao, Mert G\"{u}rb\"{u}zbalaban, and Lingjiong Zhu.
\newblock Global convergence of {S}tochastic {G}radient {H}amiltonian {M}onte
  {C}arlo for non-convex stochastic optimization: Non-asymptotic performance
  bounds and momentum-based acceleration.
\newblock {\em Operations Research}, 70(5):2931--2947, 2022.

\bibitem[GIWZ24]{GIWZ2024}
Mert G\"{u}rb\"{u}zbalaban, Mohammad~Rafiqul Islam, Xiaoyu Wang, and Lingjiong
  Zhu.
\newblock Generalized {EXTRA} stochastic gradient {L}angevin dynamics.
\newblock {\em arXiv preprint arXiv:2412.01993}, 2024.

\bibitem[GS84]{givens1984class}
Clark~R Givens and Rae~Michael Shortt.
\newblock A class of {W}asserstein metrics for probability distributions.
\newblock {\em The Michigan Mathematical Journal}, 31(2):231--240, 1984.

\bibitem[Gui22]{rmultivariatetaylor}
Emanuele Guidotti.
\newblock {calculus}: High-dimensional numerical and symbolic calculus in {R}.
\newblock {\em Journal of Statistical Software}, 104(5):1--37, 2022.

\bibitem[HJ94]{horn1994matrix}
Roger~A Horn and Charles~R Johnson.
\newblock {\em Topics in Matrix Analysis}.
\newblock Cambridge University Press, 1994.

\bibitem[HKS89]{stroock-langevin-spectrum}
Richard~A Holley, Shigeo Kusuoka, and Daniel~W Stroock.
\newblock Asymptotics of the spectral gap with applications to the theory of
  simulated annealing.
\newblock {\em Journal of Functional Analysis}, 83(2):333--347, 1989.

\bibitem[Jun22]{jung2022machinebook}
Alexander Jung.
\newblock {\em Machine learning: the basics}.
\newblock Springer Nature, 2022.

\bibitem[LF72]{lancaster1972norms}
Peter Lancaster and Hanafi~K Farahat.
\newblock Norms on direct sums and tensor products.
\newblock {\em Mathematics of Computation}, 26(118):401--414, 1972.

\bibitem[MCC{\etalchar{+}}21]{Ma2019}
Yi-An Ma, Niladri~S. Chatterji, Xiang Cheng, Nicolas Flammarion, Peter~L.
  Bartlett, and Michael~I. Jordan.
\newblock Is there an analog of {N}esterov acceleration for gradient-based
  {MCMC}?
\newblock {\em Bernoulli}, 27(3):1942--1992, 2021.

\bibitem[MMW{\etalchar{+}}21]{mou2021high}
Wenlong Mou, Yi-An Ma, Martin~J Wainwright, Peter~L Bartlett, and Michael~I
  Jordan.
\newblock High-order {L}angevin diffusion yields an accelerated {MCMC}
  algorithm.
\newblock {\em Journal of Machine Learning Research}, 22(42):1--41, 2021.

\bibitem[Mon23]{monmarche2023almost}
Pierre Monmarch{\'e}.
\newblock Almost sure contraction for diffusions on $\mathbb{R}^{d}$.
  {A}pplication to generalized {L}angevin diffusions.
\newblock {\em Stochastic Processes and their Applications}, 161:316--349,
  2023.

\bibitem[MSH02]{mattingly2002ergodicity}
Jonathan~C Mattingly, Andrew~M Stuart, and Desmond~J Higham.
\newblock Ergodicity for {SDE}s and approximations: locally {L}ipschitz vector
  fields and degenerate noise.
\newblock {\em Stochastic Processes and their Applications}, 101(2):185--232,
  2002.

\bibitem[Red12]{redfern2012maple}
Darren Redfern.
\newblock {\em The Maple Handbook: Maple V Release 4}.
\newblock Springer Science \& Business Media, 2012.

\bibitem[RRT17]{Raginsky}
Maxim Raginsky, Alexander Rakhlin, and Matus Telgarsky.
\newblock Non-convex learning via stochastic gradient {L}angevin dynamics: a
  nonasymptotic analysis.
\newblock In {\em Proceedings of the 2017 Conference on Learning Theory},
  volume~65, pages 1674--1703. PMLR, 2017.

\bibitem[SBB{\etalchar{+}}80]{stoer1980book}
Josef Stoer, Roland Bulirsch, R~Bartels, Walter Gautschi, and Christoph
  Witzgall.
\newblock {\em Introduction to Numerical Analysis}, volume 1993.
\newblock Springer, 1980.

\bibitem[SC08]{steinwartbook2008support}
Ingo Steinwart and Andreas Christmann.
\newblock {\em Support Vector Machines}.
\newblock Springer Science \& Business Media, 2008.

\bibitem[SL19]{shen2019randomized}
Ruoqi Shen and Yin~Tat Lee.
\newblock The randomized midpoint method for log-concave sampling.
\newblock In {\em Advances in Neural Information Processing Systems},
  volume~32, 2019.

\bibitem[Stu10]{stuart2010inverse}
Andrew~M Stuart.
\newblock Inverse problems: A {B}ayesian perspective.
\newblock {\em Acta Numerica}, 19:451--559, 2010.

\bibitem[TTV16]{teh2016consistency}
Yee~Whye Teh, Alexandre~H Thiery, and Sebastian~J Vollmer.
\newblock Consistency and fluctuations for stochastic gradient {L}angevin
  dynamics.
\newblock {\em Journal of Machine Learning Research}, 17(1):193--225, 2016.

\bibitem[Vil09]{Villani2009}
C\'{e}dric Villani.
\newblock Hypocoercivity.
\newblock {\em Memoirs of the American Mathematical Society}, 202(950):iv+141,
  2009.

\bibitem[Vit08]{air_quality_360}
Saverio Vito.
\newblock {Air Quality}.
\newblock UCI Machine Learning Repository, 2008.
\newblock {DOI}: https://doi.org/10.24432/C59K5F.

\bibitem[WL06]{wang2006generating}
Jin Wang and Chunlei Liu.
\newblock Generating multivariate mixture of normal distributions using a
  modified {C}holesky decomposition.
\newblock In {\em Proceedings of the 2006 Winter Simulation Conference}, pages
  342--347. IEEE, 2006.

\bibitem[WWJ16]{wibisono2016variational}
Andre Wibisono, Ashia~C Wilson, and Michael~I Jordan.
\newblock A variational perspective on accelerated methods in optimization.
\newblock {\em Proceedings of the National Academy of Sciences},
  113(47):E7351--E7358, 2016.

\bibitem[ZADS23]{Zhang2019}
Ying Zhang, \"{O}mer~Deniz Akyildiz, Theodoros Damoulas, and Sotirios Sabanis.
\newblock Nonasymptotic estimates for {S}tochastic {G}radient {L}angevin
  {D}ynamics under local conditions in nonconvex optimization.
\newblock {\em Applied Mathematics \& Optimization}, 87:25, 2023.

\end{thebibliography}


\newpage
\appendix

\section{Supporting Results for Section~\ref{section_frommonmarche}}
\label{appendix_frommonmarche}

Denote $\mathcal{M}_{m,n}(\R)$ the set of real matrices of size $m\times n$. In \cite[Section 4.3]{monmarche2023almost}, Monmarch\'{e} considers the \textit{generalized Langevin diffusions}:
\begin{align}
\label{originalequation_appendix}
dX_t &=  A Y_t d t, \nonumber\\
dY_t &=  - A^{\top}\nabla U(X_t) d t - \gamma B Y_t d t + \sqrt{\gamma}\Sigma d W_t,
\end{align}
with $A\in\mathcal M_{d,p}(\R)$;  $B,\Sigma\in\mathcal M_{p,p}(\R)$; $U\in\mathcal C^2(\R^d)$; $\gamma>0$ and $W$ is a standard $p$-dimensional Brownian motion. 

Set $b$ as the drift coefficient of \eqref{originalequation_appendix}, that is 
\begin{align}
\label{def_driftb_appendix}
    b(x,y):=\begin{pmatrix}
        Ay\\-A^{\top}\nabla U(x)-\gamma By
    \end{pmatrix}.
\end{align}

We summarize here Assumptions 1, 2 and 3 in \cite[Section 4.3]{monmarche2023almost} regarding \eqref{originalequation_appendix}.
\begin{customassump}{F}~
		\label{cond_frommonmarche}
\begin{itemize}
    \item There exist $m,L,\kappa>0$ and a symmetric positive-definite matrix $H$ of size $p\times p$ such that 
    \begin{align}
    \label{cond_kappa}
        HB+B^{\top}H\geq 2\kappa H,
    \end{align}
    and that $U$ is $m$ strongly-convex and $L$-smooth: $m I_d\leq \nabla^2 U(x)\leq L I_d$, for any $x\in \R^d$. 
\item $p\geq d$ and $A=(I_d,0,\ldots, 0)$. 
\item When $p>d$, consider the decomposition $B= \begin{pmatrix}B_{11} & B_{12}\\
B_{21}& B_{22}\end{pmatrix}$ where $B_{11}, B_{12},B_{21},B_{22}$ are respectively of size $d\times d$, $d\times (p-d)$, $(p-d)\times d$ and $(p-d)\times (p-d)$. In this case, we assume $B_{22}$ is invertible and that 
\begin{align*}
    E:=B_{11}-B_{12}B_{22}^{-1}B_{21}
\end{align*}
is symmetric positive-definite. Set $D:=B_{12}B_{22}^{-1}$. 
\item When $p=d$, we assume $B$ is symmetric positive definite. Set $E=B$ and $D=0$. 
\end{itemize}
\end{customassump}

\begin{remark}
    In \cite[Assumption 2]{monmarche2023almost}, the author writes $HB\geq \kappa H$. If one looks at the notation subsection right before Section 2 of the aforementioned reference, $HB\geq \kappa H$ for not necessarily symmetric matrix $HB$ is understood as $HB+B^{\top}H\geq 2\kappa H$, which is what we have in our Condition~\ref{cond_frommonmarche}. 
\end{remark}

\begin{remark}
   In \cite{monmarche2023almost}, beside from Condition~\ref{cond_frommonmarche},  the author also assumes that $\nabla_{\alpha} U$ for any $\abs{\alpha}=\sum_i \alpha_i\geq 2$ is bounded. Per private communication with the author, this is done out of convenience to avoid technical regularity issues regarding the semi-groups. In our case, we are interested in Theorem 9 in \cite{monmarche2023almost} which only requires the boundedness of second-order derivative of $U$ and not of the higher order derivative. 
\end{remark}

The following result is stated under Assumption 3 in \cite{monmarche2023almost}. 
\begin{lemma}
\label{lemma_constanth_i}
    Under Condition~\ref{cond_frommonmarche}, there exist constants $h_i>0,1\leq i\leq 5$ such that 
    \begin{align*}
       & HA^{\top}A H\leq h_1H,\qquad\qquad \frac{1}{h_2}I_d\leq E\leq h_3I_d,\\
      &  \begin{pmatrix} I_d&-D\\0 &0\end{pmatrix}\leq h_4H,\qquad \begin{pmatrix}
            I_d &-D\\ -D^{\top}& 0
        \end{pmatrix}\leq h_5H. 
    \end{align*}
\end{lemma}
Note that we follow the convention in \cite{monmarche2023almost}: for $m\times m$ matrices $M,H$ that are not necessarily symmetric, $M\geq H$ means $\inner{x,Mx}\geq \inner{x,Hx}$ for all $x\in \R^m$.

\begin{theorem}(\cite[Lemma 8 and Theorem 9]{monmarche2023almost})
\label{theorem_monmarche_appendix}
Assume Conditions~\ref{cond_frommonmarche} and set 
\begin{align*}
    \gamma_0=2\sqrt{\frac{h_1L}{\kappa}}\max \left\{ \sqrt{h_2h_5},\sqrt{\frac{h_4}{\kappa}} \right\}.
\end{align*}
Further assume that the friction coefficient $\gamma$ is sufficiently large: $ \gamma \geq \gamma_0$. Set $\rho=\min \left\{ \frac{m}{3h_3\gamma},\frac{\gamma \kappa}{6} \right\}$.   
Recall the drift coefficient $b$ in \eqref{def_driftb_appendix} and denote $J_b$ its Jacobian matrix. Write $\begin{pmatrix}
    I_d &-D
\end{pmatrix}$ as a block matrix. Then $M:=\begin{pmatrix}
 E & \frac{1}{\gamma}\begin{pmatrix}
    I_d &-D
\end{pmatrix}\\
\frac{1}{\gamma}\begin{pmatrix}
    I_d &-D
\end{pmatrix}^{\top} & \frac{\kappa}{Lh_1} H
\end{pmatrix}$ is a symmetric positive-definite matrix of size $(d+p)\times (d+p)$ such that 
\begin{align*}
    M J_b+J_b^{\top} M\leq -2\rho M.
\end{align*} 
Moreover, the matrix $M$ satisfies 
\begin{align}
\label{inequality_matrixM}
    \frac{1}{2}\begin{pmatrix}E&0\\0&\frac{\kappa}{Lh_1}H \end{pmatrix}\leq M\leq  \frac{3}{2}\begin{pmatrix}E&0\\0&\frac{\kappa}{Lh_1}H \end{pmatrix}.
\end{align}
\end{theorem}

A \textit{$P$-th order Langevin dynamics} as the focus of the main paper is a special case of \eqref{originalequation_appendix} where $p=(P-1)d$ and 
\begin{equation}\label{langevin_Pthorder_appendix}
A=A_{P} = \begin{pmatrix}
I_d & 0 & \dots   & 0
\end{pmatrix}\qquad \text{and}\qquad
B=B_{P} = \begin{pmatrix}
0  & -I_d & 0 & \dots & 0 \\
 I_d & 0 & -I_d & \ddots & \vdots \\
0 & \ddots & \ddots & \ddots &   0 \\
\vdots & \ddots &   I_d & 0 & - I_d\\
0 &\dots  & 0 &  I_d &   I_d
\end{pmatrix}\,.
\end{equation}

The following Corollary is Theorem~\ref{theorem_monmarche_appendix} in the special case of $P$-th order Langevin dynamics. The proof is mostly taken from~\cite[Section 4]{monmarche2023almost}. We add some details regarding dimension dependence of the parameters since this is not one of the goals of \cite{monmarche2023almost}; however, it is a crucial concern of our paper.

\begin{corollary}(\cite[Section 4]{monmarche2023almost})
\label{coro_monmarcheforPthorderlangevin}
        Conditions~\ref{cond_frommonmarche} is satisfied for the $P$-th order Langevin dynamics \eqref{langevin_Pthorder_appendix}, so that the conclusion of Theorem~\ref{theorem_monmarche_appendix} applies to \eqref{langevin_Pthorder_appendix}. In particular, regarding the matrix $M$, we have $E=I_d,D=-\begin{pmatrix}
            I_d &\ldots& I_d
        \end{pmatrix}$, while the matrix $H$, the constants $\kappa$ and $h_i,1\leq i\leq 5$ are explicitly provided in \textbf{Step 3} of the proof. 
        
        Moreover, $\rho=\rho(\gamma,L,P),\gamma_0=\gamma_0(\gamma,L,P)$ depend on $\gamma,L,P$ but do not depend on the dimension parameter $d$. Furthermore, $\lambda_{\min,M}=\lambda_{\min,M}(P)$ and $\lambda_{\max,M}=\lambda_{\max,M}(P)$ as respectively the smallest and largest eigenvalues of the positive definite matrix $M$ depend on $P$ and do not depend on the dimension parameter $d$. 
        \end{corollary}



\begin{proof}
The proof is divided into six steps.

\textbf{Step 1:} We start by observing a simplified form of the matrix $B=B_{P}$ in \eqref{langevin_Pthorder_appendix}:
\begin{align*}
    B=B_{\operatorname{sim}} \otimes I_d,\qquad B_{\operatorname{sim}} :=\begin{pmatrix}
0  & -1 & 0 & \dots & 0 \\
 1 & 0 & -1 & \ddots & \vdots \\
0 & \ddots & \ddots & \ddots &   0 \\
\vdots & \ddots &   1 & 0 & - 1\\
0 &\dots  & 0 &  1 &   1
\end{pmatrix},
\end{align*}
where $\otimes$ denotes the Kronecker product (\cite{horn1994matrix}). This simplified form indicates that $B$ and the $(P-1)\times (P-1)$ matrix $B_{\operatorname{sim}}$ have the same spectrum, and thus such spectrum does not depend on $d$. Furthermore, it indicates that if $v_i,1\leq i\leq P-1$ are eigenvectors (respectively generalized eigenvectors) of $B_{\operatorname{sim}}$ and $e_j,1\leq j\leq d$ are the standard basis of $\R^d$, then $w_i=v_i\otimes e_j, 1\leq i\leq P-1 ,1\leq j\leq d$ are eigenvectors (respectively generalized eigenvectors) of $B$. 

\textbf{Step 2:} Let us verify that $\min \{\mathrm{Re}(\lambda):\lambda \text{ is an eigenvalue of $B_{\operatorname{sim}}$}\}>0$, which together with $B,B_{\operatorname{sim}}$ having the same spectrum from the \textbf{Step 1} imply  $\min \{\mathrm{Re}(\lambda):\lambda \text{ is an eigenvalue of $B$}\}>0$.   

We have the decomposition $B_{\operatorname{sim}}=\frac{1}{2}( B_{\operatorname{sim}}+ B_{\operatorname{sim}}^{\top})+\frac{1}{2}( B_{\operatorname{sim}}- B_{\operatorname{sim}}^{\top}):=H+K$
where $H$ is a Hermitian matrix and $K$ is a skew-Hermitian matrix. Now assume $\lambda$ is an eigenvalue of $B_{\operatorname{sim}}$: $B_{\operatorname{sim}}x=\lambda x$ for a nonzero vector $x=(x_1,\ldots,x_{P-1})\in \mathbb{C}^{P-1}$.  Then $\lambda=\frac{x^*B_{\operatorname{sim}}x}{x^*x}$ where $x^*$ denotes the conjugate transpose of $x$, and hence 
\begin{align*}
    \mathrm{Re}(\lambda)=\frac{x^*Hx}{x^*x}=\frac{\abs{x_{P-1}}^2}{\abs{x}^2}. 
\end{align*}
We claim that $x_{P-1}\neq 0$ which implies $\mathrm{Re}(\lambda)>0$. Suppose the opposite that $x_{P-1}=0$, then it is easy to solve for $B_{\operatorname{sim}}x=\lambda x$ to get $x_1=x_2=\cdots=x_{P-1}=0$, which is a contradiction. This completes our argument for the \textbf{Step 2}. 

\textbf{Step 3 in the case where $B$ is diagonalizable:} Let us construct $H, \kappa$ that satisfy Condition~\ref{cond_frommonmarche} in the simpler case where $B$ is diagonalizable. The construction has been done in \cite[Lemma 4.3]{arnolderb2014sharpentropy} or \cite[Section 2.1]{arnold2020sharp}, and we summarize it here for the sake of completeness.

In this case,  $B$ has $(P-1)d$ linearly independent eigenvectors $w_i,1\leq i\leq (P-1)d$ corresponding to $(P-1)d$ eigenvalues $\lambda_i,1\leq i\leq (P-1)d$. Denote $w_i^*$ the conjugate transpose of $w_i$ and set $H=:\sum_{i=1}^{(P-1)d} w_iw_i^*$, then
\begin{align*}
    HB+B^{\top}H=\sum_{i=1}^{(P-1)d} \brac{\lambda_i+\overline{\lambda_i}}w_iw_i^*&\geq 2\min \{\mathrm{Re}(\lambda_i),1\leq i\leq (P-1)d\}\sum_{i=1}^{(P-1)d}w_iw_i^*\\
    &= 2\hat{\lambda} H, \quad \hat{\lambda}=\min \{\mathrm{Re}(\lambda):\lambda \text{ is an eigenvalue of $B$}\}. 
\end{align*}
Thus, in the case where $B$ is diagonalizable, $\kappa$ in Condition~\ref{cond_frommonmarche} can be taken as $\hat{\lambda}$ which is a positive number per our \textbf{Step 2} above.


\textbf{Step 3 in the case where $B$ is not diagonalizable:} In contrast to the previous case, there is at least one Jordan block of $B$ of length $\ell_n\geq 2$.  In this case, and the construction of $H, \kappa$ satisfying Condition~\ref{cond_frommonmarche} is more elaborate. Denote the Jordan blocks of $B$ by $B_n,1\leq n\leq H$. Each block $B_n$ of length $\ell_n$ is associated with the eigenvalue $\lambda_n$ and the set of generalized eigenvectors $v_n^{(k)}; 1\leq k\leq \ell_n$. In particular, $v_n^{(1)}$ is the (standard) eigenvector of $J_n$.

For a Jordan block $B_n$ with $\mathrm{Re}(\lambda_n)>\hat{\lambda}$, we set $H_n=\sum_{i=1}^{\ell_n} b^i_n v_n^{(i)}\brac{{v}_n^{(i)}}^*$
where 
\begin{align*}
    b_{n}^1=1; b^j_n=c_j (t_n)^{2(1-j)}, 2\leq j\leq \ell_n \qquad 
    &\text{and}\quad c_1=1;c_{j+1}=1+c_{j}^2,\,2\leq j\leq \ell_n
    \\&\text{and}\quad t_n=2(\mathrm{Re}(\lambda_n)-\kappa). 
\end{align*}
Then per \cite[Lemma 4.3]{arnolderb2014sharpentropy}, $ H_n B_n+B_n^{\top}H_n\geq 2\hat{\lambda} H_n$. 

Meanwhile, for a Jordan block $B_m$ with $\mathrm{Re}(\lambda_m)=\hat{\lambda}$, we replace the above $t_n$ with $t_m=2(\mathrm{Re}(\lambda_n)-\hat{\lambda}+\epsilon)$ for any $\epsilon\in (0,\hat{\lambda})$ and define $\widetilde{H}_m(\epsilon)=\sum_{i=1}^{\ell_m} b^i_m(\epsilon) v_m^{(i)}\brac{{v_m^{(i)}}}^*$. Then $ \widetilde{H}_m B_m+B_m^{\top}\widetilde{H}_m\geq 2(\hat{\lambda}-\epsilon) \widetilde{H}_m$. 

Therefore, in the case where $B$ is not diagonalizable, we denote $I=\{n\in\{1,\cdots,N \}:\ell_n\geq 2, \mathrm{Re}(\lambda_n)=\hat{\lambda} \}$ and define $H:=H(\epsilon)=\sum_{n\in \{1,\ldots,N\}\setminus I } H_n+\sum_{m\in I}\widetilde{H}_m(\epsilon)$. Then we have 
\begin{align*}
    H(\epsilon) B+B^{\top}H(\epsilon)\geq 2(\hat{\lambda}-\epsilon) H(\epsilon). 
\end{align*}
Thus, in the case where $B$ is not diagonalizable, $\kappa$ in Condition~\ref{cond_frommonmarche} is $\hat{\lambda}-\epsilon$ for any $\epsilon\in (0,\hat{\lambda})$. Notice $\hat{\lambda}>0$ per our \textbf{Step 2}, so that it is possible to choose $\epsilon>0$ such that $\hat{\lambda}-\epsilon>0$.

\textbf{Step 4:} Let us verify that $\kappa$, $\opnorm{H}$ and $\opnorm{H^{-1}}$ do not depend on $d$. The former is clear from the fact that $\kappa$ is either $2\hat{\lambda}$ or $2(\hat{\lambda}-\epsilon)$ in the \textbf{Step 2}, and the fact that $B$ has the same spectrum as the $(P-1)\times (P-1)$ matrix $B_{\operatorname{sim}}$, per the first paragraph of this proof. Regarding $\opnorm{H}$, we will assume $B$ is diagonalizable to keep things simple (the case of non-diagonalizable $B$ is almost the same). We know from the first paragraph of this proof that 
\begin{align}
\label{simplifiedformmatrixH}
    H=\sum_{1\leq i\leq P-1 ,1\leq j\leq d} \brac{v_i\otimes e_j}\brac{v_i^*\otimes e_j^{\top}} &= \sum_{1\leq i\leq P-1 ,1\leq j\leq d} \brac{v_iv_i^*}\otimes \brac{e_j\otimes e_j^{\top}}\nonumber\\
    &=\brac{\sum_{1\leq i\leq P-1 }v_iv_i^*}\otimes \brac{\sum_{1\leq j\leq d}e_j\otimes e_j^{\top}}\nonumber\\
    &=\brac{\sum_{1\leq i\leq P-1 }v_iv_i^*}\otimes I_d. 
\end{align}
Thus, we have $\opnorm{H}=\opnorm{\sum_{1\leq i\leq P-1 }v_iv_i^* } \opnorm{I_d}=\opnorm{\sum_{1\leq i\leq P-1 }v_iv_i^* }$ (see \cite[Theorem 8]{lancaster1972norms} regarding matrix norms and Kronecker product). Since $\sum_{1\leq i\leq P }v_iv_i^* $ is a $(P-1)\times (P-1)$ matrix, $\opnorm{H}$ does not depend on $d$. We can reach the same conclusion for $\opnorm{H^{-1}}$, noting that $H^{-1}=\brac{\sum_{1\leq i\leq P-1 }v_iv_i^*}^{-1}\otimes I_d$.

\textbf{Step 5:} We verify that the constants in Lemma~\ref{lemma_constanth_i} do not depend on $d$. In Lemma~\ref{lemma_constanth_i}, the matrix $E=I_d$ (pointed out below Assumption 3 in \cite{monmarche2023almost}). Thus, we can take $h_1=\opnorm{HA^{\top}A},h_2=1$ and $h_3=1$. Then from \eqref{simplifiedformmatrixH} and $A=\begin{pmatrix}
I_d & 0 & \dots   & 0
\end{pmatrix}=\begin{pmatrix}
1 & 0 & \dots   & 0
\end{pmatrix}\otimes I_d$, we deduce that $  h_1=\opnorm{H\begin{pmatrix}
         1 &\cdots &0\\
         \vdots&\ddots&\vdots\\
         0&\cdots&0
     \end{pmatrix}  }$. Thus, $h_1,h_2,h_3$ do not depend on $d$. What remain to study are $h_4$ and $h_5$. We have $D=-(I_d,\ldots,I_d)$ as pointed out below Assumption 3 in \cite{monmarche2023almost}. It is easy to verify that $(1+(P-1)/2)I_{(P-1)d}- \begin{pmatrix} I_d&-D\\0 &0\end{pmatrix}\geq 0$,
which implies 
\begin{equation*}
\begin{pmatrix} I_d&-D\\0 &0\end{pmatrix} H^{-1} H\leq (1+P/2)H^{-1} H\leq (1+(P-1)/2)\norm{H^{-1}}_{\operatorname{op}} H, 
\end{equation*}
and hence $h_4=\brac{1+(P-1)/2}\norm{H^{-1}}_{\operatorname{op}}$. The formula $h_5= (1+P) \norm{H^{-1}}_{\operatorname{op}}$ is obtained the same way, noting that $\begin{pmatrix}
            I_d &-D\\ -D^{\top}& 0
        \end{pmatrix}= \begin{pmatrix}
            I_d &-D\\ 0& 0
        \end{pmatrix}+\begin{pmatrix}
            I_d &0\\ -D^{\top}& 0
        \end{pmatrix}. $      
Finally, $\norm{H^{-1}}_{\operatorname{op}}=\opnorm{\brac{\sum_{1\leq i\leq P-1 }v_iv_i^*}^{-1}}$ does not depend on $d$ per the \textbf{Step 3}, so that $h_4$ and $h_5$ do not depend on $d$. 

\textbf{Step 6:} Let us consider $\rho,\gamma_0$ and $\lambda_{\min,M}$ of Theorem~\ref{theorem_monmarche_appendix} in the context of the $P$-th order Langevin dynamics \eqref{langevin_Pthorder_appendix}. Note that $m$ and $L$ are respectively the strong-convexity and smoothness constants of $U$ does not depend on $d$. This, combined with the conclusions in the \textbf{Step 3} and \textbf{Step 4}, implies that $\gamma_0=2\sqrt{\frac{h_1L}{\kappa}}\max \left\{ \sqrt{h_2h_5},\sqrt{\frac{h_4}{\kappa}} \right\}$ and $\gamma>\gamma_0$ does not depend on $d$. Regarding $\lambda_{\min,M}$, it is pointed out below Assumption 3 in \cite{monmarche2023almost} that in the case of $P$-th order Langevin dynamics \eqref{langevin_Pthorder_appendix}, $E=I_d$ in Condition~\ref{cond_frommonmarche}. Then inequality \eqref{inequality_matrixM} becomes 
\begin{equation*}
\frac{1}{2}\begin{pmatrix}I_d&0\\0&\frac{\kappa}{Lh_1}H \end{pmatrix}\leq M\leq  \frac{3}{2}\begin{pmatrix}I_d&0\\0&\frac{\kappa}{Lh_1}H \end{pmatrix}. 
\end{equation*}
Moreover, per \eqref{simplifiedformmatrixH}, $H$ and the $(P-1)\times (P-1)$ matrix $\sum_{1\leq i\leq P-1 }v_iv_i^*$ share the same spectrum which does not depend on $d$. Then by a consequence of Courant–Fischer–Weyl's min-max Theorem regarding comparison of eigenvalues of positive definite matrices (\cite[Problem 4.2.P8, Page 238]{horn1994matrix}, we can conclude $\lambda_{\min,M}$ as the smallest eigenvalue of the positive definite matrix $M$ does not depend on $d$. The same conclusion holds for $\lambda_{\max,M}$. 
The proof is complete.
\end{proof}

\section{Details of Fourth-Order Langevin Monte Carlo Algorithm}\label{appendix_4thorder}

Lemma~\ref{lemma_explicitformofxbar} as the first result of this Appendix contains explicit form of the components of $\bar{x}((k+1)\eta)$ in terms of the components of $x^{(k)}$ in the splitting scheme \eqref{scheme_4thorder}. Based on it, we will be able to derive in Lemma~\ref{lemma_meanandcovariance} the conditional mean and covariance associated with the fourth-order LMC algorithm in Section~\ref{section_4thorder}. 

Here are the components of $\hat{x}(t)$ in terms of the components of $x^{(k)}$.
\begin{align}
    \hat{v}_1(t)&=v_1^{(k)},\nonumber\\
    \hat{\theta}(t)&=\theta^{(k)}+(t-k\eta)v_1^{(k)},\label{equation_thetahat}\\
    \hat{v}_2(t)&=v_2^{(k)}+\gamma\brac{v_3^{(k)}-v_1^{(k)} }(t-k\eta),\nonumber\\
    \hat{v}_3(t)&=e^{-\gamma(t-k\eta)}v_3^{(k)}-\gamma \int_{k\eta}^t e^{-\gamma(t- s)}\hat{v}_2(s)ds+\sqrt{2\gamma}\int_{k\eta}^t e^{-\gamma(t- s)}dB_s\nonumber\\
    &=e^{-\gamma(t-k\eta)}v_3^{(k)}-\gamma v_2^{(k)}\int_{k\eta}^t e^{-\gamma(t- s)}ds\nonumber\\
    &\qquad-\gamma^2\brac{v_3^{(k)}-v_1^{(k)}}\int_{k\eta}^t e^{-\gamma(t- s)}(s-k\eta)ds+\sqrt{2\gamma}\int_{k\eta}^t e^{-\gamma(t- s)}dB_s\nonumber. 
\end{align}

Next are the components of $\tilde{x}(t)$ in terms of the components of $x^{(k)}$. 
\begin{align}
    \tilde{v}_1(t)&=v_1^{(k)}-\int_{k\eta}^t \tilde{g}(s)ds+\gamma \int_{k\eta}^t \hat{v}_2(s)ds\nonumber\\
    &=v_1^{(k)}-\int_{k\eta}^t \tilde{g}(s)ds+\gamma v_2^{(k)}(t-k\eta)+\gamma^2\brac{v_3^{(k)}-v_1^{(k)} }\frac{(t-k\eta)^2}{2},
\label{equation_v1tilde}\\
    \tilde{\theta}(t)&=\theta^{(k)}+\int_{k\eta}^t \tilde{v}_1(s)ds\nonumber\\
    &=\theta^{(k)}+v_1^{(k)}(t-k\eta)-\int_{k\eta}^t\int_{k\eta}^s \tilde{g}(r)drds+\gamma v_2^{(k)}\frac{(t-k\eta)^2}{2!}+\gamma^2\brac{v_3^{(k)}-v_1^{(k)} }\frac{(t-k\eta)^3}{3!},\nonumber\\
    \tilde{v}_2(t) &=v_2^{(k)}-\gamma\int_{k\eta}^t\tilde{v}_1(s)ds+\gamma\int_{k\eta}^t\hat{v}_3(s) ds\nonumber\\
    &=v_2^{(k)}-\gamma v_1^{(k)}(t-k\eta)+\gamma\int_{k\eta}^t\int_{k\eta}^s \tilde{g}(r)drds-\gamma^2 v_2^{(k)}\frac{(t-k\eta)^2}{2!}-\gamma^3\brac{v_3^{(k)}-v_1^{(k)} }\frac{(t-k\eta)^3}{3!}\nonumber\\
    &\qquad\quad+\gamma v_3^{(k)}\int_{k\eta}^t e^{-\gamma(s-k\eta)}ds-\gamma^2 v_2^{(k)}\int_{k\eta}^t\int_{k\eta}^s e^{-\gamma(s- r)}drds\nonumber\\
    &\qquad\qquad-\gamma^3\brac{v_3^{(k)}-v_1^{(k)} }\int_{k\eta}^t\int_{k\eta}^s e^{-\gamma(s- r)}(r-k\eta)drds+\gamma\sqrt{2\gamma}\int_{k\eta}^t\int_{k\eta}^s e^{-\gamma(s- r)}dB_rds,\nonumber\\
\tilde{v}_3(t)&=v_3^{(k)}e^{-\gamma(t-k\eta)}-\gamma \int_{k\eta}^t e^{-\gamma(t- s)}\tilde{v}_2(s)ds+\sqrt{2\gamma}\int_{k\eta}^t e^{-\gamma(t- s)}dB_s\nonumber\\
    &=v_3^{(k)}e^{-\gamma(t-k\eta)}+\sqrt{2\gamma}\int_{k\eta}^t e^{-\gamma(t- s)}dB_s-\gamma v_2^{(k)}\int_{k\eta}^t e^{-\gamma(t- s)}ds+\gamma^2 v_1^{(k)}\int_{k\eta}^t e^{-\gamma(t- s)}(s-k\eta)ds\nonumber\\
    &\quad-\gamma^2\int_{k\eta}^t\int_{k\eta}^s\int_{k\eta}^r e^{-\gamma(t- s)}\tilde{g}(w)dwdrds\nonumber\\
    &\qquad+\gamma^3 v_2^{(k)}\int_{k\eta}^t e^{-\gamma(t- s)}\frac{(s-k\eta)^2}{2!}ds+\gamma^4 \brac{v_3^{(k)}-v_1^{(k)} }\int_{k\eta}^t e^{-\gamma(t- s)}\frac{(s-k\eta)^3}{3!}ds\nonumber\\
    &\quad-\gamma^2 v_3^{(k)} \int_{k\eta}^t \int_{k\eta}^s e^{-\gamma(t- s)}e^{-\gamma(r-k\eta)}drds+\gamma^3 v_2^{(k)}\int_{k\eta}^t \int_{k\eta}^s \int_{k\eta}^r e^{-\gamma(t- s)}e^{-\gamma(r- w)}dwdrds\nonumber\\
   &\qquad +\gamma^4 \brac{v_3^{(k)}-v_1^{(k)} }\int_{k\eta}^t\int_{k\eta}^s\int_{k\eta}^r e^{-\gamma(t- s)}e^{-\gamma(r- w)}(w-k\eta)dwdrds\nonumber\\
   &\quad-\gamma^2\sqrt{2\gamma}\int_{k\eta}^t\int_{k\eta}^s\int_{k\eta}^r e^{-\gamma(r- w)}e^{-\gamma(t- s)}dB_wdrds. \nonumber
\end{align}

Finally are the components of $\bar{x}(t)$ in terms of the components of $x^{(k)}$.
\begin{align*}
&\bar{v}_1(t)=v_1^{(k)}-\int_{k\eta}^t \bar{g}(s)ds+\gamma \int_{k\eta}^t \tilde{v}_2(s)ds\\
&\quad=v_1^{(k)}-\int_{k\eta}^t \bar{g}(s)ds+\gamma v_2^{(k)}(t-k\eta)-\gamma^2 v_1^{(k)}\frac{(t-k\eta)^2}{2!}+\gamma^2\int_{k\eta}^t\int_{k\eta}^s \int_{k\eta}^r \tilde{g}(w)dwdrds\\
    &\quad-\gamma^3 v_2^{(k)}\frac{(t-k\eta)^3}{3!}-\gamma^4\brac{v_3^{(k)}-v_1^{(k)} }\frac{(t-k\eta)^4}{4!}+\gamma^2 v_3^{(k)}\int_{k\eta}^t\int_{k\eta}^s e^{-\gamma(r-k\eta)}drds\\
    &-\gamma^3 v_2^{(k)} \int_{k\eta}^t\int_{k\eta}^s\int_{k\eta}^r e^{-\gamma(r- w)}dwdrds-\gamma^4\brac{v_3^{(k)}-v_1^{(k)} }\int_{k\eta}^t\int_{k\eta}^s \int_{k\eta}^r e^{-\gamma(r- w)}(w-k\eta)dwdrds\\
    &\qquad\qquad+\gamma^2\sqrt{2\gamma}\int_{k\eta}^t\int_{k\eta}^s \int_{k\eta}^r e^{-\gamma(r- w)}dB_wdrds,\\
    &\bar{\theta}(t)=\theta^{(k)}+\int_{k\eta}^t \bar{v}_1(s)ds\\
    &\quad=\theta^{(k)}+v_1^{(k)}(t-k\eta)-\int_{k\eta}^t\int_{k\eta}^s \bar{g}(r)drds+\gamma v_2^{(k)}\frac{(t-k\eta)^2}{2!}-\gamma^2 v_1^{(k)}\frac{(t-k\eta)^3}{3!}\\
    &\quad+\gamma^2\int_{k\eta}^t\int_{k\eta}^s \int_{k\eta}^r\int_{k\eta}^w\tilde{g}(y)dydwdrds-\gamma^3 v_2^{(k)}\frac{(t-k\eta)^4}{4!}-\gamma^4\brac{v_3^{(k)}-v_1^{(k)} }\frac{(t-k\eta)^5}{5!}\\
    &\qquad+\gamma^2 v_3^{(k)}\int_{k\eta}^t\int_{k\eta}^s\int_{k\eta}^r e^{-\gamma(w-k\eta)}dwdrds-\gamma^3 v_2^{(k)} \int_{k\eta}^t\int_{k\eta}^s\int_{k\eta}^r\int_{k\eta}^w e^{-\gamma(w- y)}dydwdrds\\
    &\qquad\quad-\gamma^4\brac{v_3^{(k)}-v_1^{(k)} }\int_{k\eta}^t\int_{k\eta}^s \int_{k\eta}^r\int_{k\eta}^w e^{-\gamma(w- y)}(y-k\eta)dydwdrds\\
    &\qquad\qquad+\gamma^2\sqrt{2\gamma}\int_{k\eta}^t\int_{k\eta}^s \int_{k\eta}^r \int_{k\eta}^w e^{-\gamma(w- y)}dB_ydwdrds,\\
    &\bar{v}_2(t)=v_2^{(k)}-\gamma\int_{k\eta}^t \bar{v}_1(s)ds+\gamma \int_{k\eta}^t \tilde{v}_3(s)ds\\
    &\quad=v_2^{(k)}+\Bigg[-\gamma v_1^{(k)}(t-k\eta)+\gamma\int_{k\eta}^t\int_{k\eta}^s \bar{g}(r)drds-\gamma^2 v_2^{(k)}\frac{(t-k\eta)^2}{2}+\gamma^3 v_1^{(k)}\frac{(t-k\eta)^3}{3!}\\
    &\quad-\gamma^3\int_{k\eta}^t\int_{k\eta}^s \int_{k\eta}^r\int_{k\eta}^w\tilde{g}(y)dydwdrds+\gamma^4 v_2^{(k)}\frac{(t-k\eta)^4}{4!}+\gamma^5\brac{v_3^{(k)}-v_1^{(k)} }\frac{(t-k\eta)^5}{5!}\\
    &\qquad-\gamma^3 v_3^{(k)}\int_{k\eta}^t\int_{k\eta}^s\int_{k\eta}^r e^{-\gamma(w-k\eta)}dwdrds+\gamma^4 v_2^{(k)} \int_{k\eta}^t\int_{k\eta}^s\int_{k\eta}^r\int_{k\eta}^w e^{-\gamma(w- y)}dydwdrds\\
    &\quad+\gamma^5\brac{v_3^{(k)}-v_1^{(k)} }\int_{k\eta}^t\int_{k\eta}^s \int_{k\eta}^r\int_{k\eta}^w e^{-\gamma(w- y)}(y-k\eta)dydwdrds\\
    &\qquad-\gamma^3\sqrt{2\gamma}\int_{k\eta}^t\int_{k\eta}^s \int_{k\eta}^r \int_{k\eta}^w e^{-\gamma(w- y)}dB_ydwdrds\Bigg] \\
    &\quad+\Bigg[\gamma v_3^{(k)}\int_{k\eta}^t e^{-\gamma(s-k\eta)}ds+\gamma\sqrt{2\gamma}\int_{k\eta}^t\int_{k\eta}^s e^{-\gamma(s- r)}dB_rds\\
    &\qquad-\gamma^2 v_2^{(k)}\int_{k\eta}^t\int_{k\eta}^s e^{-\gamma(s- r)}drds+\gamma^3 v_1^{(k)}\int_{k\eta}^t\int_{k\eta}^s e^{-\gamma(s- r)}(r-k\eta)drds\\
    &\quad-\gamma^3\int_{k\eta}^t\int_{k\eta}^s\int_{k\eta}^r\int_{k\eta}^w e^{-\gamma(s- r)}\tilde{g}(y)dydwdrds+\frac{\gamma^4}{2!} v_2^{(k)}\int_{k\eta}^t\int_{k\eta}^s e^{-\gamma(s- r)}{(r-k\eta)^2}drds\\
    &+\frac{\gamma^5}{3!} \brac{v_3^{(k)}-v_1^{(k)} }\int_{k\eta}^t\int_{k\eta}^s e^{-\gamma(s- r)}{(r-k\eta)^3}drds-\gamma^3 v_3^{(k)} \int_{k\eta}^t \int_{k\eta}^s\int_{k\eta}^r e^{-\gamma(s- r)}e^{-\gamma(w-k\eta)}dwdrds\\
    &\quad+\gamma^4 v_2^{(k)}\int_{k\eta}^t \int_{k\eta}^s \int_{k\eta}^r \int_{k\eta}^w e^{-\gamma(s- r)}e^{-\gamma(w- y)}dydwdrds\\
   & \qquad+\gamma^5 \brac{v_3^{(k)}-v_1^{(k)} }\int_{k\eta}^t\int_{k\eta}^s\int_{k\eta}^r\int_{k\eta}^w e^{-\gamma(s- r)}e^{-\gamma(w- y)}(y-k\eta)dydwdrds\\
   &\qquad\quad-\gamma^3\sqrt{2\gamma}\int_{k\eta}^t\int_{k\eta}^s\int_{k\eta}^r\int_{k\eta}^w e^{-\gamma(s- y)}e^{-\gamma(w- r)}dB_ydwdrds.   \Bigg],\\
    &\bar{v}_3(t)=\bar{v}_3^{(k)}e^{-\gamma(t-k\eta)}-\gamma\int_{k\eta}^t e^{-\gamma(t- s)}\bar{v}_2(s)ds+\sqrt{2\gamma}\int_{k\eta}^t e^{-\gamma(t- s)}dB_s\\   
    &\qquad=\bar{v}_3^{(k)}e^{-\gamma(t-k\eta)}+\sqrt{2\gamma}\int_{k\eta}^t e^{-\gamma(t- s)}dB_s\\
    &+\Bigg[-\gamma v_2^{(k)}\int_{k\eta}^t e^{-\gamma(t- s)}ds+\gamma^2 v_1^{(k)}\int_{k\eta}^t e^{-\gamma(t- s)}(s-k\eta)ds-\gamma^2\int_{k\eta}^t\int_{k\eta}^s \int_{k\eta}^r e^{-\gamma(t- s)}\bar{g}(w)dwdrds\\
    &\quad+\frac{\gamma^3}{2!} v_2^{(k)}\int_{k\eta}^t e^{-\gamma(t- s)}{(s-k\eta)^2}ds-\frac{\gamma^4}{3!} v_1^{(k)}\int_{k\eta}^t e^{-\gamma(t- s)}{(s-k\eta)^3}ds\\
    &\qquad+\gamma^4\int_{k\eta}^t\int_{k\eta}^s \int_{k\eta}^r\int_{k\eta}^w\int_{k\eta}^ye^{-\gamma(t- s)}\tilde{g}(z)dzdydwdrds\\
    &\quad-\frac{\gamma^5}{4!} v_2^{(k)}\int_{k\eta}^t e^{-\gamma(t- s)}{(s-k\eta)^4}ds-\frac{\gamma^6}{5!}\brac{v_3^{(k)}-v_1^{(k)} }\int_{k\eta}^t e^{-\gamma(t- s)}{(s-k\eta)^5}ds\\
    &\qquad+\gamma^4v_3^{(k)}\int_{k\eta}^t\int_{k\eta}^s\int_{k\eta}^r\int_{k\eta}^w  e^{-\gamma(t- s)} e^{-\gamma(y-k\eta)}dydwdrds\\
    &\quad-\gamma^5 v_2^{(k)} \int_{k\eta}^t\int_{k\eta}^s\int_{k\eta}^r\int_{k\eta}^w\int_{k\eta}^y  e^{-k(t- s)}e^{-\gamma(y- z)}dzdydwdrds\\
    &\qquad-\gamma^6\brac{v_3^{(k)}-v_1^{(k)} }\int_{k\eta}^t\int_{k\eta}^s \int_{k\eta}^r\int_{k\eta}^w\int_{k\eta}^y e^{-\gamma(t- s)} e^{-\gamma(y- z)}(z-k\eta)dzdydwdrds\\
    &\quad+\gamma^4\sqrt{2\gamma}\int_{k\eta}^t\int_{k\eta}^s \int_{k\eta}^r \int_{k\eta}^w\int_{k\eta}^y  e^{-\gamma(t- s)}e^{-\gamma(y- z)}dB_zdydwdrds\Bigg] \\
    &+\Bigg[-\gamma^2 v_3^{(k)}\int_{k\eta}^t\int_{k\eta}^s e^{-\gamma(t- s)}e^{-\gamma(r-k\eta)}drds-\gamma^2\sqrt{2\gamma}\int_{k\eta}^t\int_{k\eta}^s\int_{k\eta}^r e^{-\gamma(t- s)} e^{-\gamma(r- w)}dB_wdrds\\
    &+\gamma^3 v_2^{(k)}\int_{k\eta}^t\int_{k\eta}^s\int_{k\eta}^r e^{-\gamma(t- s)}e^{-\gamma(r- w)}dwdrds-\gamma^4 v_1^{(k)}\int_{k\eta}^t\int_{k\eta}^s\int_{k\eta}^r e^{-\gamma(t- s)}e^{-\gamma(r- w)}(w-k\eta)dwdrds\\
    &\qquad+\gamma^4\int_{k\eta}^t\int_{k\eta}^s\int_{k\eta}^r\int_{k\eta}^w\int_{k\eta}^y e^{-\gamma(t- s)}e^{-\gamma(r- w)}\tilde{g}(z)dzdydwdrds\\
    &\qquad\quad-\frac{\gamma^5}{2!} v_2^{(k)}\int_{k\eta}^t\int_{k\eta}^s\int_{k\eta}^r e^{-\gamma(t- s)} e^{-\gamma(r- w)}{(w-k\eta)^2}dwdrds\\
    &\quad-\frac{\gamma^6}{3!} \brac{v_3^{(k)}-v_1^{(k)} }\int_{k\eta}^t\int_{k\eta}^s\int_{k\eta}^r e^{-\gamma(t- s)}e^{-\gamma(r- w)}{(w-k\eta)^3}dwdrds\\
    &\qquad+\gamma^4 v_3^{(k)} \int_{k\eta}^t \int_{k\eta}^s\int_{k\eta}^r\int_{k\eta}^w e^{-\gamma(t- s)}e^{-\gamma(r- w)}e^{-\gamma(y-k\eta)}dydwdrds\\
    &\quad-\gamma^5 v_2^{(k)}\int_{k\eta}^t \int_{k\eta}^s \int_{k\eta}^r \int_{k\eta}^w\int_{k\eta}^ye^{-\gamma(t- s)} e^{-\gamma(r- w)}e^{-\gamma(y- z)}dzdydwdrds\\
   &\qquad -\gamma^6 \brac{v_3^{(k)}-v_1^{(k)} }\int_{k\eta}^t\int_{k\eta}^s\int_{k\eta}^r\int_{k\eta}^w\int_{k\eta}^y e^{-\gamma(t- s)}e^{-\gamma(r- w)}e^{-\gamma(y- z)}(z-k\eta)dzdydwdrds\\
&\qquad\quad+\gamma^4\sqrt{2\gamma}\int_{k\eta}^t\int_{k\eta}^s\int_{k\eta}^r\int_{k\eta}^w\int_{k\eta}^y e^{-\gamma(t- s)}e^{-\gamma(y- z)}e^{-\gamma(r- w)}dB_zdydwdrds   \Bigg].
\end{align*}

Via the above equations and software to evaluate the iterated integrals, we obtain the following result.
\begin{lemma}
\label{lemma_explicitformofxbar}
  Recall the definition of polynomials $\tilde{g}(t)$ and $\bar{g}(t)$ in \eqref{def_polynomialsg}. The following are explicit expressions of components of $\bar{x}((k+1)\eta)$ in terms of $x^{(k)}$.

\begin{align*}
    & \bar{\theta}((k+1)\eta)=-\int_{k\eta}^{(k+1)\eta}\int_{k\eta}^s \bar{g}(r)drds+\gamma^2\int_{k\eta}^{(k+1)\eta}\int_{k\eta}^s \int_{k\eta}^r\int_{k\eta}^w\tilde{g}(y)dydwdrds\\
     &\qquad\qquad\qquad\qquad+f_0+\theta^{(k)}\mu_{00}+v_1^{(k)}\mu_{01}+v_2^{(k)}\mu_{02}+v_3^{(k)}\mu_{03};\\
&\bar{v}_1((k+1)\eta)=-\int_{k\eta}^{(k+1)\eta} \bar{g}(s)ds+\gamma^2\int_{k\eta}^{(k+1)\eta}\int_{k\eta}^s \int_{k\eta}^r\tilde{g}(w)dwdrds\\
&\qquad\qquad\qquad\qquad\quad+f_1+\theta^{(k)}\mu_{10}+v_1^{(k)}\mu_{11}+v_2^{(k)}\mu_{12}+v_3^{(k)}\mu_{13};\\
    &\bar{v}_2((k+1)\eta)=\gamma\int_{k\eta}^{(k+1)\eta}\int_{k\eta}^s \bar{g}(r)drds-\gamma^3\int_{k\eta}^{(k+1)\eta}\int_{k\eta}^s \int_{k\eta}^r\int_{k\eta}^w\tilde{g}(y)dydwdrds\\
           & \quad-\gamma^3\int_{k\eta}^{(k+1)\eta}\int_{k\eta}^s\int_{k\eta}^r\int_{k\eta}^w e^{-\gamma(s- r)}\tilde{g}(y)dydwdrds+f_2 +\theta^{(k)}\mu_{20}+v_1^{(k)}\mu_{21}+v_2^{(k)}\mu_{22}+v_3^{(k)}\mu_{23};\\
    &\bar{v}_3((k+1)\eta)=-\gamma^2\int_{k\eta}^{(k+1)\eta}\int_{k\eta}^s \int_{k\eta}^r e^{-\gamma((k+1)\eta- s)} \bar{g}(w)dwdrds\\
    & \qquad+\gamma^4\int_{k\eta}^{(k+1)\eta}\int_{k\eta}^s \int_{k\eta}^r\int_{k\eta}^w\int_{k\eta}^ye^{-\gamma((k+1)\eta- s)}\tilde{g}(z)dzdydwdrds\\
      & \qquad \quad+\gamma^4\int_{k\eta}^{(k+1)\eta}\int_{k\eta}^s\int_{k\eta}^r\int_{k\eta}^w\int_{k\eta}^y e^{-\gamma((k+1)\eta- s)}e^{-\gamma(r- w)}\tilde{g}(z)dzdydwdrds\\
      & \qquad\qquad \quad+f_3 +\theta^{(k)}\mu_{30}+v_1^{(k)}\mu_{31}+v_2^{(k)}\mu_{32}+v_3^{(k)}\mu_{33}. 
\end{align*}

Here $\mu_{ij}, 0\leq i,j\leq 3$ are the constants 
\begin{align*}
    \mu_{00}&=1;\\
    \mu_{01}&=\eta-\gamma^2 \frac{\eta^3}{3!}+ \gamma^4\frac{\eta^5}{5!}+\gamma^4\brac{-\frac{e^{-\gamma  \eta }}{\gamma ^5}+\frac{1}{\gamma ^5}-\frac{\eta }{\gamma    ^4}+\frac{\eta ^2}{2 \gamma ^3}-\frac{\eta ^3}{6 \gamma ^2}+\frac{\eta ^4}{24    \gamma }};\\
    \mu_{02}&=\gamma \frac{\eta^2}{2!}-\gamma^3 \frac{\eta^4}{4!}-\gamma^3  \brac{\frac{e^{-\gamma  \eta }}{\gamma ^4}-\frac{1}{\gamma ^4}+\frac{\eta }{\gamma ^3}-\frac{\eta ^2}{2 \gamma ^2}+\frac{\eta ^3}{6 \gamma }};\\
    \mu_{03}&=-\gamma^4\frac{\eta^5}{5!}+\gamma^2 \brac{-\frac{e^{-\gamma  \eta }}{\gamma ^3}+\frac{1}{\gamma ^3}-\frac{\eta }{\gamma    ^2}+\frac{\eta ^2}{2 \gamma }}\\
&\hspace{10em}-\gamma^4\brac{-\frac{e^{-\gamma  \eta }}{\gamma ^5}+\frac{1}{\gamma ^5}-\frac{\eta }{\gamma    ^4}+\frac{\eta ^2}{2 \gamma ^3}-\frac{\eta ^3}{6 \gamma ^2}+\frac{\eta ^4}{24    \gamma }};\\
 \mu_{10}&=0;\\
    \mu_{11}&=1-\gamma^2 \frac{\eta^2}{2!}+\gamma^4\frac{\eta^4}{4!}+\gamma^4\brac{\frac{e^{-\gamma  \eta }}{\gamma ^4}-\frac{1}{\gamma ^4}+\frac{\eta }{\gamma    ^3}-\frac{\eta ^2}{2 \gamma ^2}+\frac{\eta ^3}{6 \gamma }};\\
    \mu_{12}&= \gamma\eta-\gamma^3 \frac{\eta^3}{3!} -\gamma^3  \brac{-\frac{e^{-\gamma  \eta }}{\gamma ^3}+\frac{1}{\gamma ^3}-\frac{\eta }{\gamma    ^2}+\frac{\eta ^2}{2 \gamma }};\\
    \mu_{13}&=-\gamma^4\frac{\eta^4}{4!}+\gamma^2 \brac{\frac{e^{-\gamma  \eta }}{\gamma ^2}-\frac{1}{\gamma ^2}+\frac{\eta }{\gamma }}-\gamma^4\brac{\frac{e^{-\gamma  \eta }}{\gamma ^4}-\frac{1}{\gamma ^4}+\frac{\eta }{\gamma    ^3}-\frac{\eta ^2}{2 \gamma ^2}+\frac{\eta ^3}{6 \gamma }} ;\\
    \mu_{20}&=0;\\ 
   \mu_{21}&=-\gamma \eta+\gamma^3 \frac{\eta^3}{3!}-\gamma^5 \frac{\eta^5}{5!} -\gamma^5\brac{-\frac{e^{-\gamma  \eta }}{\gamma ^5}+\frac{1}{\gamma ^5}-\frac{\eta }{\gamma    ^4}+\frac{\eta ^2}{2 \gamma ^3}-\frac{\eta ^3}{6 \gamma ^2}+\frac{\eta ^4}{24    \gamma }}\\
   &+\gamma^3 \brac{-\frac{e^{-\gamma  \eta }}{\gamma ^3}+\frac{1}{\gamma ^3}-\frac{\eta }{\gamma    ^2}+\frac{\eta ^2}{2 \gamma }}-\frac{\gamma^5}{3!} \brac{-\frac{6 e^{-\gamma  \eta }}{\gamma ^5}+\frac{6}{\gamma ^5}-\frac{6 \eta }{\gamma    ^4}+\frac{3 \eta ^2}{\gamma ^3}-\frac{\eta ^3}{\gamma ^2}+\frac{\eta ^4}{4 \gamma    }}\\
      & -\gamma^5\brac{\frac{4 e^{-\gamma  \eta }}{\gamma ^5}-\frac{4}{\gamma ^5}+\frac{\eta  e^{-\gamma     \eta }}{\gamma ^4}+\frac{3 \eta }{\gamma ^4}-\frac{\eta ^2}{\gamma ^3}+\frac{\eta    ^3}{6 \gamma ^2}};\\
   \mu_{22}&=1-\gamma^2 \frac{\eta^2}{2}+\gamma^4 \frac{\eta^4}{4!}+\gamma^4  \brac{\frac{e^{-\gamma  \eta }}{\gamma ^4}-\frac{1}{\gamma ^4}+\frac{\eta }{\gamma ^3}-\frac{\eta ^2}{2 \gamma ^2}+\frac{\eta ^3}{6 \gamma }}-\gamma^2 \brac{\frac{e^{-\gamma  \eta }}{\gamma ^2}-\frac{1}{\gamma ^2}+\frac{\eta }{\gamma }}\\
      &+\frac{\gamma^4}{2!} \brac{\frac{2 e^{-\gamma  \eta }}{\gamma ^4}-\frac{2}{\gamma ^4}+\frac{2 \eta }{\gamma    ^3}-\frac{\eta ^2}{\gamma ^2}+\frac{\eta ^3}{3 \gamma }}+\gamma^4 \brac{-\frac{3 e^{-\gamma  \eta }}{\gamma ^4}+\frac{3}{\gamma ^4}-\frac{\eta  e^{-\gamma     \eta }}{\gamma ^3}-\frac{2 \eta }{\gamma ^3}+\frac{\eta ^2}{2 \gamma ^2}};\\
   \mu_{23}&=\gamma^5\frac{\eta^5}{5!}-\gamma^3 \brac{-\frac{e^{-\gamma  \eta }}{\gamma ^3}+\frac{1}{\gamma ^3}-\frac{\eta }{\gamma    ^2}+\frac{\eta ^2}{2 \gamma }}\\
      &+\gamma^5\brac{-\frac{e^{-\gamma  \eta }}{\gamma ^5}+\frac{1}{\gamma ^5}-\frac{\eta }{\gamma    ^4}+\frac{\eta ^2}{2 \gamma ^3}-\frac{\eta ^3}{6 \gamma ^2}+\frac{\eta ^4}{24    \gamma }}+(1-e^{-\gamma \eta})\\
      &+\frac{\gamma^5}{3!}\brac{-\frac{6 e^{-\gamma  \eta }}{\gamma ^5}+\frac{6}{\gamma ^5}-\frac{6 \eta }{\gamma    ^4}+\frac{3 \eta ^2}{\gamma ^3}-\frac{\eta ^3}{\gamma ^2}+\frac{\eta ^4}{4 \gamma    }}\\
      &+\gamma^5 \brac{\frac{4 e^{-\gamma  \eta }}{\gamma ^5}-\frac{4}{\gamma ^5}+\frac{\eta  e^{-\gamma     \eta }}{\gamma ^4}+\frac{3 \eta }{\gamma ^4}-\frac{\eta ^2}{\gamma ^3}+\frac{\eta    ^3}{6 \gamma ^2}};\\
          \mu_{30}&=0;\\
     \mu_{31}&=\gamma^2 \brac{\frac{e^{-\gamma  \eta }}{\gamma ^2}-\frac{1}{\gamma ^2}+\frac{\eta }{\gamma }}-\frac{\gamma^4}{3!} \brac{\frac{6 e^{-\gamma  \eta }}{\gamma ^4}-\frac{6}{\gamma ^4}+\frac{6 \eta }{\gamma    ^3}-\frac{3 \eta ^2}{\gamma ^2}+\frac{\eta ^3}{\gamma }}\\
   &+\frac{\gamma^6}{5!}\brac{\frac{120 e^{-\gamma  \eta }}{\gamma ^6}-\frac{120}{\gamma ^6}+\frac{120 \eta    }{\gamma ^5}-\frac{60 \eta ^2}{\gamma ^4}+\frac{20 \eta ^3}{\gamma ^3}-\frac{5    \eta ^4}{\gamma ^2}+\frac{\eta ^5}{\gamma }}\\
       &+\gamma^6\brac{-\frac{5 e^{-\gamma  \eta }}{\gamma ^6}+\frac{5}{\gamma ^6}-\frac{\eta  e^{-\gamma     \eta }}{\gamma ^5}-\frac{4 \eta }{\gamma ^5}+\frac{3 \eta ^2}{2 \gamma    ^4}-\frac{\eta ^3}{3 \gamma ^3}+\frac{\eta ^4}{24 \gamma ^2}}\\
           &-\gamma^4 \brac{-\frac{3 e^{-\gamma  \eta }}{\gamma ^4}+\frac{3}{\gamma ^4}-\frac{\eta  e^{-\gamma     \eta }}{\gamma ^3}-\frac{2 \eta }{\gamma ^3}+\frac{\eta ^2}{2 \gamma ^2}}\\    
           &+\frac{\gamma^6}{3!} \brac{-\frac{30 e^{-\gamma  \eta }}{\gamma ^6}+\frac{30}{\gamma ^6}-\frac{6 \eta     e^{-\gamma  \eta }}{\gamma ^5}-\frac{24 \eta }{\gamma ^5}+\frac{9 \eta ^2}{\gamma    ^4}-\frac{2 \eta ^3}{\gamma ^3}+\frac{\eta ^4}{4 \gamma ^2}}\\
              & +\gamma^6 \brac{\frac{10 e^{-\gamma  \eta }}{\gamma ^6}-\frac{10}{\gamma ^6}+\frac{4 \eta  e^{-\gamma     \eta }}{\gamma ^5}+\frac{6 \eta }{\gamma ^5}+\frac{\eta ^2 e^{-\gamma  \eta }}{2    \gamma ^4}-\frac{3 \eta ^2}{2 \gamma ^4}+\frac{\eta ^3}{6 \gamma ^3}};\\
      \mu_{32}&=-\gamma \brac{\frac{1}{\gamma }-\frac{e^{-\gamma  \eta }}{\gamma }}+\frac{\gamma^3}{2!} \brac{-\frac{2 e^{-\gamma  \eta }}{\gamma ^3}+\frac{2}{\gamma ^3}-\frac{2 \eta }{\gamma    ^2}+\frac{\eta ^2}{\gamma }}\\
      &-\frac{\gamma^5}{4!} \brac{-\frac{24 e^{-\gamma  \eta }}{\gamma ^5}+\frac{24}{\gamma ^5}-\frac{24 \eta }{\gamma    ^4}+\frac{12 \eta ^2}{\gamma ^3}-\frac{4 \eta ^3}{\gamma ^2}+\frac{\eta ^4}{\gamma    }}\\
          &-\gamma^5  \brac{\frac{4 e^{-\gamma  \eta }}{\gamma ^5}-\frac{4}{\gamma ^5}+\frac{\eta  e^{-\gamma     \eta }}{\gamma ^4}+\frac{3 \eta }{\gamma ^4}-\frac{\eta ^2}{\gamma ^3}+\frac{\eta    ^3}{6 \gamma ^2}}\\
              &+\gamma^3 \brac{\frac{2 e^{-\gamma  \eta }}{\gamma ^3}-\frac{2}{\gamma ^3}+\frac{\eta  e^{-\gamma     \eta }}{\gamma ^2}+\frac{\eta }{\gamma ^2}}\\
                  &-\frac{\gamma^5}{2!} \brac{\frac{8 e^{-\gamma  \eta }}{\gamma ^5}-\frac{8}{\gamma ^5}+\frac{2 \eta  e^{-\gamma     \eta }}{\gamma ^4}+\frac{6 \eta }{\gamma ^4}-\frac{2 \eta ^2}{\gamma    ^3}+\frac{\eta ^3}{3 \gamma ^2}}\\
                      &-\gamma^5\bigg(-\frac{6 e^{-\gamma  \eta }}{\gamma ^5}+\frac{6}{\gamma ^5}-\frac{3 \eta  e^{-\gamma 
   \eta }}{\gamma ^4}-\frac{3 \eta }{\gamma ^4}-\frac{\eta ^2 e^{-\gamma  \eta }}{2
   \gamma ^3}+\frac{\eta ^2}{2 \gamma ^3} \bigg);\\
       \mu_{33}&=e^{-\gamma \eta} -\frac{\gamma^6}{5!}\brac{\frac{120 e^{-\gamma  \eta }}{\gamma ^6}-\frac{120}{\gamma ^6}+\frac{120 \eta    }{\gamma ^5}-\frac{60 \eta ^2}{\gamma ^4}+\frac{20 \eta ^3}{\gamma ^3}-\frac{5    \eta ^4}{\gamma ^2}+\frac{\eta ^5}{\gamma }}\\
     &+\gamma^4\brac{-\frac{3 e^{-\gamma  \eta }}{\gamma ^4}+\frac{3}{\gamma ^4}-\frac{\eta  e^{-\gamma     \eta }}{\gamma ^3}-\frac{2 \eta }{\gamma ^3}+\frac{\eta ^2}{2 \gamma ^2}}\\
         &-\gamma^6\brac{-\frac{5 e^{-\gamma  \eta }}{\gamma ^6}+\frac{5}{\gamma ^6}-\frac{\eta  e^{-\gamma     \eta }}{\gamma ^5}-\frac{4 \eta }{\gamma ^5}+\frac{3 \eta ^2}{2 \gamma    ^4}-\frac{\eta ^3}{3 \gamma ^3}+\frac{\eta ^4}{24 \gamma ^2}}-\gamma^2 \brac{-\frac{e^{-\gamma  \eta }}{\gamma ^2}+\frac{1}{\gamma ^2}-\frac{\eta  e^{-\gamma     \eta }}{\gamma }}\\
             &-\frac{\gamma^6}{3!} \brac{-\frac{30 e^{-\gamma  \eta }}{\gamma ^6}+\frac{30}{\gamma ^6}-\frac{6 \eta     e^{-\gamma  \eta }}{\gamma ^5}-\frac{24 \eta }{\gamma ^5}+\frac{9 \eta ^2}{\gamma    ^4}-\frac{2 \eta ^3}{\gamma ^3}+\frac{\eta ^4}{4 \gamma ^2}}\\
                 &+\gamma^4 \brac{\frac{3 e^{-\gamma  \eta }}{\gamma ^4}-\frac{3}{\gamma ^4}+\frac{2 \eta  e^{-\gamma     \eta }}{\gamma ^3}+\frac{\eta }{\gamma ^3}+\frac{\eta ^2 e^{-\gamma  \eta }}{2    \gamma ^2}}\\
                    & -\gamma^6\brac{\frac{10 e^{-\gamma  \eta }}{\gamma ^6}-\frac{10}{\gamma ^6}+\frac{4 \eta  e^{-\gamma     \eta }}{\gamma ^5}+\frac{6 \eta }{\gamma ^5}+\frac{\eta ^2 e^{-\gamma  \eta }}{2    \gamma ^4}-\frac{3 \eta ^2}{2 \gamma ^4}+\frac{\eta ^3}{6 \gamma ^3}};\\
\end{align*}

Meanwhile, $f_i,0\leq i\leq 4$ are It\^{o} integrals defined via 
\begin{align*}
    f_0&=\gamma^2\sqrt{2\gamma}\int_{k\eta}^{(k+1)\eta} \frac{\gamma  (y-\eta  (k+1)) (-\gamma  \eta  (k+1)+\gamma  y+2)-2 e^{\gamma  (y-\eta
    (k+1))}+2}{2 \gamma ^3} dB_y;\\
    f_1&=\gamma^2\sqrt{2\gamma}\int_{k\eta}^{(k+1)\eta}\frac{\gamma  \eta +\gamma  \eta  k+e^{\gamma  (y-\eta  (k+1))}-\gamma y-1}{\gamma ^2}dB_y;\\
    f_2&=-\gamma^3\sqrt{2\gamma}\int_{k\eta}^{(k+1)\eta}\frac{\gamma  (y-\eta  (k+1)) (-\gamma  \eta  (k+1)+\gamma  y+2)-2 e^{\gamma  (y-\eta(k+1))}+2}{2 \gamma ^3}dB_y\\
    &\qquad\qquad\qquad\qquad-\gamma\sqrt{2\gamma}\int_{k\eta}^{(k+1)\eta}\frac{e^{\gamma  (y-\eta  (k+1))}-1}{\gamma }dB_y\\
    &\qquad-\gamma^3\sqrt{2\gamma}\int_{k\eta}^{(k+1)\eta}\frac{\gamma  \eta +\gamma  \eta  k+e^{\gamma  (y-\eta  (k+1))} (\gamma  \eta 
   (k+1)+\gamma  (-y)+2)+\gamma  (-y)-2}{\gamma ^3} dB_y; \\
   f_3&=\sqrt{2\gamma}\int_{k\eta}^{(k+1)\eta} e^{-\gamma((k+1)\eta- y)}dB_y+\gamma^4\sqrt{2\gamma}\int_{k\eta}^{(k+1)\eta}\frac{y-\eta  k}{2 \gamma    ^4}\bigg(2 e^{\gamma  (y-\eta  (k+1))}\\
    &\qquad\quad (-\gamma  \eta  (k+1)+\gamma     y-3)+\gamma  (y-\eta  (k+1)) (-\gamma  \eta  (k+1)+\gamma  y+4)+6\bigg) dB_y\\
        &\qquad-\gamma^2\sqrt{2\gamma}\int_{k\eta}^{(k+1)\eta} \frac{e^{\gamma  (y-\eta  (k+1))} (-\gamma  \eta  (k+1)+\gamma  y-1)+1}{\gamma ^2}dB_y\\
        &\qquad\qquad\qquad+\gamma^4\sqrt{2\gamma}\int_{k\eta}^{(k+1)\eta}\frac{1}{2 \gamma ^4}\bigg(2 \gamma  \eta  (k+1)+e^{\gamma  (y-\eta  (k+1))} (\gamma  (y-\eta  (k+1))  \\
&\hspace{15em}\times  (-\gamma  \eta  (k+1)+\gamma  y-4)+6)-2 \gamma  y-6\bigg)dB_y. 
\end{align*}
\end{lemma}

Next is the calculation for the conditional mean and covariance associated with the fourth-order LMC algorithm in Section~\ref{section_4thorder}. 
\begin{lemma}
    \label{lemma_meanandcovariance}
 $\E{x^{(k+1)}|x^{(k)}}$ is a multivariate normal distribution with mean $\textbf{M}(x^{(k)})=(m_i)_{0\leq i\leq 3}$ and symmetric covariance matrix ${\boldsymbol{\Sigma}}=\brac{\sigma_{ij}\cdot I_d}_{0\leq i,j\leq 3}$. 
    
    The entries $m_i, 0\leq i\leq 3$ and  $\sigma_{ij}, 0\leq i,j\leq 3$ are provided below, noting that the constants $(\mu_{ij})_{0\leq i,j\leq 3}$ are defined in Lemma~\ref{lemma_explicitformofxbar}.\\ 
\begin{align*}
    m_0&=-\int_{k\eta}^{(k+1)\eta}  \Big((k+1)\eta-r\Big)\bar{g}(r) \, dr-\frac{\gamma ^2}{6} \int_{k\eta}^{ (k+1)\eta}  \Big(y-(k+1)\eta\Big)^3 \tilde{g}(y) \, dy\\
    &\qquad\qquad+\theta^{(k)}\mu_{00}+v_1^{(k)}\mu_{01}+v_2^{(k)}\mu_{02}+v_3^{(k)}\mu_{03};\\
    m_1 & =-\int_{k\eta}^{(k+1)\eta} \bar{g}(s) \, ds +\frac{\gamma ^2 }{2}\int_{\eta  k}^{\eta  (k+1)}   \Big(w-\eta  (k+1)\Big)^2 \tilde{g}(w) \, dw\\
 &\qquad\qquad+\theta^{(k)}\mu_{10}+v_1^{(k)}\mu_{11}+v_2^{(k)}\mu_{12}+v_3^{(k)}\mu_{13};\\
   m_2& = \gamma  \int_{k\eta}^{(k+1)\eta}  \Big((k+1)\eta-r\Big)\bar{g}(r) \, dr-\frac{\gamma^3}{6}\int_{k\eta}^{(k+1)\eta}\Bigg(\Big(y-(k+1)\eta\Big)^3\\
   &\qquad-\frac{6-6e^{\gamma(y-(k+1)\eta}+3\gamma(y-(k+1)\eta)\Big(2+\gamma(y-(k+1)\eta)\Big)}{\gamma^3}\Bigg)dy\\
   &\qquad\qquad+\theta^{(k)}\mu_{20}+v_1^{(k)}\mu_{21}+v_2^{(k)}\mu_{22}+v_3^{(k)}\mu_{23};\\
m_3&= -\frac{\gamma ^2 }{2}\int_{k\eta}^{(k+1)\eta } \Big(w-\eta  (k+1)\Big)^2 \bar{g}(w)\, dw\\
&\quad+\gamma^4 \int_{k\eta}^{(k+1)\eta}\Bigg( \frac{1}{2\gamma^4}\Big(2 e^{\gamma  (z-(k+1)\eta)} (-\gamma  (k+1)\eta+\gamma  z-3)\\
&\quad\quad+\gamma  (z-(k+1)\eta) (-\gamma(k+1)\eta+\gamma  z+4)+6\Big)\\
&\quad\quad\quad+\frac{1}{6\gamma^4}\Big(\gamma  ((k+1)\eta-z) (\gamma  ((k+1)\eta-z) (\gamma(k+1)\eta+\gamma  (-z)-3)+6)\\
&\quad\quad\quad\quad+6 e^{\gamma  (z-(k+1)\eta)}-6\Big) \Bigg)\tilde{g}(z)dz\\
&\quad\quad\quad\quad\quad+\theta^{(k)}\mu_{30}+v_1^{(k)}\mu_{31}+v_2^{(k)}\mu_{32}+v_3^{(k)}\mu_{33}.
\end{align*}

Furthermore,
\begin{align*}
        &\sigma_{00}=\frac{\gamma ^3 \eta ^5}{10}-\frac{\gamma ^2 \eta ^4}{2}-\frac{e^{-2 \gamma  \eta
   }}{\gamma ^2}+\frac{4 e^{-\gamma  \eta }}{\gamma ^2}-\frac{3}{\gamma ^2}+\frac{4
   \gamma  \eta ^3}{3}+2 \eta ^2 e^{-\gamma  \eta }+\frac{2 \eta }{\gamma }-2 \eta ^2;\\
    &\sigma_{11}=\frac{2 \gamma ^3 \eta ^3}{3}-2 \gamma ^2 \eta ^2-4 \gamma  \eta  e^{-\gamma  \eta
   }+2 \gamma  \eta -e^{-2 \gamma  \eta }+1;\\
    &\sigma_{22}=\frac{\gamma ^5 \eta ^5}{10}-2 \gamma ^3 \eta ^3 e^{-\gamma  \eta }-\frac{4 \gamma ^3
   \eta ^3}{3}-\gamma ^2 \eta ^2 e^{-2 \gamma  \eta }-10 \gamma ^2 \eta ^2 e^{-\gamma
    \eta }-5 \gamma  \eta  e^{-2 \gamma  \eta }\\
    &\qquad\qquad-12 \gamma  \eta  e^{-\gamma  \eta }+8
   \gamma  \eta -\frac{13}{2} e^{-2 \gamma  \eta }+4 e^{-\gamma  \eta }+\frac{5}{2};\\
    &\sigma_{33}=\frac{\gamma ^5 \eta ^7}{210}-\frac{\gamma ^4 \eta ^6}{15}+\frac{\gamma ^4 \eta
   ^5}{10}-\frac{1}{4} \gamma ^4 \eta ^4 e^{-2 \gamma  \eta }+\frac{7 \gamma ^3 \eta
   ^5}{15}+\gamma ^3 \eta ^4 e^{-\gamma  \eta }-\frac{4 \gamma ^3 \eta
   ^4}{3}-\frac{7}{2} \gamma ^3 \eta ^3 e^{-2 \gamma  \eta }\\
   &\quad-2 \gamma ^3 \eta ^3
   e^{-\gamma  \eta }+\frac{2 \gamma ^3 \eta ^3}{3}-2 \gamma ^2 \eta ^4-\frac{1}{2}
   \gamma ^2 \eta ^3 e^{-2 \gamma  \eta }+10 \gamma ^2 \eta ^3 e^{-\gamma  \eta
   }+\frac{22 \gamma ^2 \eta ^3}{3}-\frac{77}{4} \gamma ^2 \eta ^2 e^{-2 \gamma  \eta
   }\\
   &-10 \gamma ^2 \eta ^2 e^{-\gamma  \eta }-8 \gamma ^2 \eta ^2-\frac{21 e^{-2
   \gamma  \eta }}{2 \gamma ^2}+\frac{192 e^{-\gamma  \eta }}{\gamma ^2}-\frac{363}{2
   \gamma ^2}+4 \gamma  \eta ^3 e^{-\gamma  \eta }+6 \gamma  \eta ^3-\frac{1}{2} \eta
   ^2 e^{-2 \gamma  \eta }\\
   &\quad+32 \eta ^2 e^{-\gamma  \eta }-6 \gamma  \eta ^2 e^{-2
   \gamma  \eta }+36 \gamma  \eta ^2 e^{-\gamma  \eta }-24 \gamma  \eta
   ^2-\frac{101}{4} \eta  e^{-2 \gamma  \eta }+84 \eta  e^{-\gamma  \eta
   }-\frac{197}{4} \gamma  \eta  e^{-2 \gamma  \eta }\\
   &+8 \gamma  \eta  e^{-\gamma 
   \eta }+32 \gamma  \eta -\frac{9 \eta  e^{-2 \gamma  \eta }}{2 \gamma }+\frac{96
   \eta  e^{-\gamma  \eta }}{\gamma }+\frac{159 \eta }{2 \gamma }-\frac{397}{8} e^{-2
   \gamma  \eta }+88 e^{-\gamma  \eta }-\frac{149 e^{-2 \gamma  \eta }}{4 \gamma
   }\\
   &\qquad\qquad+\frac{204 e^{-\gamma  \eta }}{\gamma }-\frac{667}{4 \gamma }-\frac{39 \eta
   ^2}{2}+\frac{283 \eta }{4}-\frac{307}{8};\\
    &\sigma_{01}=\sigma_{10}=\frac{\gamma ^3 \eta ^4}{4}-\gamma ^2 \eta ^3-\gamma  \eta ^2 e^{-\gamma  \eta }+2
   \gamma  \eta ^2+2 \eta  e^{-\gamma  \eta }+\frac{e^{-2 \gamma  \eta }}{\gamma
   }-\frac{2 e^{-\gamma  \eta }}{\gamma }+\frac{1}{\gamma }-2 \eta;\\
    &\sigma_{02}=\sigma_{20}=-\frac{1}{10} \gamma ^4 \eta ^5+\frac{\gamma ^3 \eta ^4}{4}+\gamma ^2 \eta ^3
   e^{-\gamma  \eta }+\frac{\gamma ^2 \eta ^3}{3}+2 \gamma  \eta ^2 e^{-\gamma  \eta
   }-2 \gamma  \eta ^2-\eta  e^{-2 \gamma  \eta }+2 \eta  e^{-\gamma  \eta }\\
   &\qquad\qquad\qquad\qquad-\frac{5
   e^{-2 \gamma  \eta }}{2 \gamma }+\frac{10 e^{-\gamma  \eta }}{\gamma }-\frac{15}{2
   \gamma }+4 \eta;\\
   & \sigma_{30}=\sigma_{03}=\frac{\gamma ^4 \eta ^6}{60}-\frac{3 \gamma ^3 \eta ^5}{20}-\frac{1}{2} \gamma ^3
   \eta ^4 e^{-\gamma  \eta }+\frac{\gamma ^3 \eta ^4}{4}+\frac{2 \gamma ^2 \eta
   ^4}{3}-4 \gamma ^2 \eta ^3 e^{-\gamma  \eta }-2 \gamma ^2 \eta ^3+\frac{2 e^{-2
   \gamma  \eta }}{\gamma ^2}\\
   &\quad-\frac{32 e^{-\gamma  \eta }}{\gamma
   ^2}+\frac{30}{\gamma ^2}-\gamma  \eta ^3 e^{-\gamma  \eta }-\frac{5 \gamma  \eta
   ^3}{3}-8 \eta ^2 e^{-\gamma  \eta }+\frac{1}{2} \gamma  \eta ^2 e^{-2 \gamma  \eta
   }-12 \gamma  \eta ^2 e^{-\gamma  \eta }+5 \gamma  \eta ^2\\
   &\quad+\frac{7}{2} \eta  e^{-2
   \gamma  \eta }-18 \eta  e^{-\gamma  \eta }+\frac{\eta  e^{-2 \gamma  \eta }}{2
   \gamma }-\frac{18 \eta  e^{-\gamma  \eta }}{\gamma }-\frac{21 \eta }{2 \gamma
   }+\frac{27 e^{-2 \gamma  \eta }}{4 \gamma }-\frac{36 e^{-\gamma  \eta }}{\gamma
   }+\frac{117}{4 \gamma }+3 \eta ^2-8 \eta;\\
    &\sigma_{12}=\sigma_{21}=-\frac{1}{4} \gamma ^4 \eta ^4+\frac{\gamma ^3 \eta ^3}{3}+3 \gamma \
^2 \eta ^2
   e^{-\gamma  \eta }+2 \gamma ^2 \eta ^2+\gamma  \eta  \
e^{-2 \gamma  \eta }+8 \gamma
    \eta  e^{-\gamma  \eta }-4 \gamma  \
\eta +\frac{5}{2} e^{-2 \gamma  \eta
   }-\frac{5}{2};\\
    &\sigma_{13}=\sigma_{31}=\frac{\gamma ^4 \eta ^5}{20}-\frac{5 \gamma ^3 \eta ^4}{12}-\gamma ^3 \eta ^3
   e^{-\gamma  \eta }+\frac{2 \gamma ^3 \eta ^3}{3}+\frac{5 \gamma ^2 \eta
   ^3}{3}-\frac{1}{2} \gamma ^2 \eta ^2 e^{-2 \gamma  \eta }-8 \gamma ^2 \eta ^2
   e^{-\gamma  \eta }-5 \gamma ^2 \eta ^2\\
   &\quad-\gamma  \eta ^2 e^{-\gamma  \eta }-3 \gamma
    \eta ^2-\frac{1}{2} \eta  e^{-2 \gamma  \eta }-12 \eta  e^{-\gamma  \eta
   }-\frac{7}{2} \gamma  \eta  e^{-2 \gamma  \eta }-22 \gamma  \eta  e^{-\gamma  \eta
   }+8 \gamma  \eta -\frac{27}{4} e^{-2 \gamma  \eta }\\
   &\qquad\qquad\qquad\qquad-4 e^{-\gamma  \eta }-\frac{2
   e^{-2 \gamma  \eta }}{\gamma }-\frac{10 e^{-\gamma  \eta }}{\gamma
   }+\frac{12}{\gamma }-\frac{3 \eta }{2}+\frac{43}{4};\\
   & \sigma_{23}=\sigma_{32}=-\frac{1}{60} \gamma ^5 \eta ^6+\frac{\gamma ^4 \eta ^5}{10}+\frac{1}{2} \gamma ^4
   \eta ^4 e^{-\gamma  \eta }-\frac{\gamma ^4 \eta ^4}{4}-\frac{\gamma ^3 \eta
   ^4}{12}+\frac{1}{2} \gamma ^3 \eta ^3 e^{-2 \gamma  \eta }+5 \gamma ^3 \eta ^3
   e^{-\gamma  \eta }\\
   &\qquad+\frac{4 \gamma ^3 \eta ^3}{3}-\frac{4 \gamma ^2 \eta
   ^3}{3}+\frac{19}{4} \gamma ^2 \eta ^2 e^{-2 \gamma  \eta }+20 \gamma ^2 \eta ^2
   e^{-\gamma  \eta }+2 \gamma ^2 \eta ^2+\frac{1}{2} \gamma  \eta ^2 e^{-2 \gamma 
   \eta }+5 \gamma  \eta ^2 e^{-\gamma  \eta }\\
   &\qquad\quad+6 \gamma  \eta ^2+\frac{7}{2} \eta 
   e^{-2 \gamma  \eta }+18 \eta  e^{-\gamma  \eta }+\frac{63}{4} \gamma  \eta  e^{-2
   \gamma  \eta }+24 \gamma  \eta  e^{-\gamma  \eta }-16 \gamma  \eta +\frac{143}{8}
   e^{-2 \gamma  \eta }\\
   &\qquad\qquad-12 e^{-\gamma  \eta }+\frac{25 e^{-2 \gamma  \eta }}{4 \gamma
   }+\frac{2 e^{-\gamma  \eta }}{\gamma }-\frac{33}{4 \gamma }-7 \eta -\frac{47}{8}.
\end{align*}
\end{lemma}

\begin{proof}
The formula for $m_i$ immediately follows from Lemma~\ref{lemma_explicitformofxbar}, so that we only need to show how to compute the entries of the covariance matrix. We have 
\begin{align}\label{eqn:sigma:00}
\E{\brac{\theta^{(k+1)}-\E{\theta^{(k+1)}|x^{(k)}}}\brac{\theta^{(k+1)}-\E{\theta^{(k+1)}|x^{(k)}}}^\top\Big| x^{(k)} }=\E{f_0(f_0)^\top}=\sigma_{00}\cdot I_d,
\end{align}
where $f_0$ is defined in Lemma~\ref{lemma_explicitformofxbar}. Then $\sigma_{00}$ on the right hand side of \eqref{eqn:sigma:00} can be computed by It\^{o} isometry and software. The remaining covariance entries $\sigma_{ij}$'s are obtained in the same way via  $\E{f_i(f_j)^\top}=\sigma_{ij}\cdot I_d$. 
\end{proof}

Based on Lemma~\ref{lemma_explicitformofxbar}, it possible to decompose $\bar{x}((k+1)\eta)$ into higher and lower order terms with respect to $\eta$, as the next lemma will show. 
\begin{lemma}
\label{lemma_decomposexk+1}
Recall the unique minimizer $\theta^*$ of $U$, $\textbf{M}(x^{(k)})$ from Lemma~\ref{lemma_meanandcovariance} and the Jacobian matrix 
 \\$ J_b(\theta^*, 0,\ldots,0) = \begin{bmatrix}
    0 & I_d & 0  & 0 \\
    -\nabla^2 U(\theta^*)I_d & 0 & \gamma   & 0 \\
    0 & -\gamma I_d & 0  & \gamma \\
    0 & 0 & -\gamma I_d & -\gamma I_d
\end{bmatrix} $. Then it holds that 
\begin{align*}
    \textbf{M}(x^{(k)})=x^{(k)}+\eta J_b(\theta^*, 0,\ldots,0)\brac{x^{(k)}-\brac{\theta^*,0,\ldots,0}}+R\brac{x^{(k)}-\brac{\theta^*,0,\ldots,0}},
\end{align*}
and
\begin{align*}
    &x^{(k+1)}-\brac{\theta^*,0,\ldots,0}=\brac{x^{(k)}-\brac{\theta^*,0,\ldots,0}}\\
    &\qquad\qquad+\eta J_b(\theta^*, 0,\ldots,0)\brac{x^{(k)}-\brac{\theta^*,0,\ldots,0}}+R\brac{x^{(k)}-\brac{\theta^*,0,\ldots,0}}+F_k,
\end{align*}
where $R$ is a $4d\times 4d$ matrix with $\abs{R_{ij}}\leq C\eta^2$, $1\leq i,j\leq 4d$ and $C$ is a constant dependent only on $\gamma$. Moreover, the entries $f_i$'s of the $4d$-dimensional vector $F_k=\brac{f_0 \quad f_1\quad f_3\quad f_4}^{\top}$ are defined in Lemma~\ref{lemma_explicitformofxbar}. 
\end{lemma}

\begin{proof}
Without loss of generality, let us assume that the unique minimizer $\theta^*=0$.  
 
First, we will rewrite $\mu_{ij},0\leq i,j\leq 3$ that appear in the formula of $\bar{x}((k+1)\eta)$ in Lemma~\ref{lemma_explicitformofxbar} and make explicit lower and higher order terms with respect to $\eta$. 
\begin{align*}
&\mu_{00}=1;\\
& \mu_{01}
=\eta-\sum_{k=2}^\infty \frac{\gamma^4(-\gamma \eta)^k}{k!}-\gamma^2 \frac{\eta^3}{3!}+ \gamma^4\frac{\eta^5}{5!}+\gamma^4\brac{\frac{\eta ^2}{2 \gamma ^3}-\frac{\eta ^3}{6 \gamma ^2}+\frac{\eta ^4}{24    \gamma }}=\eta-O(\eta^2);\\
&\mu_{02}=\gamma \frac{\eta^2}{2!}-\gamma^3 \frac{\eta^4}{4!}-\gamma^3  \brac{\sum_{k=2}^\infty \frac{(-\gamma \eta)^k}{\gamma^4 k!}-\frac{\eta ^2}{2 \gamma ^2}+\frac{\eta ^3}{6 \gamma }}=O(\eta^2);\\
&\mu_{03}=-\gamma^4\frac{\eta^5}{5!}+\frac{\gamma\eta^2}{2}-\gamma^4\brac{\frac{\eta ^2}{2 \gamma ^3}-\frac{\eta ^3}{6 \gamma ^2}+\frac{\eta ^4}{24    \gamma }}=O(\eta^2).
\end{align*}

Similarly, we have 
\begin{align*}
   & \mu_{10}=0, \quad\mu_{11}=1+O(\eta^2),\quad \mu_{12}=-\gamma\eta+O(\eta^2),\quad \mu_{13}=O(\eta^2),\\
   &\mu_{20}=0,\quad  \mu_{21}=-\gamma\eta+O(\eta^2),\quad \mu_{22}=1+O(\eta^2), \quad \mu_{23}=\eta\gamma+O(\eta^2),\\
   & \mu_{30}=0,\quad \mu_{31}=O(\eta^2),\quad  \mu_{32}=-\gamma\eta+O(\eta^2),\quad \mu_{33}=1-\gamma\eta+O(\eta^2). 
\end{align*}

Next, we consider the integral terms containing $\hat{g}$ and $\bar{g}$ in the formula of $\bar{x}((k+1)\eta)$ in Lemma~\ref{lemma_explicitformofxbar}. Since $\nabla U(\theta^*)=\nabla U(0)=0$, we can write
\begin{equation*}
-\int_{k\eta}^{(k+1)\eta} \bar{g}(s)ds=-\eta\nabla U(\theta^{(k)})+O(\eta^2)=\eta\nabla^2 U(0)I_d\theta^{(k)}+O(\eta^2).
\end{equation*}
Meanwhile, all the remaining integral terms containing $\hat{g}$ and $\bar{g}$ in $\bar{x}((k+1)\eta)$ are of the order $O(\eta^2)$. 

Consequently, we can deduce the equation in the statement of this lemma from the above calculations and Lemma~\ref{lemma_explicitformofxbar}. 
\end{proof}

We also need the following moment bounds.
\begin{lemma} 
\label{lemma_momentbound}
Assume 
    \begin{align}
    \label{def_eta*}
        \eta\leq \eta^*:=\frac{\rho}{2}\opnorm{M}^{-1}\opnorm{M^{-1}}^{-1}\frac{1}{1+5\gamma^2+L^2},
    \end{align}
 where the matrix $M$ is provided in Example~\ref{example_P4} and $\gamma,\rho,L$ are from Theorem~\ref{theorem_frommonmarche_mainpaper}. Then there exists a positive constant $C_1$ such that for all $k$, 
    \begin{align*}
           \E{\abs{x^{(k+1)}}^{2\alpha}}\leq (C_1)^\alpha (d+2\alpha)^\alpha, 
    \end{align*}
  where $C_1$ depends on $\gamma$, $L$ from Condition~\ref{cond_mainpaper} and $c$ from Condition~\ref{cond_derivativegrowthrate}, but not on $d$. 
    
This further implies 
  \begin{align*}
       \sup_{t\in [k\eta,(k+1)\eta]} \E{\abs{\hat{x}(t)}^2+\abs{\tilde{x}(t)}^2+\abs{\bar{x}(t)}^2+\abs{\tilde{g}(t)}^2+\abs{\bar{g}(t)}^2}\leq C_2\brac{d+1},
    \end{align*}
    for a universal constant $C_2\geq 1$ that depends on $\gamma$, $L$ from Condition~\ref{cond_mainpaper} and $c$ from Condition~\ref{cond_derivativegrowthrate}. $C_2$ does not depend on the dimension $d$. 
\end{lemma}

\begin{proof}

\textbf{Part 1 of the proof:}

Without loss of generality, let us assume $\theta^*=0$ and denote $J_b(0)=J_b(\theta^*,0,\ldots,0)$. Let $w_k\sim\mathcal{N}(0,I_{4d})$ then per Lemma~\ref{lemma_meanandcovariance} and Lemma~\ref{lemma_decomposexk+1}, 
\begin{align}
\label{momentboundmiddlestep}
\E{\abs{x^{(k+1)}}^{2\alpha}_M}&=\E{\abs{M^{1/2}\textbf{M}(x^{(k)})+\brac{M\boldsymbol{\Sigma}}^{1/2}w_k }^{2\alpha}}\nonumber\\
   &\leq \sum_{j=0}^{2\alpha}{2\alpha \choose j}\E{\abs{M^{1/2}\textbf{M}(x^{(k)})}^j}\E{\abs{\brac{M\boldsymbol{\Sigma}}^{1/2}w_k}^{2\alpha-j}}\nonumber\\
   &\leq \sum_{j=0}^{2\alpha}{2\alpha \choose j}\E{\abs{M^{1/2}\textbf{M}(x^{(k)})}^{2\alpha}}^{j/2\alpha}\E{\abs{\brac{M\boldsymbol{\Sigma}}^{1/2}w_k}^{2\alpha}}^{1-j/2\alpha}\nonumber\\
   &=\brac{\E{\abs{\textbf{M}(x^{(k)})}_M^{2\alpha}}^{1/\alpha}+\E{\abs{\brac{M\boldsymbol{\Sigma}}^{1/2}w_k}^{2\alpha}}^{1/\alpha} }^{2\alpha}.
\end{align}
Let us study the first term on the right hand side. Lemma~\ref{lemma_decomposexk+1} and \eqref{contraction}  imply that 
\begin{align*}
   & \abs{\textbf{M}(x^{(k)})}_M^2=\abs{x^{(k)}+\eta J_b(0) x^{(k)} }_M^2\\
&\qquad=\abs{x^{(k)}}^2_M+\eta\brac{x^{(k)}}^{\top}\brac{J_b(0)^{\top}M+MJ_b(0)}x^{(k)}+\eta^2 \brac{x^{(k)}}^{\top}J_b(0)^{\top}MJ_b(0)x^{(k)}\\
    &\qquad\quad\leq \abs{x^{(k)}}^2_M+\eta (x^{(k)})^{\top}(-2\rho)Mx^{(k)}+\eta^2 \opnorm{J_b(0)^{\top}MJ_b(0)}\opnorm{M^{-1}}\abs{x^{(k)}}_M^2 \\
   &\qquad\qquad\leq \brac{1-2\eta\rho}\abs{x^{(k)}}_M^2+\eta^2 \opnorm{J_b(0)}^2\opnorm{M}\opnorm{M^{-1}}\abs{x^{(k)}}_M^2. 
\end{align*}
At this point, notice that the operator norm is bounded from above by the Frobenius norm, so that combining with $L$-smoothness of $U$ and the explicit form of $J_b$ in Lemma~\ref{lemma_decomposexk+1}, we get $\opnorm{J_b(0)}^2\leq 1+5\gamma^2+L^2$. $\opnorm{M}$ and $\opnorm{M^{-1}}$ can be computed using the explicit form of $M$ in Example~\ref{example_P4}. Then assuming $\eta<\eta^*$ as defined in the statement of the lemma, we arrive at 
\begin{align*}
\E{\abs{\textbf{M}(x^{(k)})}_M^{2\alpha}}^{1/\alpha}\leq \brac{1-\frac{3}{2}\eta\rho}^{\alpha} \E{ \abs{x^{(k)}}_M^{2\alpha}}. 
\end{align*}
This can be combined with~\eqref{momentboundmiddlestep} and \eqref{equivalenceMnorm} to get the desired bound on $  \E{\abs{x^{(k+1)}}^{2\alpha}}$. 

\textbf{Part 2 of the proof:}

In this part, we will use the result from Part 1 and the explicit formula at the beginning of Appendix~\ref{appendix_4thorder} to bound $\sup_{t\in [k\eta,(k+1)\eta]} \E{\abs{\hat{x}(t)}^2+\abs{\tilde{x}(t)}^2+\abs{\bar{x}(t)}^2+\abs{\tilde{g}(t)}^2+\abs{\bar{g}(t)}^2}$. The upcoming argument is rather tedious, so we will only demonstrate how to bound $\sup_{s\in [k\eta,(k+1)\eta]}\E{\abs{\tilde{v}_1(t)}^2}$. By Equation~\eqref{equation_v1tilde}, we can write
\begin{align*}
  \sup_{t\in [k\eta,(k+1)\eta]}  \E{\abs{\tilde{v}_1(t)}^2}&\leq 5\bigg(\E{\abs{v^{(k)}_1}^2}+\gamma^2\eta^2 \E{\abs{v^{(k)}_2}^2}+\gamma^4\frac{\eta^4}{4}\brac{\E{\abs{v^{(k)}_1}^2}+\E{\abs{v^{(k)}_3}^2}}\\
  &\qquad\qquad\qquad\qquad+\eta^2\sup_{t\in [k\eta,(k+1)\eta]}\E{\abs{\tilde{g}(t)}}^2\bigg). 
\end{align*}
Lemma~\ref{lemma_polyapprox} and $L$-smoothness of $U$ implies 
\begin{align*}
\E{\abs{\tilde{g}(t)}}^2\leq \brac{\frac{L_\alpha}{\alpha!}}^2 \sup_{t\in [k\eta,(k+1)\eta]}\E{\abs{\hat{\theta}(t)}^{2\alpha}}+L_\alpha^2 \sup_{t\in [k\eta,(k+1)\eta]}\E{\abs{\hat{\theta}(t)}^{2}}, 
\end{align*}
where $\alpha$ is from Condition~\ref{cond_derivativegrowthrate}. Moreover, we also have from \eqref{equation_thetahat} and Part 1 of this proof that 
\begin{align*}
\sup_{t\in [k\eta,(k+1)\eta]}\E{\abs{\hat{\theta}(t)}^{2\alpha}}
&\leq \E{\brac{2\abs{\theta^{(k)}}\vee 2\abs{v_1^{(k)}} }^{2\alpha}}
\\
&\leq 4\E{\abs{\theta^{(k)}}^{2\alpha}}+4\E{\abs{v_1^{(k)}}^{2\alpha}}\leq 8(C_1)^\alpha (d+2\alpha)^\alpha.
\end{align*}
Condition~\ref{cond_derivativegrowthrate} then says $\brac{\frac{L_\alpha}{\alpha!}}^2\sup_{t\in [k\eta,(k+1)\eta]}\E{\abs{\hat{\theta}(t)}^{2\alpha}}\leq c d \eta^2$. Thus, we arrive at 
\begin{align*}
\E{\abs{\tilde{g}(t)}}^2\leq c d \eta^2+L^28C_1 (d+2). 
\end{align*}
Hence,
\begin{align*}
     \sup_{t\in [k\eta,(k+1)\eta]}  \E{\abs{\tilde{v}_1(t)}^2}\leq 5\brac{1+\gamma^2\eta^2+\gamma^4\frac{\eta^4}{4}}4C_1 (d+2)+5\eta^2\brac{c d \eta^2+L^28C_1 (d+2)}. 
\end{align*}
This completes the proof. 
\end{proof}

Next are the proofs of Lemma~\ref{lemma_polyapprox} and Lemma~\ref{lemma_boundupdatedifference_4thorder} in the main paper.

\begin{proof}[Proof of Lemma~\ref{lemma_polyapprox}]
By Condition~\ref{cond_mainpaper} and \cite[Theorem 5.6.2]{cartan1971differentialbook}, we have 
\begin{align*}
    \abs{\nabla U(x)-P_{\alpha-1}(x)}\leq L_{\alpha}\frac{\abs{x}^\alpha}{\alpha!}.
\end{align*}
This leads to
    \begin{align*}
      &\sup_{t\in [k\eta,(k+1)\eta]}  \E{\abs{\nabla U(\hat{\theta}(t))-\tilde{g}(t)}^2}=\sup_{t\in [k\eta,(k+1)\eta]}  \E{\abs{\nabla U(\hat{\theta}(t))-P_{\alpha-1}(\hat{\theta}(t))}^2}\\
      &\qquad\qquad\leq \brac{\frac{L_\alpha}{\alpha!}}^2 \sup_{t\in [k\eta,(k+1)\eta]} \E{\abs{\hat{\theta}(t)^{2\alpha}}}. 
    \end{align*} 
The bound on $\sup_{t\in [k\eta,(k+1)\eta]}  \E{\abs{\nabla U(\tilde{\theta}(t))-\bar{g}(t)}^2}$ is obtained in the same way. 
\end{proof}

\begin{proof}[Proof of Lemma~\ref{lemma_boundupdatedifference_4thorder}]

\textbf{First part of the proof:} we will bound the difference of the components of $\tilde{x}(t)-\hat{x}(t)$ in $L^2$-norm for $t\in (k\eta,(P=1)\eta]$. 

We start with $ \tilde{v}_1(t)-\hat{v}_1(t)=\int_{k\eta}^t \left(-\tilde{g}(s)+\gamma \hat{v}_2(s)\right)ds$, which combined with the moment bounds in Lemma~\ref{lemma_momentbound} leads to \begin{align*}
\E{\abs{\tilde{v}_1(t)-\hat{v}_1(t) }^2}\leq C_2(d+1)(\gamma+1)^2(t-k\eta)^2.
\end{align*}
Next, $\tilde{\theta}(t)-\hat{\theta}(t)=\int_{k\eta}^t \left(\tilde{v}_1(s)-\hat{v}_1(s)\right)ds$ combined with Lemma~\ref{lemma_momentbound} leads to \begin{align*}
\E{\abs{\tilde{\theta}(t)-\hat{\theta}(t)}^2}\leq C_2(d+1)(\gamma+1)^2(t-k\eta)^4.
\end{align*}

Moreover,
\begin{align*}
    \tilde{v}_2(t)-\hat{v}_2(t)=\int_{k\eta}^t \left(-\gamma\brac{\tilde{v}_1(s)-\hat{v}_1(s) }+\gamma\brac{\hat{v}_3(s)-v_3^{(k)} }\right)ds.
\end{align*}
We know $\hat{v}_3(t)-v_3^{(k)}=\int_{k\eta}^t \left(-\gamma \hat{v}_3(s)-\gamma \hat{v}_2(s)\right)ds+\sqrt{2\gamma}\brac{B_t-B_{k\eta}}$ which along with Lemma~\ref{lemma_momentbound} imply the bound 
\begin{align*}
\E{\abs{\hat{v}_3(t)-v_3^{(k)}}^2 }\leq C_2(d+1)(2\gamma+\sqrt{2\gamma})^2\brac{t-k\eta}. 
\end{align*}
Consequently,
\begin{align}
\label{bound_differencev2}
    \E{\abs{\tilde{v}_2(t)-\hat{v}_2(t)}^2}&\leq C_2(d+1)(\gamma+1)^2(t-k\eta)^4+C_2(d+1)(2\gamma+\sqrt{2\gamma})^2\brac{t-k\eta}^3\nonumber\\
    &\leq C_2(d+1)\brac{(\gamma+1)^2+ (2\gamma+\sqrt{2\gamma})^2}\brac{t-k\eta}^3. 
\end{align}

Finally, $ \tilde{v}_3(t)-\hat{v}_3(t)=\int_{k\eta}^t -\gamma e^{-\gamma s}\brac{ \tilde{v}_2(s)-\hat{v}_2(s)}ds$ and the moment bound in Lemma~\ref{lemma_momentbound} imply \begin{align*}
\E{\abs{ \tilde{v}_3(t)-\hat{v}_3(t)}^2}\leq C_2(d+1)\brac{(\gamma+1)^2+ (2\gamma+\sqrt{2\gamma})^2}\brac{t-k\eta}^5.
\end{align*}

\textbf{Second part of the proof:} We will bound the difference of the components of $\bar{x}(t)-\tilde{x}(t)$ in $L^2$ norm for $t\in (k\eta,(k+1)\eta]$. 

We start with
\begin{align}
\label{bound_differencev1}
    &\bar{v}_1(t)-\tilde{v}_1(t)=\int_{k\eta}^t -\brac{\bar{g}(s)-\tilde{g}(s)}ds+\int_{k\eta}^t \gamma\brac{\tilde{v}_2(s)-\hat{v}_2(s)} ds\nonumber\\
    &\qquad=\int_{k\eta}^t \bigg(-\brac{\bar{g}(s)-\nabla U(\tilde{\theta}(s))}- \brac{\nabla U(\tilde{\theta}(s))-\nabla U(\hat{\theta}(s))}\nonumber\\
    &\qquad\qquad\qquad\qquad-\brac{\tilde{g}(s)-\nabla U(\hat{\theta}(s))}\bigg)ds+\int_{k\eta}^t \gamma\brac{\tilde{v}_2(s)-\hat{v}_2(s)} ds.
\end{align}
The bound in \eqref{bound_differencev2} leads to
\begin{align*}
    \E{\abs{\int_{k\eta}^t \gamma\brac{\tilde{v}_2(s)-\hat{v}_2(s)} ds}^2}\leq C_2(d+1)\gamma^2\brac{(\gamma+1)^2+ (2\gamma+\sqrt{2\gamma})^2}\brac{t-k\eta}^5. 
\end{align*}
Meanwhile, Lemma~\ref{lemma_polyapprox}, Lemma~\ref{lemma_momentbound} and Condition~\ref{cond_derivativegrowthrate} say
\begin{align}
\label{bound_polyapproxtildetheta}
  \E{\abs{\int_{k\eta}^t \left(\bar{g}(s)-\nabla U(\tilde{\theta}(s))\right)ds}^2}\leq    \brac{\frac{L_\alpha}{\alpha!}}^2(C_1)^\alpha(d+2\alpha)^\alpha \eta^2\leq cd\eta^9. 
\end{align}

Similarly, we have $ \E{\abs{\int_{k\eta}^t \left(\tilde{g}(s)-\nabla U(\hat{\theta}(s))\right)ds}^2}\leq cd\eta^9$. Also based on the first part of the proof and $L$-smoothness of $U$ in Condition~\ref{cond_mainpaper},  
\begin{align*}
\E{\abs{\int_{k\eta}^t \left(\nabla U(\tilde{\theta}(s))-\nabla U(\bar{\theta}(s))\right)ds}^2}\leq C_2(d+1)L^2(\gamma+1)^2(t-k\eta)^6. 
\end{align*}
By combining the previous calculations, we get for any $t\in (k\eta,(k+1)\eta]$ that
\begin{align*}
    \E{\abs{\bar{v}_1(t)-\tilde{v}_1(t)}^2}\leq C_2(d+1)\brac{\gamma^2\brac{(\gamma+1)^2+ (2\gamma+\sqrt{2\gamma})^2}+c }\eta^5. 
\end{align*}

Next,  $\tilde{\theta}(t)-\bar{\theta}(t)=\int_{k\eta}^t \left(\tilde{v}_1(s)-\bar{v}_1(s)\right)ds$ leads to 
\begin{align*}
    \E{\abs{\tilde{\theta}(t)-\bar{\theta}(t)}^2}\leq C_2(d+1)\brac{\gamma^2\brac{(\gamma+1)^2+ (2\gamma+\sqrt{2\gamma})^2}+c }\eta^7. 
\end{align*} 

Moreover, $\tilde{v}_2(t)-\bar{v}_2(t)=\int_{k\eta}^t\left( -\gamma\brac{\tilde{v}_1(s)-\bar{v}_1(s) }+\gamma\brac{\hat{v}_3(s)-\tilde{v}_3(s)} \right)ds$ leads to 
\begin{align*}
  \E{\abs{\bar{v}_2(t)-\tilde{v}_2(t)}^2}&\leq C_2(d+1)\gamma^2\brac{\gamma^2\brac{(\gamma+1)^2+ (2\gamma+\sqrt{2\gamma})^2}+c }\eta^7\\
  &\qquad\qquad+C_2(d+1)\gamma^2\brac{(\gamma+1)^2+ (2\gamma+\sqrt{2\gamma})^2}\eta^7. 
\end{align*}

Finally, $ \tilde{v}_3(t)-\bar{v}_3(t)=\int_{k\eta}^t -\gamma e^{-\gamma s}\brac{ \tilde{v}_2(s)-\bar{v}_2(s)}ds$ implies
\begin{align*}
     \E{\abs{\bar{v}_3(t)-\tilde{v}_3(t)}^2}&\leq C_2(d+1)\gamma^4\brac{\gamma^2\brac{(\gamma+1)^2+ (2\gamma+\sqrt{2\gamma})^2}+c }\eta^9\\
  &\qquad\qquad+C_2(d+1)\gamma^4\brac{(\gamma+1)^2+ (2\gamma+\sqrt{2\gamma})^2}\eta^9. 
\end{align*}
This completes the proof.
\end{proof}
\section{Details of $P$-th Order Langevin Monte Carlo Algorithm}
\label{appendix_Pthorder}

The following is a generalized version of Lemma~\ref{lemma_explicitformofxbar}. 
\begin{lemma}
    \label{lemma_explicitform_allstages}

   Choose any positive integers $i$ and $j$ in $[1,P-1]$. Then 
   \begin{enumerate}[label=\arabic*.]
\item the auxiliary process $\{v^{\operatorname{st}_{ j}}_i(t),t\geq 0\}$ in Section~\ref{section_Pthorder} has the form
\begin{align}
\label{generalform_vi}
    &v^{\operatorname{st}_{ j}}_i(t)=\sum_{1\leq \ell\leq P-1}  v_\ell^{(k)}\mu^{\operatorname{st}_{ j}}_{i,\ell}(t)+\theta^{(k)}\mu^{\operatorname{st}_{ j}}_{i,P}(t)+\mathds{1}_{ \{i\geq P-j\}}\int_{k\eta}^t h^{\operatorname{st}_{ j}}_i(s,t) dB_s\nonumber\\
    &\hspace{20em}+\sum_{2\leq \ell\leq j}\int_{k\eta}^t \kappa^{\operatorname{st}_{ j}}_{i,\ell}(s,t)g^{\operatorname{st}_{ \ell}}(s)ds,
\end{align}
such that 
\begin{enumerate}[label=\alph*.]
    \item The kernel $h^{\operatorname{st}_{ j}}_i(s,t)$ is deterministic and has the form
    \begin{align}
    \label{generalformkernel}
     \qquad\qquad  \qquad     h^{\operatorname{st}_{ j}}_i(s,t)= \sum_{0\leq m\leq M_1}a_{1,m}e^{b_{1,m}(t-s)+c_{1,m}}(s-k\eta)^{d_{1,m}},
    \end{align}
    where $M_1$ is a positive integer; $d_{1,m}$'s are non-negative integers; $a_{1,m},b_{1,m},c_{1,m}$'s are rational functions in variables $k,\eta,\gamma,t$, i.e. they are ratios of multivariate polynomials in $k,\eta,\gamma,t$. \footnote{The coefficients $a_{1,m},b_{1,m},c_{1,m}$ and $d_{1,m}$ on the right hand side of \eqref{generalformkernel} depend on $i,j$; however we hide this dependence to lighten the notations. We do the same thing in Equation~\eqref{generalformmu} and Equation~\eqref{generalforkappa}.}
    \item $\mu^{\operatorname{st}_{ j}}_{i,\ell}(t)$ has the form 
    \begin{align}
    \label{generalformmu}
     \qquad\qquad  \qquad    \mu^{\operatorname{st}_{ j}}_{i,\ell}(t)= \sum_{0\leq m\leq M_2} a_{2,m}e^{b_{2,m}(t-k\eta)+c_{2,m}}(t-k\eta)^{d_{2,m}},
    \end{align}
    where $M_2$ is a positive integer; $a_{2,m},b_{2,m},c_{2,m}$'s are rational functions in variables $k,\eta,\gamma$; and $d_{2,m}$'s are non-negative integers. 
    \item  $\kappa^{\operatorname{st}_{ j}}_{i,\ell}(s,t)$ is deterministic and has the form
    \begin{align}
    \label{generalforkappa}
     \qquad\qquad  \qquad     \kappa^{\operatorname{st}_{ j}}_{i,\ell}(s,t)= \sum_{0\leq m\leq M_3}a_{3,m}e^{b_{3,m}(t-s)+c_{3,m}}d_{3,m}(s),
    \end{align}
    where $M_3$ is a positive integer; $d_{3,m}$'s are polynomial in $s$; $a_{3,m},b_{3,m},c_{3,m}$'s are rational functions in variables $k,\eta,\gamma,t$, i.e. they are ratios of multivariate polynomials in $k,\eta,\gamma,t$. 
\end{enumerate}
\item the auxiliary process $\{\theta^{\operatorname{st}_{ j}}(t),t\geq 0\}$ in Section~\ref{section_Pthorder} has the form
\begin{align}
\label{generalform_theta}
   & \theta^{\operatorname{st}_{ j}}(t)= \sum_{1\leq \ell\leq P-1}  v_\ell^{(k)}\mu^{\operatorname{st}_{ j}}_{P,\ell}(t)+\theta^{(k)}\mu^{\operatorname{st}_{ j}}_{P,P}(t)+\mathds{1}_{ \{j=P-1\}}\int_{k\eta}^t h^{\operatorname{st}_{ j}}_P(s,t) dB_s\nonumber\\
     &\hspace{20em}+\sum_{2\leq \ell\leq j}\int_{k\eta}^t \kappa^{\operatorname{st}_{ j}}_{P,\ell}(s,t)g^{\operatorname{st}_{ \ell}}(s)ds,
\end{align}
such that $\mu_{P,\ell}(t)$ has a similar form to~\eqref{generalformmu}, while  $h^{\operatorname{st}_{ j}}_P(s,t)$ and $\kappa^{\operatorname{st}_{ j}}_{P,\ell}(s,t)$ have similar forms to respectively~\eqref{generalformkernel} and~\eqref{generalforkappa}. 
\end{enumerate}
\end{lemma}

\begin{proof}

We will employ an induction argument.

\textbf{Step 1:} Stage $j=1$. 

We will verify that $ v^{\operatorname{st}_{ 1}}_{n}(t),1\leq n\leq P-1$ and $ \theta^{\operatorname{st}_{ 1}}(t)$ have respectively the general forms~\eqref{generalform_vi} and~\eqref{generalform_theta}.

We have $v^{\operatorname{st}_{ 1}}_1(t) =v_1^{(k)}$ so that
\begin{align*}
    \theta^{\operatorname{st}_{ 1}}_1(t) =v^{(k)}_1(t-k\eta),
\end{align*}
and 
\begin{align}
\label{v2stage1}
   v^{\operatorname{st}_{ 1}}_2(t) &=v_2^{(k)}-\gamma\int_{k\eta}^t v^{\operatorname{st}_{ 1}}_1(s)ds+\gamma v^{(k)}_3(t-k\eta)=v_2^{(k)}-\gamma (t-k\eta)v_1^{(k)}+\gamma v^{(k)}_3(t-k\eta). 
\end{align}
Proceed similarly for increasing $n$ to get for $3\leq n\leq P-2$, 
\begin{align}
\label{v_nformula_stage1}
    v^{\operatorname{st}_{ 1}}_n(t) &=v_n^{(k)}-\gamma\int_{k\eta}^t v^{\operatorname{st}_{ 1}}_{n-1}(s)ds+\gamma v^{(k)}_{n-1}(t-k\eta)=\sum_{\ell=1}^{P-1} a_\ell (t-k\eta)^{d_\ell} v_\ell^{(k)},
\end{align}
where $a_\ell$'s are polynomials in $\gamma$ and $d_\ell$'s are non-negative integers.  If we set $\mu^{\operatorname{st}_{ 1}}_{n,\ell}(t):=a_\ell (t-k\eta)^{d_\ell}$ then this coefficient is of the form described in the statement of the lemma. 

Moreover, the formula \eqref{v_nformula_stage1} in the case $n=P-2$ implies
\begin{align*}
  &  v^{\operatorname{st}_{ 1}}_{P-1}(t) =e^{-\gamma(t-k\eta)}v^{(k)}_{P-1}-\gamma\int_{k\eta}^t e^{-\gamma(t-s)} v^{\operatorname{st}_{ 1}}_{P-2}(s)ds+\sqrt{2\gamma}\int_{k\eta}^t e^{-\gamma(t-s)}dB_s\\
 &=e^{-\gamma(t-k\eta)}v^{(k)}_{P-1} +\sqrt{2\gamma}\int_{k\eta}^t e^{-\gamma(t-s)}dB_s- \sum_{\ell=1}^{P-1} v_\ell^{(k)}\gamma a_{\ell}\int_{k\eta}^t e^{-\gamma(t-s)}  (s-k\eta)^{d_{\ell}}ds . 
\end{align*}
Via integration by parts, it is easy to see 
\begin{align}
\label{int_productexpandpoly}
    \int_{k\eta}^t e^{-\gamma(t-s)}  (s-k\eta)^{d_{\ell}}ds=\sum_j e^{-\gamma(t-k\eta)}a_{\ell,j}(t-k\eta)^{d_{\ell,j}},
\end{align}
where $d_{\ell,j}$'s are non-negative integers and $a_{\ell,j}$'s are polynomials in $\gamma$. Setting 
\begin{align*}
\mu^{\operatorname{st}_{ 1}}_{P-1,\ell}(t):=\gamma a_{\ell}\sum_j e^{-\gamma(t-k\eta)}a_{\ell,j}(t-k\eta)^{d_{\ell,j}},
\end{align*}
for $1\leq \ell\leq P-2$ and 
\begin{align*}
\mu^{\operatorname{st}_{ 1}}_{P-1,P-1}(t):=\gamma a_{P-1}\sum_j e^{-\gamma(t-k\eta)}a_{P-1,j}(t-k\eta)^{d_{P-1,j}}+e^{-\gamma(t-k\eta)}, 
\end{align*}
we arrive at 
\begin{align}
\label{vP-1stage1}
    v^{\operatorname{st}_{ 1}}_{P-1}(t) =\sqrt{2\gamma}\int_{k\eta}^t e^{-\gamma(t-s)}dB_s+\sum_{\ell=1}^{P-1} v_\ell^{(k)}\mu^{\operatorname{st}_{ 1}}_{P-1,\ell}(t). 
\end{align}

Finally, notice that among $\theta^{\operatorname{st}_{ 1}}_1(t)$ and $ v^{\operatorname{st}_{ 1}}_n(t), 1\leq n\leq P-1$, the It\^{o} integral only appears in $v^{\operatorname{st}_{ 1}}_{P-1}(t)$, which explains the indicator functions in \eqref{generalform_vi} and \eqref{generalform_theta} when $j=1$.

\textbf{Step 2:} Stage $j=2$. 

We will verify that $ v^{\operatorname{st}_{ 2}}_{n}(t),1\leq n\leq P-1$ and $ \theta^{\operatorname{st}_{ 2}}(t)$ have respectively the general forms~\eqref{generalform_vi} and~\eqref{generalform_theta}. 

We have based on~\eqref{v2stage1} that 
\begin{align}
\label{v1st2}
      v^{\operatorname{st}_{ 2}}_{1}(t) &=v^{(k)}_1-\int_{k\eta}^t g^{\operatorname{st}_{ 2}}(s) ds+\gamma\int_{k\eta}^t v^{\operatorname{st}_{ 1}}_{2}(s)ds\nonumber\\
      &=v^{(k)}_1-\int_{k\eta}^t g^{\operatorname{st}_{ 2}}(s) ds-\gamma^2 v^{(k)}_1\frac{(t-k\eta)^2}{2}+\gamma^2 v^{(k)}_3\frac{(t-k\eta)^2}{2}.
\end{align}
This implies 
\begin{align*}
     \theta^{\operatorname{st}_{ 2}}(t)&=\theta^{(k)}-\int_{k\eta}^tv_1^{\operatorname{st}_{ 2}}(s)ds
    \\&=\theta^{(k)}-\int_{k\eta}^t\int_{k\eta}^{s_2}g^{\operatorname{st}_{ 2}}(s_1)ds_1ds_2-\gamma^3v^{(k)}_1\frac{(t-k\eta)^3}{6}+\gamma^2v^{(k)}_3\frac{(t-k\eta)^3}{6},
\end{align*}
noting that $\int_{k\eta}^t\int_{k\eta}^{s_2}g^{\operatorname{st}_{ 2}}(s_1)ds_1ds_2=\int_{k\eta}^t (t-s_1)g^{\operatorname{st}_{ 2}}(s_1)ds_1$. 
Per the previous calculations \eqref{v1st2} and \eqref{v_nformula_stage1} in the case $n=3$, we can further write
\begin{align*}
   & v^{\operatorname{st}_{ 2}}_{2}(t)=v_2^{(k)}-\gamma\int_{k\eta}^t v^{\operatorname{st}_{ 2}}_{1}(s)ds+\gamma\int_{k\eta}^t v^{\operatorname{st}_{ 1}}_{3}(s)ds\\
    &=v^{(k)}_2+\bigg(-\gamma\theta^{(k)}+\gamma\int_{k\eta}^t (t-s_1)g^{\operatorname{st}_{ 2}}(s_1)ds_1+\gamma^4v_1^{(k)}\frac{(t-k\eta)^4}{4!}-\gamma^3v_3^{(k)}\frac{(t-k\eta)^4}{4!} \bigg)\\
    &\qquad\qquad\qquad+\bigg(\gamma \sum_{i=1}^{P-1}\frac{(t-k\eta)^{a_{3,i}+1}}{a_{3,i}+1}v_i^{(k)}b_{3,i}(\gamma) \bigg),
\end{align*}
which is of the form described in the statement of the lemma. Proceed similarly for increasing $n,3\leq n\leq P-3$ to get
\begin{align}
\label{middlenstage2}
&v^{\operatorname{st}_{ 2}}_{n}(t)=v_n^{(k)}-\gamma\int_{k\eta}^t v_{n-1}^{\operatorname{st}_{ 2}}(s)ds+\gamma\int_{k\eta}^t v^{\operatorname{st}_{ 1}}_{n+1}(s)ds\nonumber\\
&\quad= \sum_{1\leq \ell\leq P-1}  v_\ell^{(k)}\mu^{\operatorname{st}_{ 2}}_{i,\ell}(t)+\theta^{(k)}\mu^{\operatorname{st}_{ 2}}_{i,P}(t)+e_n\int_{k\eta}^t\int_{k\eta}^{s_n}\cdots\int_{k\eta}^{s_2}g^{\operatorname{st}_{ 2}}(s_1) ds_1\ldots ds_{n-1}ds_n,
\end{align}
where $e_n$ are rational functions in variables $k,\eta,\gamma$. Moreover, the last term can be simplified as  
\begin{align*}
&e_n\int_{k\eta}^t\int_{k\eta}^{s_n}\cdots\int_{k\eta}^{s_2}g^{\operatorname{st}_{ 2}}(s_1) ds_1\ldots ds_{n-1}ds_n\\
    &=e_n\int_{k\eta}^t \int_{s_1}^t\int_{s_1}^{s_n}\ldots\int_{s_1}^{s_3}g^{\operatorname{st}_{ 2}}(s_1) ds_2\ldots ds_nds_1 = \int_{k\eta}^t p(s_1)g^{\operatorname{st}_{ 2}}(s_1)ds_1,
\end{align*}
where $p$ is a polynomial in $s_1$. 

Next, we have
\begin{align*}
   v^{\operatorname{st}_{ 2}}_{P-2}(t)=v^{(k)}_{P-2}-\gamma\int_{k\eta}^t  v^{\operatorname{st}_{ 2}}_{P-3}(s)ds+\gamma\int_{k\eta}^t v^{\operatorname{st}_{ 1}}_{P-1}(s)ds .  
\end{align*}
We will only expand the term $\int_{k\eta}^t  v^{\operatorname{st}_{ 2}}_{P-3}(s)ds$ using \eqref{middlenstage2} when $n=P-3$. The term $\int_{k\eta}^t v^{\operatorname{st}_{ 1}}_{P-1}(s)ds$ can be handled in similar fashion using \eqref{vP-1stage1}: 
\begin{align*}
   & \int_{k\eta}^t  v^{\operatorname{st}_{ 2}}_{P-3}(s)ds=v_{P-3}^{(k)}(t-k\eta) + \sum_{1\leq \ell\leq P-1}  v_\ell^{(k)}\int_{k\eta}^t\mu^{\operatorname{st}_{ 2}}_{i,\ell}(s)ds+\theta^{(k)}\int_{k\eta}^t\mu^{\operatorname{st}_{ 2}}_{i,P}(s)ds\\
   &\hspace{20em}+e_n\int_{s_2=k\eta}^t \int_{s_1=k\eta}^{s_2} p(s_1)g^{\operatorname{st}_{ 2}}(s_1)ds_1ds_2, 
\end{align*}
where we can further compute that 
\begin{align*}
\int_{s_2=k\eta}^t \int_{s_1=k\eta}^{s_2} p(s_1)g^{\operatorname{st}_{ 2}}(s_1)ds_1ds_2
&=e_n\int{k\eta}^t\int_{s_1}^tp(s_1)g^{\operatorname{st}_{ 2}}(s_1)ds_2ds_1
\\
&=e_n\int_{k\eta}^t (t-s_1)p(s_1)g^{\operatorname{st}_{ 2}}(s_1)ds_1. 
\end{align*}
Hence, we arrive at 
\begin{align*}
  &  v^{\operatorname{st}_{ 2}}_{P-2}(t)\\
  &\quad=\sum_{1\leq \ell\leq P-1}  v_\ell^{(k)}\mu^{\operatorname{st}_{ 2}}_{P-2,\ell}(t)+\theta^{(k)}\mu^{\operatorname{st}_{ 2}}_{P-2,P}(t)+\int_{k\eta}^t h^{\operatorname{st}_{ 2}}_{P-2}(s,t) dB_s+\int_{k\eta}^t \kappa^{\operatorname{st}_{ 2}}_{P-2}(s,t)g^{\operatorname{st}_{ 2}}(s)ds
\end{align*}
as described in \eqref{generalform_vi}.

Finally, we have 
\begin{align*}
    v^{\operatorname{st}_{ 2}}_{P-1}(t) =e^{-\gamma(t-k\eta)}v^{(k)}_{P-1}-\gamma\int_{k\eta}^t e^{-\gamma(t-s)} v^{\operatorname{st}_{ 2}}_{P-2}(s)ds+\sqrt{2\gamma}\int_{k\eta}^t e^{-\gamma(t-s)}dB_s. 
\end{align*}
The second term on the right hand side can be expanded by plugging in the formula for $v^{\operatorname{st}_{ 2}}_{P-2}(s)$, then applying \eqref{int_productexpandpoly} and the fact that $\int_{k\eta}^t \int_{k\eta}^{s_3}\int_{k\eta}^{s_2}e^{-\gamma(t-s_3)}e^{-\gamma(s_2-s_1)}dB_{s_1}ds_2ds_3=\int_{k\eta}^t\int_{s_1}^t\int_{s_1}^{s_3}e^{-\gamma(t-s_3)}e^{-\gamma(s_2-s_1)}ds_2ds_3dB_{s_1}=\int_{k\eta}^t\brac{\frac{1}{\gamma^2}-\frac{e^{-\gamma(t-s_1)}}{\gamma^2}+\frac{s_1e^{-\gamma(t-s_1)}}{\gamma}-\frac{te^{-\gamma(t-s_1)}}{\gamma}}dB_{s_1}$, 
and also that  $\int_{k\eta}^{t} \int_{k\eta}^{s_2}e^{-\gamma(t-s_2)}p(s_1)g^{\operatorname{st}_{ 2}}(s_1)ds_1ds_2=\int_{k\eta}^t \int_{s_1}^t e^{-\gamma(t-s_2)}p(s_1)g^{\operatorname{st}_{ 2}}(s_1)ds_2ds_1=\int_{k\eta}^t (1/\gamma) e^{-\gamma(t-s_1)}p(s_1)g^{\operatorname{st}_{ 2}}(s_1) ds_1$. Consequently, $v^{\operatorname{st}_{ 2}}_{P-1}(t)$ can be written as \eqref{generalform_vi}.

Finally, we note that among $\theta^{\operatorname{st}_{ 2}}_1(t)$ and $ v^{\operatorname{st}_{ 2}}_n(t), 1\leq n\leq P-1$, It\^{o} integrals only appear in the formulas of  $ v^{\operatorname{st}_{ 2}}_{P-1}(t)$ and $ v^{\operatorname{st}_{ 2}}_{P-2}(t)$, which explains  the indicator functions in \eqref{generalform_vi} and \eqref{generalform_theta} when $j=2$.  This completes the proof for Stage $j=2$.

\textbf{Step 3:} Induction argument. 

As the induction hypothesis, we assume the statement of the lemma holds for Stage $j$ and verify Stage $j+1$. The proof is similar to \textbf{Step 2} above (proceeding from Stage $1$ to Stage $2$) and is therefore omitted. 
\end{proof}

\begin{lemma}
    \label{lemma_meanandcovariance_Pthorder}
 $\E{x^{(k+1)}|x^{(k)}}$ follows a multivariate normal distribution in $\R^{Pd}$ whose mean vector and covariance matrix can be determined from Lemma~\ref{lemma_explicitform_allstages}. 
\end{lemma}
\begin{proof}
   Lemma~\ref{lemma_explicitform_allstages} provides us with the formulas for the components $x^{(k+1)}=x^{\operatorname{st}_{ P-1}}((k+1)\eta)$. Based on those formulas, we can see that $\E{x^{(k+1)}|x^{(k)}}$ follows a multivariate normal distribution in $\R^{Pd}$. From there, calculating the mean and covariance is straightforward and is the same as the proof of Lemma~\ref{lemma_meanandcovariance}.  
\end{proof}


The following is a general version of Lemma~\ref{lemma_decomposexk+1}.

\begin{lemma}
\label{lemma_decomposexk+1_Pthorder}
Recall the unique minimizer $\theta^*$ of $U$ and the $Pd\times Pd$ Jacobian matrix 
 \begin{align*}
    J_b(\theta^*, 0,\ldots,0)=\begin{pmatrix}
        0_d & I_d & 0_d &\cdots &\cdots &\cdots &\cdots &\cdots&0_d\\
        -\nabla U^2(\theta^*)I_d &0_d &\gamma I_d & 0_d &\cdots &\cdots&\cdots&\cdots& 0_d\\
        0_d & -\gamma I_d &0_d &\gamma I_d &0_d &\cdots &\cdots&\cdots &0_d\\
         0_d & 0_d& -\gamma I_d &0_d &\gamma I_d &0_d &\cdots &\cdots& 0_d\\
          0_d & 0_d& 0_d& -\gamma I_d &0_d &\gamma I_d &0_d &\cdots & 0_d\\
         \vdots &\ddots &\ddots&\ddots&\ddots&\ddots&\ddots&\ddots&\vdots\\
         0_d&\cdots &\cdots&\cdots&\cdots&\cdots&-\gamma I_d& 0_d &\gamma I_d\\
          0_d&\cdots &\cdots&\cdots&\cdots&\cdots&0_d&-\gamma I_d &-\gamma I_d
    \end{pmatrix}.
    \end{align*}
    Then it holds for Stage $j,1\leq j\leq P-1$ that 
\begin{align}
\label{decompose_allstages}
    &x^{\operatorname{st}_{ j}}(t) -\brac{\theta^*,0,\ldots,0}=\brac{x^{(k)}-\brac{\theta^*,0,\ldots,0}}+(t-k\eta) J_b(\theta^*, 0,\ldots,0)\nonumber\\
    &\qquad\qquad\cdot\brac{x^{(k)}-\brac{\theta^*,0,\ldots,0}}+R(t)\brac{x^{(k)}-\brac{\theta^*,0,\ldots,0}}+F_k(t),
\end{align}
where $R(t)$ is a $Pd\times Pd$ matrix with  $\abs{R_{ij}}(t)\leq C(t-k\eta)^2$, $1\leq i,j\leq 4d$ and $C$ is a constant that depends only on $\gamma,P$. Moreover, $F_k(t)$ is the $Pd$-dimensional vector $\brac{f_P(t) \quad f_1(t)\quad f_2(t)\quad f_3(t)\quad \cdots \quad f_{P-1}(t)}^{\top}$ where for each $i$, $f_i(t):=\int_{k\eta}^t h^{\operatorname{st}_{j}}_i(s,t) dB_s$ is the $d$-dimensional It\^{o} integral defined in Lemma~\ref{lemma_explicitform_allstages}.

    Consequently, we have 
\begin{align}
\label{decomposexk+1_Pthorder}
    &x^{(k+1)}-\brac{\theta^*,0,\ldots,0}=\brac{x^{(k)}-\brac{\theta^*,0,\ldots,0}}+\eta J_b(\theta^*, 0,\ldots,0)\brac{x^{(k)}-\brac{\theta^*,0,\ldots,0}}\nonumber\\
    &\qquad\qquad\qquad\qquad\qquad\qquad+R((k+1)\eta)\brac{x^{(k)}-\brac{\theta^*,0,\ldots,0}}+F_k((k+1)\eta). 
\end{align}
\end{lemma}

\begin{proof}
Since $x^{(k+1)}=x^{\operatorname{st}_{ P-1}}((k+1)\eta) $, it is sufficient to prove \eqref{decompose_allstages}. Without loss of generality, we assume the unique minimizer of $U$ is $\theta^*=0$. The proof follows an induction argument.

\textbf{Step 1: The Base Case $j=2$.}

Per the proof of Lemma~\ref{lemma_explicitform_allstages}, we can deduce that 
\begin{align*}
   & v_1^{\operatorname{st}_{ 2}}(t)=v_1^{(k)}-\int_{k\eta}^t g^{\operatorname{st}_{ 2}}(s)ds+\gamma\int_{k\eta}^t v^{\operatorname{st}_{ 1}}_2(s)ds\\
    &=v_1^{(k)}-(t-k\eta)\nabla^2U(0)\theta^{(k)}+(t-k\eta)v_2^{(k)}+O\brac{(t-k\eta)^2}\sum_{i\neq 1,2}v_i^{(k)}+\int_{k\eta}^t h^{\operatorname{st}_{2}}_1(s,t) dB_s.
\end{align*}
Next, we have 
\begin{align*}
    &\theta^{\operatorname{st}_{ 2}}(t)=\theta^{(k)}+\int_{k\eta}^t v^{\operatorname{st}_{ 1}}_1(s)ds\\
    &=\theta^{(k)}+(t-k\eta)v_1^{(k)}+O\brac{(t-k\eta)^2}\sum_{i\neq 1}v_i^{(k)}+\int_{k\eta}^t h^{\operatorname{st}_{2}}_P(s,t) dB_s.
\end{align*}
Similarly, 
\begin{align*}
     &v_n^{\operatorname{st}_{ 2}}(t)=v_n^{(k)}-\gamma \int_{k\eta}^t v_{n-1}^{\operatorname{st}_{ 2}}(s)ds+\gamma \int_{k\eta}^t v_{n+1}^{\operatorname{st}_{ 1}}(s)ds\\
     &\qquad=v_n^{(k)}-(t-k\eta)\gamma v_{n-1}^{(k)}+(t-k\eta)\gamma v_{n+1}^{(k)}\\
     &\qquad\qquad+O\brac{(t-k\eta)^2}\sum_{i\notin \{n-1,n,n+1\}}v_i^{(k)}+\int_{k\eta}^t h^{\operatorname{st}_{2}}_n(s,t) dB_s, \quad 2\leq n \leq P-2;\\
    &v_{P-1}^{\operatorname{st}_{ 2}}(t)=v_{P-1}^{(k)}-(t-k\eta)\gamma v_{P-2}^{(k)}-(t-k\eta)\gamma v_{P-1}^{(k)}\\
    &\qquad\qquad+O\brac{(t-k\eta)^2}\sum_{i\notin \{P-2,P-1\}}v_i^{(k)}+\int_{k\eta}^t h^{\operatorname{st}_{2}}_{P-1}(s,t) dB_s.
\end{align*}
At this point, we can conclude \eqref{decompose_allstages} holds for Stage $j=2$.

\textbf{Step 2: The Induction Argument.}

We assume \eqref{decompose_allstages} is true up to Stage $j$ and verify Stage $j+1$. The argument is similar to the one in \textbf{Step 1} and is therefore omitted. 
The proof is complete.
\end{proof}

The upcoming moment bounds are similar to the ones in Lemma~\ref{lemma_momentbound}. 
\begin{lemma}
    \label{lemma_momentbound_Pthorder}
Assume
    \begin{align}
    \label{def_eta**}
        \eta\leq \eta^{**}:=\frac{\rho}{2}\opnorm{M}^{-1}\opnorm{M^{-1}}^{-1}\frac{1}{1+(1+(P-2)2)\gamma^2+L^2},
    \end{align}
 where the matrix $M$ and $\gamma,\rho,L$ are from Theorem~\ref{theorem_frommonmarche_mainpaper}. Then there exists a positive constant $\widetilde{C}_1$ such that for all $k$, 
    \begin{align*}
           \E{\abs{x^{(k+1)}}^{2\alpha}}\leq \brac{\widetilde{C}_1}^\alpha (d+2\alpha)^\alpha, 
    \end{align*}
    where $\widetilde{C}_1$ depends on $P$ and $\gamma$, $L$ from Condition~\ref{cond_mainpaper} and $c$ from Condition~\ref{cond_derivativegrowthrate}, and not on $d$. 
This further implies
    \begin{align*}
    \sup_{1\leq j,n\leq P-1}   \sup_{s\in [k\eta,(k+1)\eta]} \E{\abs{v^{\operatorname{st}_j}_n(s)}^2+\abs{\theta^{\operatorname{st}_j}_n(s)}^2}+ \sup_{2\leq j\leq P-1}   &\sup_{s\in [k\eta,(k+1)\eta]} \E{\abs{g^{\operatorname{st}_j}(s)}^2}\\
    &\leq \widetilde{C}_2\brac{d+1},
    \end{align*}
    for a universal constant $\widetilde{C}_2>1$ that depends only on $P,\gamma$ and $L$. 
\end{lemma}

\begin{proof}
Similarly as how we apply Lemma~\ref{lemma_decomposexk+1} to get the moment bounds in Lemma~\ref{lemma_momentbound} in Lemma~\ref{lemma_momentbound} for fourth-order LMC algorithm, we can apply Lemma~\ref{lemma_decomposexk+1_Pthorder} to derive the moment bounds for $P$-th order LMC algorithm. The proof is very similar to the proof of Lemma~\ref{lemma_momentbound} and is therefore omitted. 
\end{proof}

\begin{proof}[Proof of Lemma~\ref{lemma_basecases}]

\textbf{Step 1: $j=1$}

We start with $v^{\operatorname{st}_{ 1}}_n(t)-v^{(k)}_1=0, t\in (k\eta,(k+1)\eta]$. Next,  
\begin{align*}
    v^{\operatorname{st}_{ 1}}_2(t)-v^{(k)}_2=\int_{k\eta}^t\left( -\gamma v^{\operatorname{st}_{ 1}}_1(s)+\gamma v^{(k)}_3\right)ds,
\end{align*}
which combined with the second moment bounds in Lemma~\ref{lemma_momentbound_Pthorder} lead to 
\begin{align*}
 \sup_{t\in(k\eta,(k+1)\eta]}   \E{\abs{v^{\operatorname{st}_{ 1}}_2(t)-v^{(k)}_2}^2}\leq \eta^2\gamma^2\brac{\E{\abs{v^{(k)}_1}^2}+\E{\abs{v^{(k)}_3}^2} } \leq C^{\operatorname{st}_{ 1}}_2d\eta^2. 
\end{align*}
Proceed similarly for increasing $n$ with $3\leq n\leq P-2$ to obtain 
\begin{align*}
    \sup_{t\in(k\eta,(k+1)\eta]}   \E{\abs{v^{\operatorname{st}_{ 1}}_n(t)-v^{(k)}_n}^2}&\leq \eta^2\gamma^2 \brac{ \sup_{t\in(k\eta,(k+1)\eta]}\E{\abs{v^{\operatorname{st}_{ 1}}_{n-1}(t)}^2 }+\E{\abs{v^{(k)}_{n+1}}^2} }\\&\leq C^{\operatorname{st}_{ 1}}_nd\eta^2. 
\end{align*}
Next, we have $v^{\operatorname{st}_{ 1}}_{P-1}(t)=v^{(k)}_{P-1}+\int_{k\eta}^t\left( -\gamma v^{\operatorname{st}_{ 1}}_{P-2}(s)+\gamma v^{(k)}_{P-1}\right)ds+\sqrt{2\gamma}\brac{B_t-B_{k\eta}} $ which along with Lemma~\ref{lemma_momentbound_Pthorder} imply
\begin{align*}
   & \sup_{t\in (k\eta,(k+1)\eta]} \E{\abs{v^{\operatorname{st}_{ 1}}_{P-1}(t)-v^{(k)}_{P-1} }^2 }\leq C^{\operatorname{st}_{ 1}}_{P-1}d\eta. 
\end{align*}

Finally, Lemma~\ref{lemma_momentbound_Pthorder} implies
\begin{align*}
      \sup_{t\in(k\eta,(k+1)\eta]}   \E{\abs{\theta^{\operatorname{st}_{ 1}}(t)-\theta^{(k)}}^2}&\leq \eta^2\gamma^2 \sup_{t\in(k\eta,(k+1)\eta]}\E{\abs{v^{(k)}_{1}}^2 }\leq C^{\operatorname{st}_{ 1}}_Pd\eta^2. 
\end{align*}

\textbf{Step 2: $j=2$ and $P=3$}

We have $v^{\operatorname{st}_{ 2}}_1(t)-v^{\operatorname{st}_{1 }}_1(t)=v^{\operatorname{st}_{ 2}}_1(t)-v^{(k)}_1=\int_{k\eta}^t\left( -g^{\operatorname{st}_{2 }}(s)+\gamma v^{\operatorname{st}_{ 1}}_2(s) \right)  ds$
so that by Lemma~\ref{lemma_momentbound_Pthorder}, 
\begin{align}
\label{bound_v1stage2}
&\sup_{t\in(k\eta,(k+1)\eta]}\E{\abs{v^{\operatorname{st}_{ 2}}_1(t)-v^{\operatorname{st}_{1 }}_1(t) }^2}\nonumber\\
&\leq \eta^2\sup_{t\in(k\eta,(k+1)\eta]}\E{\abs{g^{\operatorname{st}_{2 }}(t)}^2}+\gamma^2\eta^2\sup_{t\in(k\eta,(k+1)\eta]}\E{\abs{v^{\operatorname{st}_{ 1}}_2(t)}^2}\leq  C^{\operatorname{st}_{ 2}}_1 d\eta^2. 
\end{align}

Moreover, $v^{\operatorname{st}_{ 2}}_2(t)-v^{\operatorname{st}_{1 }}_2(t)=\int_{k\eta}^t \left(-\gamma\brac{v^{\operatorname{st}_{ 2}}_1(s)-v^{\operatorname{st}_{ 1}}_1(s) }+\gamma\brac{v^{\operatorname{st}_{ j}}_3(s)-v^{(k)}_3(s) }\right)ds$ so that using~\eqref{bound_v1stage2} and the calculation in \textbf{Step 1}, we get 
\begin{align*}
 \sup_{t\in(k\eta,(k+1)\eta]}   \E{\abs{v^{\operatorname{st}_{ 2}}_2(t)-v^{\operatorname{st}_{1 }}_2(t)}^2}\leq C^{\operatorname{st}_{ 2}}_2 d\eta^4. 
\end{align*}

Finally, by~\eqref{bound_v1stage2}, we get 
\begin{align*}
     \sup_{t\in(k\eta,(k+1)\eta]}   \E{\abs{\theta^{\operatorname{st}_{ 2}}(t)-\theta^{\operatorname{st}_{1 }}(t)}^2}\leq \eta^2  \sup_{t\in (k\eta,(k+1)\eta]} \E{\abs{v^{\operatorname{st}_{ 2}}_1(t)-v^{\operatorname{st}_{1}}_1(t)}^2}\leq C^{\operatorname{st}_{ 2}}_3d\eta^{4}. 
\end{align*}


\textbf{Step 3: $j=2$ and $P\geq 4$}

In the same way as \eqref{bound_v1stage2}, we get 
\begin{align}
\label{bound_v1stage2_Pmorethan4}
    \sup_{t\in(k\eta,(k+1)\eta]}\E{\abs{v^{\operatorname{st}_{ 2}}_1(t)-v^{\operatorname{st}_{1 }}_1(t) }^2}\leq C^{\operatorname{st}_{ 2}}_1 d\eta^2. 
\end{align}
Next, \eqref{bound_v1stage2_Pmorethan4} and the calculation in \textbf{Step 1} imply that 
\begin{align*}
    &\sup_{t\in(k\eta,(k+1)\eta]}   \E{\abs{v^{\operatorname{st}_{ 2}}_2(t)-v^{\operatorname{st}_{1 }}_2(t)}^2}\\&\leq \gamma^2\eta^2\sup_{t\in(k\eta,(k+1)\eta]}  \brac{ \E{\abs{v^{\operatorname{st}_{ 2}}_1(t)-v^{\operatorname{st}_{1 }}_1(t)}^2}+   \E{\abs{v^{\operatorname{st}_{ 2}}_3(t)-v^{(k)}_3}^2} }\leq C^{\operatorname{st}_{ 2}}_2 d\eta^4. 
\end{align*}
Proceed similarly for increasing $n,3\leq n\leq P-3$ to obtain 
\begin{align}
\label{bound_middlestepPmorethan4}
    &\sup_{t\in(k\eta,(k+1)\eta]}   \E{\abs{v^{\operatorname{st}_{ 2}}_n(t)-v^{\operatorname{st}_{ 1}}_n(t)}^2}\nonumber\\
    &\leq \eta^2\gamma^2 \sup_{t\in(k\eta,(k+1)\eta]}  \brac{ \E{\abs{v^{\operatorname{st}_{ 2}}_{n-1}(t)-v^{\operatorname{st}_{1 }}_{n-1}(t)}^2}+   \E{\abs{v^{\operatorname{st}_{ 1}}_{n+1}(t)-v^{(k)}_{n+1}}^2} }\leq C^{\operatorname{st}_{ 2}}_nd\eta^4. 
\end{align}
Furthermore, \eqref{bound_middlestepPmorethan4} and the calculation in \textbf{Step 1} lead to
\begin{align}
\label{bound_P-2_Pmorethan4}
    &\sup_{t\in(k\eta,(k+1)\eta]}   \E{\abs{v^{\operatorname{st}_{ 2}}_{P-2}(t)-v^{\operatorname{st}_{ 1}}_{P-2}(t)}^2}\nonumber\\
    &\leq \eta^2\gamma^2 \sup_{t\in(k\eta,(k+1)\eta]}  \brac{ \E{\abs{v^{\operatorname{st}_{ 2}}_{P-3}(t)-v^{\operatorname{st}_{1 }}_{P-3}(t)}^2}+   \E{\abs{v^{\operatorname{st}_{ 1}}_{P-1}(t)-v^{(k)}_{P-1}}^2} }\nonumber\\
    &\leq \eta^2\gamma^2 \sup_{t\in(k\eta,(k+1)\eta]}  \brac{C^{\operatorname{st}_{ 2}}_{P-3}d\eta^2+C^{\operatorname{st}_{ 2}}_{P-1}d\eta } \leq C^{\operatorname{st}_{ 2}}_{P-2}d\eta^3. 
\end{align}
We also have from \eqref{bound_P-2_Pmorethan4} that 
\begin{align}
\label{bound_vP-1_Pmorethan4}
   & \sup_{t\in (k\eta,(k+1)\eta]} \E{\abs{v^{\operatorname{st}_{ 2}}_{P-1}(t)-v^{\operatorname{st}_{ 1}}_{P-1}(t) }^2 }\nonumber\\
    &\leq  \sup_{t\in (k\eta,(k+1)\eta]} \E{\abs{-\gamma\int_{k\eta}^t e^{-\gamma s}\brac{ v^{\operatorname{st}_{ 2}}_{P-2}(s)-v^{\operatorname{st}_{ 1}}_{P-2}(s)}ds }^2 }\leq C^{\operatorname{st}_{ 2}}_{P-1}d\eta^{5}. 
\end{align}
Finally, \eqref{bound_v1stage2_Pmorethan4} implies 
\begin{align*}
  \sup_{t\in (k\eta,(k+1)\eta]}   \E{\abs{{\theta^{\operatorname{st}_{2 }}(t)}-{\theta^{\operatorname{st}_{1 }}(t)} }^2}&\leq \eta^2  \sup_{t\in (k\eta,(k+1)\eta]} \E{\abs{v^{\operatorname{st}_{ 2}}_1(t)-v^{\operatorname{st}_{1 }}_1(t)}^2}\leq C^{\operatorname{st}_{ j}}_Pd\eta^4. 
\end{align*}
This completes the proof. 
\end{proof}

\begin{proof}[Proof of Lemma~\ref{lemma_stagedifference_Pthorder}]

We will prove the formulas in Parts $a),b),c)$ and $d)$ for Stage $j\geq 3$ via induction. We will assume $P\geq 4$, since there is no Stage $3$ when $P=3$. 

\textbf{First half of the proof: checking the base case $j=3$}

The first half of the proof will consist of four steps.

\textbf{Step 1:} Verifying Part $a)$ for Stage $j=3$. 

We have
\begin{align*}
    &v^{\operatorname{st}_{ 3}}_1(t)-v^{\operatorname{st}_{2 }}_1(t)\\
  &=\int_{k\eta}^t \left(-\brac{g^{\operatorname{st}_{3 }}(s)-g^{\operatorname{st}_{2 }}(s) }+\gamma\brac{v^{\operatorname{st}_{ j}}_2(s)-v^{\operatorname{st}_{1 }}_2(s) } \right) ds\\
  &=\int_{k\eta}^t \bigg(-\brac{g^{\operatorname{st}_{3 }}(s)-\nabla U\brac{\theta^{\operatorname{st}_{2 }}(s)}}-\brac{\nabla U\brac{\theta^{\operatorname{st}_{2 }}(s)}-\nabla U\brac{\theta^{\operatorname{st}_{1 }}(s)} }\\
  &\qquad\qquad-\brac{\nabla U\brac{\theta^{\operatorname{st}_{1 }}(s)}-g^{\operatorname{st}_{2 }}(s) }\bigg)ds+\gamma \int_{k\eta}^t \left(v^{\operatorname{st}_{2 }}_2(s)-v^{\operatorname{st}_{1 }}_2(s) \right) ds.
\end{align*}
so that by $L$-smoothness of $U$, 
\begin{align}
    &\sup_{t\in(k\eta,(k+1)\eta]}\E{\abs{v^{\operatorname{st}_{ 3}}_1(t)-v^{\operatorname{st}_{2}}(t)}^2}\nonumber\\
    &\leq \eta^2\sup_{t\in (k\eta,(k+1)\eta]}\Bigg( \E{\abs{g^{\operatorname{st}_{3 }}(t)-\nabla U\brac{\theta^{\operatorname{st}_{2 }}(t)} }^2}+\E{\abs{\nabla U\brac{\theta^{\operatorname{st}_{1 }}(t)}-g^{\operatorname{st}_{2 }}(t) }^2 } \nonumber\\
    &\qquad\qquad+L^2\E{\abs{\theta^{\operatorname{st}_{2 }}(s)-\theta^{\operatorname{st}_{1 }}(t)}^2 }+\gamma^2\E{\abs{ v^{\operatorname{st}_{2 }}_2(t)-v^{\operatorname{st}_{1 }}_2(t)}^2 } \Bigg). \label{four:terms:2}
\end{align}
Note the third and last terms on the right hand side in \eqref{four:terms:2} are bounded in Lemma~\ref{lemma_momentbound_Pthorder} as 
\begin{align*}
\E{\abs{\theta^{\operatorname{st}_{2 }}(s)-\theta^{\operatorname{st}_{1 }}(t)}^2 }\leq C^{\operatorname{st}_{ 2}}_P d\eta^{4}\qquad\text{and}\qquad\E{\abs{ v^{\operatorname{st}_{2 }}_2(t)-v^{\operatorname{st}_{1 }}_2(t)}^2 }\leq  C^{\operatorname{st}_{ 2}}_2 d\eta^{4}.
\end{align*}
Regarding the first two terms on the right hand side in \eqref{four:terms:2}, similar to the argument at~\eqref{bound_polyapproxtildetheta}, Condition~\ref{cond_derivativegrowthrate} indicates there exists a positive constant $c$ such that 
\begin{align*}
 &\sup_{t\in (k\eta,(k+1)\eta]}   \brac{\E{\abs{g^{\operatorname{st}_{3 }}(t)-\nabla U\brac{\theta^{\operatorname{st}_{2 }}(t)} }^2}+\E{\abs{\nabla U\brac{\theta^{\operatorname{st}_{1 }}(t)}-g^{\operatorname{st}_{2 }}(t) }^2 }}\leq c  d\eta^{2P-1}. 
\end{align*}
Thus, we get
\begin{align}
    \label{bound_v1_stage3}
    \sup_{t\in(k\eta,(k+1)\eta]}\E{\abs{v^{\operatorname{st}_{ 3}}_1(t)-v^{\operatorname{st}_{2}}(t)}^2}\leq C^{\operatorname{st}_{ j+1}}_1d\eta^{6}.  
\end{align}

Next, we have $v^{\operatorname{st}_{ 3}}_2(t)-v^{\operatorname{st}_{2}}_2(t)=\int_{k\eta}^t \left(-\gamma\brac{v^{\operatorname{st}_{ 3}}_1(s)-v^{\operatorname{st}_{ 2}}_1(s) }+\gamma\brac{v^{\operatorname{st}_{ 2}}_3(s)-v^{\operatorname{st}_{ 1}}_3(s) }\right)ds$, 
so that 
\begin{align}
   &\sup_{t\in (k\eta,(k+1)\eta]} \E{\abs{v^{\operatorname{st}_{ 3}}_2(t)-v^{\operatorname{st}_{2 }}_2(t)}^2}\nonumber\\
   &\leq \gamma^2\eta^2 \sup_{t\in (k\eta,(k+1)\eta]}\brac{\E{\abs{v^{\operatorname{st}_{ 3}}_1(t)-v^{\operatorname{st}_{2 }}_1(t)}^2}+\E{\abs{v^{\operatorname{st}_{ 2}}_3(t)-v^{\operatorname{st}_{ 1}}_3(t) }^2} }.\label{two:terms}
\end{align}
The first term on the right hand side in \eqref{two:terms} is bounded at~\eqref{bound_v1_stage3}, while the second term in \eqref{two:terms} is bounded in Lemma~\ref{lemma_basecases} as $\E{\abs{v^{\operatorname{st}_{ 2}}_3(t)-v^{\operatorname{st}_{ 1}}_3(t) }^2}\leq C^{\operatorname{st}_{ 2}}_3d\eta^{4} $. Then 
\begin{align}
    \label{bound_v2:0}
    \sup_{t\in (k\eta,(k+1)\eta]} \E{\abs{v^{\operatorname{st}_{ 3}}_2(t)-v^{\operatorname{st}_{2 }}_2(t)}^2}\leq C^{\operatorname{st}_{ 3}}_2d\eta^{6}.  
\end{align}
Now proceed similarly for increasing $n$ with $3\leq n\leq P-4$ and we get
\begin{align}
\label{bound_v_middle_stage3}
     \sup_{t\in (k\eta,(k+1)\eta]} \E{\abs{v^{\operatorname{st}_{ 3}}_n(t)-v^{\operatorname{st}_{2 }}_2(t)}^2}\leq C^{\operatorname{st}_{ 3}}_nd\eta^{6}, \quad 3\leq n\leq P-4.  
\end{align}

\textbf{Step 2:} Verifying Part $b)$ for Stage $j=3$. 

We use \eqref{bound_v_middle_stage3} in the case $n=P-4$ and Lemma~\ref{lemma_basecases} to get
\begin{align}
\label{bound_vP-j-1_stage3}
   &\sup_{t\in (k\eta,(k+1)\eta]} \E{\abs{v^{\operatorname{st}_{ 3}}_{P-3}(t)-v^{\operatorname{st}_{2 }}_{P-3}(t)}^2}\nonumber\\
   &\leq \gamma^2\eta^2 \sup_{t\in (k\eta,(k+1)\eta]}\brac{\E{\abs{v^{\operatorname{st}_{ 3}}_{P-4}(t)-v^{\operatorname{st}_{2 }}_{P-4}(t)}^2}+\E{\abs{v^{\operatorname{st}_{ 2}}_{P-2}(t)-v^{\operatorname{st}_{ 1}}_{P-2}(t) }^2} }\nonumber\\
   &\leq  \gamma^2\eta^2 \brac{C^{\operatorname{st}_{ 3}}_{P-4}d\eta^{6}+C^{\operatorname{st}_{ 2}}_{P-2}\eta^{3} }\leq C^{\operatorname{st}_{ 3}}_{P-3}d\eta^{5}. 
\end{align}

\textbf{Step 3:} Verifying Part $c)$ for Stage $j=3$. 

Using~\eqref{bound_vP-j-1_stage3} and Lemma~\ref{lemma_basecases}, we can write
\begin{align*}
      &\sup_{t\in (k\eta,(k+1)\eta]} \E{\abs{v^{\operatorname{st}_{ 3}}_{P-2}(t)-v^{\operatorname{st}_{2 }}_{P-2}(t)}^2}\nonumber\\
   &\leq \gamma^2\eta^2 \sup_{t\in (k\eta,(k+1)\eta]}\brac{\E{\abs{v^{\operatorname{st}_{ 3}}_{P-3}(t)-v^{\operatorname{st}_{2 }}_{P-3}(t)}^2}+\E{\abs{v^{\operatorname{st}_{ j}}_{P-1}(t)-v^{\operatorname{st}_{ 1}}_{P-1}(t) }^2} }\nonumber\\
   &\leq  \gamma^2\eta^2 \brac{C^{\operatorname{st}_{ 3}}_{P-4}d\eta^{5}+C^{\operatorname{st}_{ 2}}_{P-2}\eta^{8+2(P-1)-2P-1} }\leq C^{\operatorname{st}_{ 3}}_{P-2}d\eta^{7}. 
\end{align*}
Now proceed similarly for increasing $n,P-j\leq n\leq P-2$ and we get
\begin{align}
\label{bound_v_middlesecond_stage3}
    \sup_{t\in (k\eta,(k+1)\eta]} \E{\abs{v^{\operatorname{st}_{ 3}}_{P-2}(t)-v^{\operatorname{st}_{2 }}_{P-2}(t)}^2}\leq C^{\operatorname{st}_{ 3}}_{n}d\eta^{4+3+2\brac{n-(P-2)}}= C^{\operatorname{st}_{ 3}}_{n}d\eta^{11+2n-2P}. 
\end{align}
Finally, the bound in \eqref{bound_v_middlesecond_stage3} in the case $n=P-2$ implies
\begin{align}
\label{bound_vP-1_stage3}
   & \sup_{t\in (k\eta,(k+1)\eta]} \E{\abs{v^{\operatorname{st}_{ 3}}_{P-1}(t)-v^{\operatorname{st}_{ 2}}_{P-1}(t) }^2 }\nonumber\\
    &\leq  \sup_{t\in (k\eta,(k+1)\eta]} \E{\abs{-\gamma\int_{k\eta}^t e^{-\gamma s}\brac{ v^{\operatorname{st}_{ 3}}_{P-2}(s)-v^{\operatorname{st}_{ 2}}_{P-2}(s)} ds}^2 }\leq C^{\operatorname{st}_{ 3}}_{P-1}d\eta^{9}. 
\end{align}

\textbf{Step 4:} Verifying Part $d)$ for Stage $j=3$. 

By \eqref{bound_v1_stage3}, we have 
\begin{align*}
  \sup_{t\in (k\eta,(k+1)\eta]}   \E{\abs{{\theta^{\operatorname{st}_{3 }}(t)}-{\theta^{\operatorname{st}_{2 }}(t)} }^2}&\leq \eta^2  \sup_{t\in (k\eta,(k+1)\eta]} \E{\abs{v^{\operatorname{st}_{ 3}}_1(t)-v^{\operatorname{st}_{2 }}_1(t)}^2}\leq C^{\operatorname{st}_{ j}}_Pd\eta^{8}. 
\end{align*}


\textbf{Second half of the proof: the induction argument}

The second half will also consist of four steps. As the induction hypothesis, we assume Part~$a),b),c)$ and $d)$ of the current Proposition are true up to Stage~$j$.   

\textbf{Step 1:} Verifying Part $a)$ for Stage $j+1$. 

We have
\begin{align*}
    &v^{\operatorname{st}_{ j+1}}_1(t)-v^{\operatorname{st}_{j }}_1(t)\\
  &=\int_{k\eta}^t \left(-\brac{g^{\operatorname{st}_{j+1 }}(s)-g^{\operatorname{st}_{j }}(s) }+\gamma\brac{v^{\operatorname{st}_{ j}}_2(s)-v^{\operatorname{st}_{j-1 }}_2(s) } \right) ds\\
  &=\int_{k\eta}^t \bigg(-\brac{g^{\operatorname{st}_{j+1 }}(s)-\nabla U\brac{\theta^{\operatorname{st}_{j }}(s)}}-\brac{\nabla U\brac{\theta^{\operatorname{st}_{j }}(t)}-\nabla U\brac{\theta^{\operatorname{st}_{j-1 }}(s)} }\\
  &\qquad\qquad-\brac{\nabla U\brac{\theta^{\operatorname{st}_{j-1 }}(s)}-g^{\operatorname{st}_{j }}(s) }ds+\gamma \int_{k\eta}^t v^{\operatorname{st}_{j }}_2(s)-v^{\operatorname{st}_{j-1 }}_2(s) \bigg) ds,
\end{align*}
so that by $L$-smoothness of $U$, we obtain:
\begin{align}
    &\sup_{t\in(k\eta,(k+1)\eta]}\E{\abs{v^{\operatorname{st}_{ j+1}}_1(t)-v^{\operatorname{st}_{j}}(t)}^2}\nonumber\\
    &\leq \eta^2\sup_{t\in (k\eta,(k+1)\eta]}\Bigg( \E{\abs{g^{\operatorname{st}_{j+1 }}(t)-\nabla U\brac{\theta^{\operatorname{st}_{j }}(t)} }^2}+\E{\abs{\nabla U\brac{\theta^{\operatorname{st}_{j }}(t)}-g^{\operatorname{st}_{j-1 }}(t) }^2 } \nonumber\\
    &\qquad\qquad+L^2\E{\abs{\theta^{\operatorname{st}_{j }}(s)-\theta^{\operatorname{st}_{j-1 }}(t)}^2 }+\gamma^2\E{\abs{ v^{\operatorname{st}_{j }}_2(t)-v^{\operatorname{st}_{j-1 }}_2(t)}^2 } \Bigg). \label{four:terms}
\end{align}
The third and last terms on the right hand side in \eqref{four:terms} are bounded respectively by Part $d)$ and Part $a)$ of the induction hypothesis: 
\begin{align*}
\E{\abs{\theta^{\operatorname{st}_{j }}(s)-\theta^{\operatorname{st}_{j-1 }}(t)}^2 }\leq C^{\operatorname{st}_{ j}}_P d\eta^{2j+2}
\qquad\text{and}\qquad\E{\abs{ v^{\operatorname{st}_{j }}_2(t)-v^{\operatorname{st}_{j-1 }}_2(t)}^2 }\leq  C^{\operatorname{st}_{ j}}_2 d\eta^{2j}. 
\end{align*}
Regarding the first two terms on the right hand side in \eqref{four:terms}, similar to the argument at~\eqref{bound_polyapproxtildetheta}, Condition~\ref{cond_derivativegrowthrate} indicates there is a positive constant $c$ such that 
\begin{align*}
 &\sup_{t\in (k\eta,(k+1)\eta]}   \brac{\E{\abs{g^{\operatorname{st}_{j+1 }}(t)-\nabla U\brac{\theta^{\operatorname{st}_{j }}(t)} }^2}+\E{\abs{\nabla U\brac{\theta^{\operatorname{st}_{j }}(t)}-g^{\operatorname{st}_{j-1 }}(t) }^2 }}\leq c  d\eta^{2P-1}. 
\end{align*}

Since for $1\leq j\leq P-1$, we have $2j\leq 2P-1$, the above calculations lead to
\begin{align}
    \label{bound_v1}
    \sup_{t\in(k\eta,(k+1)\eta]}\E{\abs{v^{\operatorname{st}_{ j+1}}_1(t)-v^{\operatorname{st}_{j}}(t)}^2}\leq C^{\operatorname{st}_{ j+1}}_1d\eta^{2j+2}.  
\end{align}
Next, we have $v^{\operatorname{st}_{ j+1}}_2(t)-v^{\operatorname{st}_{j }}_2(t)=\int_{k\eta}^t\left( -\gamma\brac{v^{\operatorname{st}_{ j+1}}_1(s)-v^{\operatorname{st}_{ j}}_1(s) }+\gamma\brac{v^{\operatorname{st}_{ j}}_3(s)-v^{\operatorname{st}_{ j-1}}_3(s) }\right)ds$, 
so that 
\begin{align*}
   &\sup_{t\in (k\eta,(k+1)\eta]} \E{\abs{v^{\operatorname{st}_{ j+1}}_2(t)-v^{\operatorname{st}_{j }}_2(t)}^2}\\
   &\leq \gamma^2\eta^2 \sup_{t\in (k\eta,(k+1)\eta]}\brac{\E{\abs{v^{\operatorname{st}_{ j+1}}_1(t)-v^{\operatorname{st}_{j }}_1(t)}^2}+\E{\abs{v^{\operatorname{st}_{ j}}_3(t)-v^{\operatorname{st}_{ j-1}}_3(t) }^2} }.
\end{align*}
The first term on the right hand side is bounded at~\eqref{bound_v1}, while the second term is bounded per Part $a)$ of the induction hypothesis as $\E{\abs{v^{\operatorname{st}_{ j}}_3(t)-v^{\operatorname{st}_{ j-1}}_3(t) }^2}\leq C^{\operatorname{st}_{ j}}_3d\eta^{2j} $. Then 
\begin{align}
    \label{bound_v2}
    \sup_{t\in (k\eta,(k+1)\eta]} \E{\abs{v^{\operatorname{st}_{ j+1}}_2(t)-v^{\operatorname{st}_{j }}_2(t)}^2}\leq C^{\operatorname{st}_{ j+1}}_2d\eta^{2j+2}.  
\end{align}
Now proceed similarly for increasing $n$ with $3\leq n\leq P-j-2$ and we get
\begin{align}
\label{bound_v_middle}
     \sup_{t\in (k\eta,(k+1)\eta]} \E{\abs{v^{\operatorname{st}_{ j+1}}_n(t)-v^{\operatorname{st}_{j }}_2(t)}^2}\leq C^{\operatorname{st}_{ j+1}}_nd\eta^{2j+2}, \quad 3\leq n\leq P-j-2.  
\end{align}
Via \eqref{bound_v1}, \eqref{bound_v2} and  \eqref{bound_v_middle}, we confirm via induction that Part $a)$ of this Proposition is true. 

\textbf{Step 2:} Verifying Part $b)$ for Stage $j+1$. 

We use \eqref{bound_v_middle} in the case $n=P-j-2$ and Part $b)$ of the induction hypothesis to get
\begin{align}
\label{bound_vP-j-1}
   &\sup_{t\in (k\eta,(k+1)\eta]} \E{\abs{v^{\operatorname{st}_{ j+1}}_{P-j-1}(t)-v^{\operatorname{st}_{j }}_{P-j-1}(t)}^2}\nonumber\\
   &\leq \gamma^2\eta^2 \sup_{t\in (k\eta,(k+1)\eta]}\brac{\E{\abs{v^{\operatorname{st}_{ j+1}}_{P-j-2}(t)-v^{\operatorname{st}_{j }}_{P-j-2}(t)}^2}+\E{\abs{v^{\operatorname{st}_{ j}}_{P-j}(t)-v^{\operatorname{st}_{ j-1}}_{P-j}(t) }^2} }\nonumber\\
   &\leq  \gamma^2\eta^2 \brac{C^{\operatorname{st}_{ j+1}}_{P-j-2}d\eta^{2j+2}+C^{\operatorname{st}_{ j}}_{P-j}\eta^{2j-1} }\leq C^{\operatorname{st}_{ j+1}}_{P-j-1}d\eta^{2j+1}. 
\end{align}
Thus, Part $b)$ of this Proposition is true. 

\textbf{Step 3:} Verifying Part $c)$ for Stage $j+1$. 

Using~\eqref{bound_vP-j-1} and Part $c)$ of the induction hypothesis, we can write
\begin{align*}
      &\sup_{t\in (k\eta,(k+1)\eta]} \E{\abs{v^{\operatorname{st}_{ j+1}}_{P-j}(t)-v^{\operatorname{st}_{j }}_{P-j}(t)}^2}\nonumber\\
   &\leq \gamma^2\eta^2 \sup_{t\in (k\eta,(k+1)\eta]}\brac{\E{\abs{v^{\operatorname{st}_{ j+1}}_{P-j-1}(t)-v^{\operatorname{st}_{j }}_{P-j-1}(t)}^2}+\E{\abs{v^{\operatorname{st}_{ j}}_{P-j+1}(t)-v^{\operatorname{st}_{ j-1}}_{P-j+1}(t) }^2} }\nonumber\\
   &\leq  \gamma^2\eta^2 \brac{C^{\operatorname{st}_{ j+1}}_{P-j-2}d\eta^{2j+1}+C^{\operatorname{st}_{ j}}_{P-j}\eta^{4j+2(P-j+1)-2P-1} }\leq C^{\operatorname{st}_{ j+1}}_{P-j}d\eta^{2j+3}. 
\end{align*}
Now proceed similarly for increasing $n,P-j\leq n\leq P-2$ and we get
\begin{align}
\label{bound_v_middlesecond}
    \sup_{t\in (k\eta,(k+1)\eta]} \E{\abs{v^{\operatorname{st}_{ j+1}}_{P-j}(t)-v^{\operatorname{st}_{j }}_{P-j}(t)}^2}\leq C^{\operatorname{st}_{ j+1}}_{n}d\eta^{2j+3+2\brac{n-(P-j)}}= C^{\operatorname{st}_{ j+1}}_{n}d\eta^{4(j+1)+2n-2P-1}. 
\end{align}
Finally, the bound in \eqref{bound_v_middlesecond} in the case $n=P-2$ implies
\begin{align}
\label{bound_vP-1}
   & \sup_{t\in (k\eta,(k+1)\eta]} \E{\abs{v^{\operatorname{st}_{ j+1}}_{P-1}(t)-v^{\operatorname{st}_{ j}}_{P-1}(t) }^2 }\nonumber\\
    &\leq  \sup_{t\in (k\eta,(k+1)\eta]} \E{\abs{-\gamma\int_{k\eta}^t e^{-\gamma s}\brac{ v^{\operatorname{st}_{ j+1}}_{P-2}(s)-v^{\operatorname{st}_{ j}}_{P-2}(s)} }^2 }\leq C^{\operatorname{st}_{ j+1}}_{P-1}d\eta^{4j+1}. 
\end{align}
By \eqref{bound_v_middlesecond} and \eqref{bound_vP-1}, we conclude via induction that Part $c)$ of this Proposition is true. 

\textbf{Step 4:} Verifying Part $d)$ for Stage $j+1$. 

By \eqref{bound_v1}, we have 
\begin{align*}
  \sup_{t\in (k\eta,(k+1)\eta]}   \E{\abs{{\theta^{\operatorname{st}_{j+1 }}(t)}-{\theta^{\operatorname{st}_{j }}(t)} }^2}&\leq \eta^2  \sup_{t\in (k\eta,(k+1)\eta]} \E{\abs{v^{\operatorname{st}_{ j+1}}_1(t)-v^{\operatorname{st}_{j }}_1(t)}^2}\leq C^{\operatorname{st}_{ j}}_Pd\eta^{2j+4},
\end{align*}
so that Part $d)$ of this Proposition is true. This also completes our induction argument.

The estimate for Stage $P-1$ of the Proposition is straightforward given the previous results and the results in Lemma~\ref{lemma_basecases}. This completes the proof. 
\end{proof}

\begin{proof}[Proof of Theorem~\ref{theorem_mixingtime_Pthorder}]

\end{proof}


\section{Choice of Polynomial Approximation}
\label{section_polyappro}

In this appendix, we expand on Remark~\ref{remark_polyapprox} regarding the difficulty in applying Lagrange polynomial interpolation to our MCMC algorithm based on fourth-order Langevin dynamics. 

    Recall from \cite[Section 3.3]{mou2021high} and also from \cite{stoer1980book}, the Chebyshev nodes on the interval $[k\eta,(k+1)\eta]$ are $  s_i=k\eta+\frac{\eta}{2}\brac{1+\cos\brac{\frac{2i-1}{2\alpha}\pi} }, i=1,2,\ldots,\alpha$. Then the $(\alpha-1)$-degree Lagrange polynomial associated with a $\R^d$-valued path $t\in [k\eta,(k+1)\eta]\mapsto z(t)$ is $\phi_z(t):=\sum_{i=1}^\alpha z(s_i)\prod_{j\neq i}\frac{t-s_i}{s_j-s_i}$. The error estimate when $z$ has up to $\alpha$-th order derivatives is (\cite[Section 3.1]{stoer1980book})
\begin{align}
    \label{error_lagrangeinterpolation}
 \sup_{t\in [k\eta,(k+1)\eta]} \abs{z(t)-\phi_z(t)}\leq \frac{\eta^\alpha}{2^{\alpha-1}\alpha!} \sup_{t\in[k\eta,(k+1)\eta] }\abs{\frac{d^\alpha}{dt^\alpha}z(t)}. 
\end{align}

Coming back to our MCMC algorithm based on fourth-order Langevin dynamics, we need to approximate the path
\begin{align}
    \label{easypath}
  p_1(t):  t\mapsto \nabla U\brac{\theta^{(k)}+(t-k\eta)v_1^{(k)}}, 
\end{align}
and also the path
\begin{align}
    \label{hardpath}
   p_2(t): s\mapsto \nabla U\brac{\tilde{\theta}(t)},
\end{align}
where
\begin{align*}
    \tilde{\theta}(t)&=\theta^{(k)}+v_1^{(k)}(t-k\eta)-\int_{k\eta}^{t}\int_{k\eta}^s  \nabla U\brac{\theta^{(k)}+(r-k\eta)v_1^{(k)}}drds\\
    &\qquad\qquad+\gamma v_2^{(k)}\frac{(t-k\eta)^2}{2!}+\gamma^2\brac{v_3^{(k)}-v_1^{(k)} }\frac{(t-k\eta)^3}{3!}. 
\end{align*}
Lagrange polynomial interpolation of the path $p_1$ in \eqref{easypath} has been done in \cite{mou2021high} by defining $g_1(t):=\sum_{i=1}^\alpha \nabla U\brac{\theta^{(k)}+(s_i-k\eta)v_1^{(k)}}\prod_{j\neq i}\frac{t-s_i}{s_j-s_i}$. Note that $g_1(t)$ is a polynomial of degree $\alpha-1$ in $t$, and the error $\sup_{t\in[k\eta,(k+1)\eta]} \abs{p_1(t)-g_1(t) }$ is bounded in \cite[Section 4.3.2]{mou2021high} using \eqref{error_lagrangeinterpolation} as 
\begin{align*}
   & \sup_{t\in[k\eta,(k+1)\eta]} \abs{p_1(t)-g_1(t) }\leq \frac{\eta^\alpha}{2^{\alpha-1}\alpha!} \sup_{t\in[k\eta,(k+1)\eta] }\abs{\frac{d^\alpha}{dt^\alpha}\nabla U\brac{\theta^{(k)}+(t-k\eta)v_1^{(k)}}}\\
    &\qquad \leq  \frac{\eta^\alpha}{2^{\alpha-1}\alpha!} \sup_{t\in[k\eta,(k+1)\eta] }\abs{\nabla^\alpha U\brac{\theta^{(k)}+(t-k\eta)v_1^{(k)}}v_1^{(k)} }. 
\end{align*}
We observe that this bound is simple since $t\mapsto\theta^{(k)}+(t-k\eta)v_1^{(k)}$ is a linear function in $t$, so that second and higher derivatives of $t\mapsto\theta^{(k)}+(t-k\eta)v_1^{(k)}$ immediately vanish. 

Meanwhile, we can approximate the path $p_2$ by defining 
\begin{equation*}
g_2(t):=\sum_{i=1}^\alpha \nabla U(T(s_i))\prod_{j\neq i}\frac{t-s_i}{s_j-s_i},
\end{equation*}
where 
\begin{align*}
&T(t)=\theta^{(k)}+v_1^{(k)}(t-k\eta)-\int_{k\eta}^{t}\int_{k\eta}^s g_1(r)drds\\
    &\qquad\qquad+\gamma v_2^{(k)}\frac{(t-k\eta)^2}{2!}+\gamma^2\brac{v_3^{(k)}-v_1^{(k)} }\frac{(t-k\eta)^3}{3!}. 
\end{align*}
From \eqref{error_lagrangeinterpolation}, we get 
\begin{align*}
    &\sup_{t\in [k\eta,(k+1)\eta]} \abs{g_2(t)-\nabla U(T(t)) }\leq \frac{\eta^{\alpha-1}}{2^{\alpha-2}(\alpha-1)!}\sup_{t\in [k\eta,(k+1)\eta]} \abs{\frac{d^{\alpha-1}}{dt^{\alpha-1}} \nabla U(T(t))}. 
\end{align*}
In particular, the fact that  $g_1(t)$ is a polynomial of degree $\alpha-1$ suggests $T(t)$ is a polynomial of degree $\alpha+1$. We use Fa\`{a} Di Bruno's formula to get
\begin{align}
\label{faadibruno}
    \frac{d^{\alpha-1}}{dt^{\alpha-1}} \nabla U(T(t))= \sum_{M_{\alpha-1}}\frac{(\alpha-1)!}{\prod_{i=1}^{\alpha-1}m_i i^{m_i} }\nabla^{1+\sum_{i=1}^{\alpha-1}m_i} U(T(t))\prod_{i=1}^{\alpha-1}\brac{\frac{d^i}{dt^i}T(t)}^{m_i},
\end{align}
where $M_{\alpha-1}:=\{(m_1,\ldots,m_{\alpha-1}): m_i\geq 0\text{ and } \sum_{i=1}^{\alpha-1}im_i=\alpha-1\}$. Since $T(t)$ is not a linear function and is a polynomial of potentially high degree, most terms in \eqref{faadibruno} does not vanish, which makes the error bound quite challenging. 

\section{Extra Calculations for the Numerical Experiments}
\label{app:extra}
\subsection{Quadratic loss function}
In this section, we consider the case where the loss function $U(\theta)$ is quadratic. Consider the mean-squared error in a regression problem with Ridge regularization.If we have a training dataset $Z=\{z_1,z_2,\cdots,z_n\}$, with $z_i=(X_i,y_i), i=1,2,\cdots,n$. Here, $X_i\in\mathbb{R}^{d}$ is a $d$-dimensional input and $y_i\in \mathbb{R}$ is the one-dimensional output. Then the potential (loss) function is defined as
\begin{equation}
    U(\theta) = \frac{1}{2n} \sum_{i=1}^{n} \left(y_i-\theta^{\top}X_i\right)^2+\frac{\lambda}{2}|\theta|^2 = \frac{1}{2n}|y-X\theta|^2+\frac{\lambda}{2}|\theta|^2.
    \label{eq:n1}
\end{equation}
The gradient of the loss is then given by 
\begin{equation}
    \nabla U(\theta) = -\frac{1}{n} X^{\top}(y-X\theta) +\lambda \theta = \frac{1}{n}X^{\top}X\theta -\frac{1}{n}X^{\top}y+\lambda\theta=\left(\frac{1}{n}X^{\top}X+\lambda I\right)\theta-\frac{1}{n}X^{\top}y.
\end{equation}
Define $A=\left(\frac{1}{n}X^{\top}X+\lambda I\right)$ and $b=\frac{1}{n}X^{\top}y$. Then $\nabla U(\theta)=A\theta-b$ is linear in $\theta$.

\textbf{Third-order computation:}
Now we are ready to get an explicit form of the vector $\Delta U(\theta, v_1)$ used in equation~(\ref{eq:o3}) for the third-order dynamics
\begin{align*}
    \Delta U(\theta,v_1): & = \int_0^{\eta} \nabla U(\theta + t v_1)dt=\int_0^{\eta}\left[A(\theta+tv_1)-b\right]dt
    =A\left(\eta\theta+\frac{\eta^2}{2}v_1\right)-b\eta.
\end{align*}

\textbf{Fourth-order computation:}
We use the following version of the formulas to compute the mean vector components $m_i$s' which are the expanded forms of the formulas presented in Lemma~\ref{lemma_meanandcovariance}. Note that, except for $\theta, v_1, v_2,$ and $v_3$, all other variables used in the computation processes are dummy variables.
\begin{align*}
m_0 & = - \int_{k\eta}^{(k+1)\eta} \int_{k\eta}^{s} \nabla U \Bigg(\theta + (r-k\eta)v_1 -\int_{k\eta}^{r}\int_{k\eta}^{w}\nabla U (\theta + (y-k\eta)v_1)dydw \\
&\qquad +\gamma v_2\frac{(r-k\eta)^2}{2!}+\gamma^2(-v_1+v_3)\frac{(r-k\eta)^3}{3!}\Bigg)drds\\
&\qquad\qquad+\gamma^2\int_{k\eta}^{(k+1)\eta}\int_{k\eta}^{s}\int_{k\eta}^{r}\int_{k\eta}^{w}\nabla U(\theta +(y-k\eta)v_1)dydwdrds\\
&\qquad\qquad\qquad+ \theta \mu_{00}+v_1\mu_{01}+v_2\mu_{02}+v_3\mu_{03}.
\end{align*}
Let us split the integral into small parts, we have
\begin{align*}
m_0 = - \int_{k\eta}^{(k+1)\eta} \int_{k\eta}^{s} \nabla U (\theta + (r-k\eta)v_1 - T_1 + T_2) drds + T_3 + T_4,
\end{align*}
where
\begin{align*}
T_1 & = \int_{k\eta}^{r}\int_{k\eta}^{w}\nabla U (\theta + (y-k\eta)v_1)dydw,\\
T_2 & = \gamma v_2\frac{(r-k\eta)^2}{2!}+\gamma^2(-v_1+v_3)\frac{(r-k\eta)^3}{3!},\\
T_3 & = \gamma^2\int_{k\eta}^{(k+1)\eta}\int_{k\eta}^{s}\int_{k\eta}^{r}\int_{k\eta}^{w}\nabla U(\theta +(y-k\eta)v_1)dydwdrds,\\
T_4 & = \theta \mu_{00}+v_1\mu_{01}+v_2\mu_{02}+v_3\mu_{03}.
\end{align*}
This implies 
\begin{align*}
    m_0 &=\theta + \mu_{01}v_1+\mu_{02}v_2+\mu_{03}v_3\\
    &\qquad + \left(\frac{\eta ^4\gamma ^2 }{24}-\frac{\eta^2}{2}\right) (A  \theta-b)+\frac{\eta^4}{24}(A(A\theta-b))\\
    &\qquad\qquad  + \left(\frac{ \eta ^5\gamma ^2}{60}-\frac{\eta ^3}{6}\right)A v_1+\frac{\eta ^5}{120} A (A v_1)-\frac{\eta ^4\gamma}{24} A v_2 -\frac{\eta ^5\gamma ^2 }{120} A v_3,
\end{align*}  
where we used $\mu_{00}=1$.
Next,
\begin{align*}
m_1 & = - \int_{k\eta}^{(k+1)\eta}\nabla U\Bigg(\theta+v_1(s-k\eta)-\int_{k\eta}^{s}\int_{k\eta}^r\nabla U\left(\theta+(w-k\eta)v_1\right)dwdr+ \gamma v_2\frac{(s-k\eta)^2}{2!}\\&+\gamma^2\left(-v_1+v_3\right)\frac{(s-k\eta)^3}{3!}\Bigg)ds + \gamma^2 \int_{k\eta}^{(k+1)\eta}\int_{k\eta}^{s}\int_{k\eta}^{r}\nabla U\left(\theta+(w-k\eta)v_1\right)dwdrds\\
&+ \theta\mu_{10}+v_1\mu_{11}+v_2\mu_{12}+v_3\mu_{13}.
\end{align*}  
Split the integral into smaller parts, we have
\begin{align*}
m_1 = -\int_{k\eta}^{(k+1)\eta} \nabla U\left(\theta + (s-k\eta)v_1-T_1 +T_2\right)ds + T_3+T_4,
\end{align*}
where
\begin{align*}
T_1 & = \int_{k\eta}^{s}\int_{k\eta}^r\nabla U\left(\theta+(w-k\eta)v_1\right)dwdr,\\
T_2 & = \gamma v_2\frac{(s-k\eta)^2}{2!}+\gamma^2\left(-v_1+v_3\right)\frac{(s-k\eta)^3}{3!},\\
T_3 & = \gamma^2 \int_{k\eta}^{(k+1)\eta}\int_{k\eta}^{s}\int_{k\eta}^{r}\nabla U\left(\theta+(w-k\eta)v_1\right)dwdrds,\\
T_4 & =\theta\mu_{10}+v_1\mu_{11}+v_2\mu_{12}+v_3\mu_{13},
\end{align*}  
which implies 
\begin{align*}
m_1 & = \mu _{11} v_1+\mu _{12} v_2+\mu _{13} v_3+\left(\frac{\eta ^3\gamma ^2 }{6}-\eta\right) (A \theta- b)+\frac{\eta ^3}{6} A(A  \theta-b)\\
  &\qquad +\left(\frac{\eta ^4\gamma ^2 }{12}-\frac{\eta ^2}{2}\right)A v_1 +\frac{\eta ^4}{24} A(A  v_1)-\frac{\eta ^3\gamma  }{6} A  v_2-\frac{\eta ^4\gamma ^2 }{24} A  v_3, 
\end{align*} 
where we used $\mu_{10}=0$.
Next, we compute $m_2$ as follows:
\begin{align*}
    m_2& = \gamma \int_{k\eta}^{(k+1)\eta}\int_{k\eta}^s \nabla U\Bigg(\theta+v_1(r-k\eta)-\int_{k\eta}^r\int_{k\eta}^w\nabla U\Big(\theta+(y-k\eta)v_1\Big)dydw \\
    &+\gamma v_2\frac{(r-k\eta)^2}{2!}+\gamma^2\Big(-v_1+v_3\Big)\frac{(r-k\eta)^3}{3!}\Bigg)drds\\
    & - \gamma^3 \int_{k\eta}^{(k+1)\eta}\int_{k\eta}^s\int_{k\eta}^r\int_{k\eta}^w\nabla U\Big(\theta+(y-k\eta)v_1\Big)dydwdrds\\
    &- \gamma^3 \int_{k\eta}^{(k+1)\eta}\int_{k\eta}^s\int_{k\eta}^r\int_{k\eta}^w e^{-\gamma(s-r)}\nabla U\Big(\theta+(y-k\eta)v_1\Big)dydwdrds\\
    &+\theta\mu_{20}+v_1\mu_{21}+v_2\mu_{22}+v_3\mu_{23}.
\end{align*}
Split the integral into smaller parts. Define 
\begin{align*}
T_1 & = \int_{k\eta}^r\int_{k\eta}^w\nabla U\Big(\theta+(y-k\eta)v_1\Big)dydw\\
T_2 & = \gamma v_2\frac{(r-k\eta)^2}{2!}+\gamma^2\Big(-v_1+v_3\Big)\frac{(r-k\eta)^3}{3!}\\
T_3 & = - \gamma^3 \int_{k\eta}^{(k+1)\eta}\int_{k\eta}^s\int_{k\eta}^r\int_{k\eta}^w\nabla U\Big(\theta+(y-k\eta)v_1\Big)dydwdrds\\
T_4 & = - \gamma^3 \int_{k\eta}^{(k+1)\eta}\int_{k\eta}^s\int_{k\eta}^r\int_{k\eta}^w e^{-\gamma(s-r)}\nabla U\Big(\theta+(y-k\eta)v_1\Big)dydwdrds\\
T_5 & = \theta\mu_{20}+v_1\mu_{21}+v_2\mu_{22}+v_3\mu_{23}.
\end{align*}  
Then the integral becomes  
\begin{align*}
m_2& = \gamma \int_{k\eta}^{(k+1)\eta}\int_{k\eta}^s \nabla U(\theta + (r-k\eta)v_1-T_1+T_2)drds +T_3 +T_4+T_5\\
& =\mu_{21}v_1 + \mu_{22}v_2 + \mu_{23}v_3+ \Big(\frac{1-e^{-\eta\gamma}}{\gamma}-\frac{\eta^4\gamma^3}{24}-\frac{\eta^3\gamma^2}{6}+\eta^2\gamma-\eta\Big)(A\theta-b)\\
    &-\frac{\eta^4\gamma}{24}(A(A\theta-b))+ \Big(-\frac{\eta^5\gamma^3}{60}-\frac{\eta^4\gamma^2}{24}+\frac{\eta^3\gamma}{3}-\frac{\eta^2}{2}+\frac{\eta}{\gamma}-\frac{1-e^{-\eta\gamma}}{\gamma^2}\Big)(Av_1)\\
    &-\frac{\eta^5\gamma}{120}(A(Av_1)) + \frac{\eta^4\gamma^3}{24}(Av_2) +\frac{\eta^5\gamma^3}{120}(Av_3),
\end{align*}
where we used $\mu_{20}=0$.
Finally,
\begin{align*}
    m_3&=-\gamma^2\int_{k\eta}^{(k+1)\eta}\int_{k\eta}^s \int_{k\eta}^r e^{-\gamma((k+1)\eta- s)}\\
    &\cdot\nabla U \Bigg(\theta+(w-k\eta)v_1 -\int_{k\eta}^w\int_{k\eta}^y \nabla U\Big(\theta+(z-k\eta)v_1\Big)dzdy+\gamma v_2\frac{(w-k\eta)^2}{2!}\\
    &+\gamma^2\Big(-v_1+v_3 \Big)\frac{(w-k\eta)^3}{3!} \Bigg)dwdrds\\
    &+\gamma^4\int_{k\eta}^{(k+1)\eta}\int_{k\eta}^s \int_{k\eta}^r\int_{k\eta}^w\int_{k\eta}^ye^{-\gamma((k+1)\eta- s)}\nabla U\Big(\theta+(z-k\eta)v_1\Big)dzdydwdrds\\
      &+\gamma^4\int_{k\eta}^{(k+1)\eta}\int_{k\eta}^s\int_{k\eta}^r\int_{k\eta}^w\int_{k\eta}^y e^{-\gamma((k+1)\eta- s)}e^{-\gamma(r- w)}\nabla U\Big(\theta+(z-k\eta)v_1\Big)dzdydwdrds\\
& +\theta\mu_{30}+v_1\mu_{31}+v_2\mu_{32}+v_3\mu_{33}. 
\end{align*}
Denote  
\begin{align*}
T_1 & = \int_{k\eta}^w\int_{k\eta}^y \nabla U\Big(\theta+(z-k\eta)v_1\Big)dzdy,\\
T_2 & = \gamma v_2\frac{(w-k\eta)^2}{2!}+\gamma^2\Big(-v_1+v_3 \Big)\frac{(w-k\eta)^3}{3!},\\
T_3 & = \gamma^4\int_{k\eta}^{(k+1)\eta}\int_{k\eta}^s \int_{k\eta}^r\int_{k\eta}^w\int_{k\eta}^ye^{-\gamma((k+1)\eta- s)}\nabla U\Big(\theta+(z-k\eta)v_1\Big)dzdydwdrds,\\
T_4 & = \gamma^4\int_{k\eta}^{(k+1)\eta}\int_{k\eta}^s\int_{k\eta}^r\int_{k\eta}^w\int_{k\eta}^y e^{-\gamma((k+1)\eta- s)}e^{-\gamma(r- w)}\nabla U\Big(\theta+(z-k\eta)v_1\Big)dzdydwdrds,\\
T_5 & = \theta\mu_{30}+v_1\mu_{31}+v_2\mu_{32}+v_3\mu_{33}. 
\end{align*}
Then the integral becomes 
\begin{align*}
m_3 & = -\gamma^2 \int_{k\eta}^{(k+1)\eta}\int_{k\eta}^{s}\int_{k\eta}^{r} e^{-\gamma((k+1)\eta-s)}\nabla U(\theta+(w-k\eta)v_1-T_1+T_2)dwdrds
\\
&\qquad\qquad\qquad+T_3 +T_4+T_5.
\end{align*}
Re-arranging, we have 
\begin{align*}
    m_3 & = \mu_{31}v_1+\mu_{32}v_2+\mu_{33}v_3\\
    &\quad+ \left(\frac{\eta^4\gamma^3}{24}-\eta^2\gamma+\eta\left(3+e^{-\eta\gamma}\right)-\frac{4\left(1-e^{-\eta\gamma}\right)}{\gamma}\right)(A\theta-b)\\
    & \quad\quad+ \left(\frac{\eta^4\gamma}{24}-\frac{\eta^3}{6}+\frac{\eta^2}{2\gamma}-\frac{\eta}{\gamma^2}+\frac{1-e^{-\eta\gamma}}{\gamma^3}\right)A (A\theta-b)\\
    & \quad\quad\quad+\left(\frac{\eta^5 \gamma^3}{60}-\frac{\eta^4 \gamma^2}{24}-\frac{\eta^3 \gamma}{6}+\eta^2-\frac{4 e^{-\eta \gamma}}{\gamma^2}-\frac{\eta e^{-\eta \gamma}}{\gamma}-\frac{3 \eta}{\gamma}+\frac{4}{\gamma^2}\right)(A v_1 )\\
    & \quad\quad\quad\quad+\left(\frac{\eta^5 \gamma}{120}-\frac{\eta^4}{24}+\frac{\eta^3}{6 \gamma}-\frac{\eta^2}{2 \gamma^2}+\frac{e^{-\eta \gamma}}{\gamma^4}+\frac{\eta}{\gamma^3}-\frac{1}{\gamma^4}\right)(A(Av_1)) \\
    & \quad\quad\quad\quad\quad+ \left(-\frac{1}{24} \eta^4 \gamma^2+\frac{\eta^3 \gamma}{6}-\frac{\eta^2}{2}+\frac{e^{-\eta \gamma}}{\gamma^2}+\frac{\eta}{\gamma}-\frac{1}{\gamma^2}\right)(A v_2)\\
    & \quad\quad\quad\quad\quad\quad+\Big(-\frac{1}{120} \eta^5 \gamma^3+\frac{\eta^4 \gamma^2}{24}-\frac{\eta^3 \gamma}{6}+\frac{\eta^2}{2}-\frac{e^{-\eta \gamma}}{\gamma^2}-\frac{\eta}{\gamma}+\frac{1}{\gamma^2}\Big)(A v_3 ),
\end{align*}
where we used $\mu_{30}=0$.
    
\subsection{Logistic loss function}
\label{dis:logloss}
Similar to the quadratic case, let us assume that we have an input data set $X\in \mathbb{R}^{n\times d}$, an output dataset $\mb{y}\in \{0,1\}^n$, and $\theta \in \mathbb{R}^d$ being the model parameters or weights. Then the predicted probability for the $i$-th sample $y_i=1$ is
\begin{equation}
   \sigma(z_i) =\mathbb{P}\left(y_i=1| X_i;\theta\right)= \frac{1}{1+e^{-z_i}}=\hat{y}_i;
   \label{eq:sigmoid}
\end{equation}
where $z_i = X_i^{\top}\theta\in \mathbb{R}$ and $\sigma(z)$ is a real-valued function. However, if $z=X\theta \in \mathbb{R}^n$ then we define the vector-valued sigmoid function as
\begin{equation}
    \Vec{\sigma}(z):=\frac{1}{1+e^{-z}}:=\left(\frac{1}{1+e^{-z_{1}}},\ldots,\frac{1}{1+e^{-z_{n}}}\right).
    \label{eq:sigmoidvec}
\end{equation}

Therefore, for a two-class classification problem, we define 
\begin{equation}
    \begin{split}
        \mathbb{P}(y_i=1|X_i;\theta)& =\hat{y}_i=\sigma(x_i^{\top}\theta), \\
        \mathbb{P}(y_i=0|X_i;\theta)&=1-\hat{y}_i=1-\sigma(X_i^{\top}\theta). \label{eq:prob}
    \end{split}
\end{equation}
Then for a given $y_i$, the probability of taking one of the classes, we combine the above equations (\ref{eq:prob}) into a single equation
\begin{equation}
    \mathbb{P}(y_i|X_i;\theta) = \hat{y_i}^{y_i}(1-\hat{y}_i)^{1-y_i}.
\end{equation}
For the independent and identically distributed (\textit{i.i.d.}) data we define the loss
\begin{align*}
    \mathcal{L}(\theta) = \prod_{i=1}^n \mathbb{P}(y_i|X_i;\theta)= \prod_{i=1}^n \hat{y_i}^{y_i}(1-\hat{y}_i)^{1-y_i},
\end{align*}
which implies 
\begin{align*}
\log \mathcal{L}(\theta) =\sum_{i=0}^n\left[y_i\log (\hat{y}_i)+ (1-y_i)\log (1-\hat{y}_i)\right].
\end{align*}
We take the negative of $\log$ likelihood as the probabilities are often smaller numbers. Then we define the potential function with a penalty term (i.e., $L_2$ or Ridge regularization),
\begin{align*}
    U (\theta) &= -\sum_{i=0}^n\left[y_i\log (\hat{y}_i)+ (1-y_i)\log (1-\hat{y}_i)\right]  + \frac{\lambda}{2}|\theta|^2\\
    & = -\sum_{i=0}^n\left[y_i\log (\sigma(z_i))+ (1-y_i)\log (1-\sigma(z_i))\right]  + \frac{\lambda}{2}|\theta|^2.
\end{align*}
Therefore, the gradient of the regularized loss function in vector form would be 
\begin{align}\label{nabla:U:formula}
   \nabla U(\theta) = X^{\top}(\hat{y}-y)+\lambda\theta= X^{\top}(\Vec{\sigma}(X\theta)-y) +\lambda \theta,
\end{align}
where $\Vec{\sigma}$ is defined in \eqref{eq:sigmoidvec}.
The detailed derivation of $\nabla U(\theta)$ in \eqref{nabla:U:formula} will be given in Lemma~\ref{dis:grad-log}.


The fourth-order sampling process requires the computation of the integral of the gradient
$\int_0^t \nabla U(\theta + tv_1)$
for any $t\in [0,\eta]$; however, for a non-polynomial or black-box potential, it is quite hard or sometimes impossible to compute the exact integrals. Thus, in this case, we approximate the integrals using \textbf{Taylor Series} expansion.

For the Taylor expansion, let us define,
\begin{align*}
      z(t) := X(\theta + t v_1) = X\theta + t X v_1 \in \mathbb{R}^n,\qquad
      s(t) := \Vec{\sigma}(z(t)) \in \mathbb{R}^n,
\end{align*}
where $\Vec{\sigma}$ is defined in \eqref{eq:sigmoidvec}.
Furthermore, we define
\begin{align*}
    \omega(t) = \nabla U(\theta+tv_1)
    = \lambda (\theta + t v_1) + X^\top \left( \Vec{\sigma}(X(\theta + t v_1)) - y \right)=  \lambda(\theta + t v_1) + X^\top (s(t) - y).
\end{align*}
Now we expand $\omega(t)$ in the Taylor series for $t=0$ up to a 3rd-degree polynomial to approximate the integrals in the sampling process.
\begin{align}\label{Taylor:eqn}
    \omega(t)= \omega(0)+\omega'(0)t+\omega''(0)\frac{t^2}{2}+\omega'''(0)\frac{t^3}{6} + \mathcal{O}(t^4).
\end{align}
The next steps are the computation of the derivatives. First, the constant term in \eqref{Taylor:eqn} is given by
\begin{align}\label{omega:0}
    \omega(0) =\lambda \theta + X^\top (\Vec{\sigma}(X\theta) - y).
\end{align}
We can compute that the first derivative is given by
\begin{align*}
    \omega'(t)& = \lambda v_1 + X^\top \left( \frac{ds(t)}{dt}  \right).
\end{align*}
Moreover, 
\begin{equation}
    \frac{ds(t)}{dt}  = \Vec{\sigma}(z(t)) \odot (1 - \Vec{\sigma}(z(t))) \odot (Xv_1)=s(t) \odot (1- s(t))\odot (Xv_1),
    \label{eq:sigmoid1der}
\end{equation}
which implies 
\begin{align}
\omega'(t) = \lambda v_1 + X^{\top}\Big[s(t) \odot (1- s(t))\odot (Xv_1)\Big],\nonumber
\end{align}
and in particular,
\begin{align}
\omega'(0)=\lambda v_1 + X^{\top}\Big[s \odot (1- s)\odot (Xv_1)\Big].\label{omega:prime:0}
\end{align}
We can compute that the second derivative is given by
\begin{align}\label{omega:2}
    \omega''(t)= X^\top \frac{d}{dt}\left[ \Vec{\sigma}(z(t))(1 - \Vec{\sigma}(z(t))) \odot (X v_1) \right].
\end{align}
Since $s(t)=\Vec{\sigma}(z(t))$, we can compute that $s'(t)=s(t)(1-s(t))(Xv_1)$. Then we have
\begin{align}
    \frac{d}{dt}[s(t)(1-s(t))] = s'(t)(1-s(t))-s(t)s'(t) = s(t)(1-s(t))(1-2s(t))(Xv_1).\label{plug:1}
\end{align}
Plugging \eqref{plug:1} into \eqref{omega:2}, we obtain,
\begin{align*}
    \omega''(t)= X^{\top}\Big[s(t)\odot(1-s(t))\odot(1-2s(t))\odot (Xv_1)\odot (Xv_1)\Big],
\end{align*}
which implies
\begin{align}\label{omega:double:prime:0}
\omega''(0) = X^{\top}\Big[s\odot(1-s)\odot(1-2s)\odot (Xv_1)\odot (Xv_1)\Big].
\end{align}
We can compute the third derivative is given by
\begin{align*}
    \omega'''(t)& = X^\top \frac{d}{dt}\left[ s(t)(1 - s(t))(1 - 2s(t)) \odot (Xv_1)\odot (Xv_1) \right].
\end{align*}
Using the results in equation (\ref{eq:sigmoid1der}) and equation (\ref{plug:1}), we obtain:
\begin{align*}
    &\frac{d}{dt}\left[ s(t)(1 - s(t))(1 - 2s(t))\right]
    \\
    & = \left(s'(t)(1-s(t))(1-2s(t))+s(t)\frac{d}{dt}\big[(1 - s(t))(1 - 2s(t))\big]\right)(Xv_1)\\
    & = \left(s(t)(1-s(t))\left(1-6s(t)+6[s(t)]^2\right)\right)(Xv_1),
\end{align*}
which implies
\begin{align}
\label{omega:triple:prime:0}
    \omega'''(0) = X^{\top}\Big[s\odot(1-s)\odot(1-6s+6s^2)\odot (Xv_1)\odot (Xv_1)\odot (Xv_1)\Big].
\end{align}
Substituting \eqref{omega:0}, \eqref{omega:prime:0}, \eqref{omega:double:prime:0} and \eqref{omega:triple:prime:0} into \eqref{Taylor:eqn}, we get the Taylor expansion of the gradient function,
\begin{align*}
    \omega(t)=& \lambda \theta + X^\top (s - y) \\
    &+ \left[ \lambda v_1 + X^\top \left( s \odot (1 - s) \odot (Xv_1) \right) \right]t \\
    &+\Bigg[X^\top \Big( s \odot (1 - s) \odot (1 - 2s) \odot (Xv_1)\odot (Xv_1) \Big)\Bigg]\frac{t^2}{2}  \\
    &+\Bigg[X^\top \Big( s \odot (1 - s) \odot (1 - 6s + 6s^2) \odot (Xv_1)\odot (Xv_1)\odot (Xv_1) \Big)\Bigg]\frac{t^3}{6}+O(t^4),
\end{align*}
where $s=\sigma(X\theta)\in \mathbb{R}^n$ and $\odot$ is the elementwise (Hadamard) product. We can rewrite $\omega (t)$ as 
\begin{align*}
    \omega(t) = \nabla U(\theta+tv_1)
    &=\lambda\theta+M_0+(M_1\odot (Xv_1)+\lambda v_1)t +\frac{1}{2}(M_2\odot Xv_1\odot Xv_1)t^2\\
    &\qquad\qquad+\frac{1}{6}\Big(M_3\odot Xv_1\odot Xv_1\odot Xv_1\Big)t^3,
\end{align*}
where 
\begin{align*}
    &M_0 = X^{\top} (s-y),\qquad
    M_1 =X^\top \left( s \odot (1 - s) \right),\\
    &M_2  = X^\top \Big( s \odot (1 - s) \odot (1 - 2s) \Big),\qquad
    M_3  = X^\top \Big( s \odot (1 - s) \odot (1 - 6s + 6s^2)  \Big),
\end{align*}
and all $\{M_i\}_{i=0}^3 \in \mathbb{R}^d$ and $s=\Vec{\sigma}(X\theta)\in \mathbb{R}^n$. \\


\textbf{Fourth-order computations:} Once we have the Taylor expanded form of $\nabla U(\theta+tv_1)$ (i.e., $\omega(t)$), the calculation processes are the same as the quadratic function. We use standard mathematical software \textit{Mathematica 12.0} to compute those nested integrals and obtain the following results.

\begin{align*}
    m_0 & =  \left(\frac{\gamma ^2 \eta ^4 \lambda}{24}  +\frac{\eta ^4 \lambda ^2}{24}-\frac{\eta ^2 \lambda }{2}+\mu_{00}\right)\theta+\left(\frac{\gamma ^2 \eta ^5 \lambda}{60}  +\frac{\eta ^5 \lambda ^2}{120}-\frac{\eta ^3 \lambda }{6}+\mu _{01}\right)v_1\\
    &+ \left(\mu _{02}-\frac{\gamma  \eta ^4 \lambda}{24}  \right)v_2+ \left(\mu _{03}-\frac{\gamma ^2 \eta ^5 \lambda }{120} \right)v_3+\left(\frac{\gamma ^2 \eta ^4}{24}+\frac{\eta ^4 \lambda }{24}-\frac{\eta ^2}{2}\right)M_0 \\
    &+\left(\frac{\gamma ^2 \eta ^5}{120}+\frac{\eta ^5 \lambda }{120}-\frac{\eta ^3}{6}\right)M_1 \odot X v_1+\left(\frac{\gamma ^2 \eta ^6}{720}+\frac{\eta ^6 \lambda }{720}-\frac{\eta ^4}{24}\right)M_2\odot Xv_1\odot Xv_1\\
    &+\left(\frac{\gamma ^2 \eta ^7}{5040}+\frac{\eta ^7 \lambda }{5040}-\frac{\eta ^5}{120}\right)M_3 \odot Xv_1\odot Xv_1\odot Xv_1,
\end{align*}
\begin{align*}
    m_1 &=  \left(\frac{\gamma ^2 \eta ^3 \lambda}{6}  +\frac{\eta ^3 \lambda ^2}{6}-\eta  \lambda \right)\theta+\left(\frac{\gamma ^2 \eta ^4 \lambda}{12}  +\frac{\eta ^4 \lambda ^2}{24}-\frac{\eta ^2 \lambda }{2}+\mu _{11}\right)v_1 \\
    &+ \left(\mu _{12}-\frac{\gamma  \eta ^3 \lambda}{6}  \right)v_2+ \left(\mu _{13}-\frac{\gamma ^2 \eta ^4 \lambda}{24}  \right)v_3+\left(\frac{\gamma ^2 \eta ^3}{6}+\frac{\eta ^3 \lambda }{6}-\eta \right)M_0\\
    & + \left(\frac{\gamma ^2 \eta ^4}{24}+\frac{\eta ^4 \lambda }{24}-\frac{\eta ^2}{2}\right)M_1 \odot Xv_1+ \left(\frac{\gamma ^2 \eta ^5}{120}+\frac{\eta ^5 \lambda }{120}-\frac{\eta ^3}{6}\right)M_2\odot X v_1\odot Xv_1\\
    &+ \left(\frac{\gamma ^2 \eta ^6}{720}+\frac{\eta ^6 \lambda }{720}-\frac{\eta ^4}{24}\right)M_3\odot Xv_1\odot Xv_1\odot Xv_1,
\end{align*}
\begin{align*}
    m_2&=\left(-\frac{\gamma  \eta ^2 \lambda }{24}  \left(\gamma ^2 \eta ^2+\eta ^2 \lambda -12\right)+\frac{\lambda  \left(-\gamma ^3 \eta ^3+3 \gamma ^2 \eta ^2-6 \gamma  \eta -6 e^{-\gamma  \eta }+6\right)}{6 \gamma } \right)\theta\\
    &+\left(\mu _{21}-\frac{\gamma  \eta ^3 \lambda\left(2 \gamma ^2 \eta ^2+\eta ^2 \lambda -20\right)}{120}   +\frac{\lambda  \left(-\gamma ^4 \eta ^4+4 \gamma ^3 \eta ^3-12 \gamma ^2 \eta ^2+24 \gamma  \eta +24 e^{-\gamma  \eta }-24\right)}{24 \gamma ^2}\right)v_1\\
    &+\left(\frac{\gamma ^2 \eta ^4 \lambda}{24}  +\mu _{22}\right)v_2+\left(\frac{\gamma ^3 \eta ^5 \lambda}{120}  +\mu _{23}\right)v_3 \\
    &+ \left(-\frac{\gamma  \eta ^2 \left(\gamma ^2 \eta ^2+\eta ^2 \lambda -12\right)}{24} -\frac{\gamma ^3 \eta ^3-3 \gamma ^2 \eta ^2+6 \gamma  \eta +6 e^{-\gamma  \eta }-6}{6 \gamma }\right)M_0\\
    & + \left(\frac{\gamma ^4 \eta ^4-4 \gamma ^3 \eta ^3+12 \gamma ^2 \eta ^2-24 \gamma  \eta -24 e^{-\gamma  \eta }+24}{24 \gamma ^2}-\frac{\gamma  \eta ^3 \left(\gamma ^2 \eta ^2+\eta ^2 \lambda -20\right)}{120} \right)M_1\odot Xv_1\\
    &+\left(-\frac{\gamma  \eta ^4 \left(\gamma ^2 \eta ^2+\eta ^2 \lambda -30\right) }{720} +\frac{1-e^{-\gamma  \eta }}{\gamma ^3}-\frac{1}{120} \gamma ^2 \eta ^5-\frac{\eta }{\gamma ^2}+\frac{\gamma  \eta ^4}{24}+\frac{\eta ^2}{2 \gamma }-\frac{\eta ^3}{6}\right)M_2\odot Xv_1 \odot Xv_1\\
    &+\Bigg(-\frac{\gamma  \eta ^5 \left(\gamma ^2 \eta ^2+\eta ^2 \lambda -42\right)}{5040}\\
    &-\frac{\gamma ^6 \eta ^6-6 \gamma ^5 \eta ^5+30 \gamma ^4 \eta ^4-120 \gamma ^3 \eta ^3+360 \gamma ^2 \eta ^2-720 \gamma  \eta -720 e^{-\gamma  \eta }+720}{720 \gamma ^4}\Bigg)
    \\
    &\quad\qquad\qquad\qquad M_3\odot Xv_1\odot Xv_1\odot Xv_1,
\end{align*}
\begin{align*}
    m_3 & = \Bigg(\frac{\lambda  e^{-\gamma  \eta }}{24 \gamma ^3} \Big(\gamma ^6 \eta ^4 e^{\gamma  \eta }+\gamma ^4 \eta ^2 e^{\gamma  \eta } \left(\eta ^2 \lambda-24\right)-4 \gamma ^3 \eta  \left(e^{\gamma  \eta } \left(\eta ^2 \lambda -18\right)-6\right)\\
    &+12 \gamma ^2 \left(e^{\gamma  \eta } \left(\eta ^2 \lambda -8\right)+8\right)-24 \gamma  \eta  \lambda  e^{\gamma  \eta }+24 \lambda  \left(e^{\gamma  \eta }-1\right)\Big)\Bigg)\theta\\
    & + \Bigg(\mu_{31}+\frac{\lambda  e^{-\gamma  \eta }}{120 \gamma ^4}\Big(2 \gamma ^7 \eta ^5 e^{\gamma  \eta }-5 \gamma ^6 \eta ^4 e^{\gamma  \eta }+\gamma ^5 \eta ^3 e^{\gamma  \eta } \left(\eta ^2 \lambda -20\right)-5 \gamma ^4 \eta ^2 e^{\gamma  \eta } \left(\eta ^2 \lambda -24\right)\\
    &+20 \gamma ^3 \eta  \left(e^{\gamma  \eta } \left(\eta ^2 \lambda -18\right)-6\right)-60 \gamma ^2 \left(e^{\gamma  \eta } \left(\eta ^2 \lambda -8\right)+8\right)+120 \gamma  \eta  \lambda  e^{\gamma  \eta }\\
    &-120 \lambda  \left(e^{\gamma  \eta }-1\right)\Big)\Bigg)v_1\\
    & + \Bigg(\mu_{32}+\frac{\lambda  \left(-\gamma ^4 \eta ^4+4 \gamma ^3 \eta ^3-12 \gamma ^2 \eta ^2+24 \gamma  \eta +24 e^{-\gamma  \eta }-24\right)}{24 \gamma ^2}\Bigg)v_2\\
    &+ \Bigg(\mu_{33}+\frac{\lambda  \left(-\gamma ^5 \eta ^5+5 \gamma ^4 \eta ^4-20 \gamma ^3 \eta ^3+60 \gamma ^2 \eta ^2-120 \gamma  \eta -120 e^{-\gamma  \eta }+120\right)}{120 \gamma ^2}\Bigg)v_3\\
    & + \Bigg(\frac{\gamma ^3 \eta ^4}{24}+\frac{\lambda -\lambda  e^{-\gamma  \eta }}{\gamma ^3}-\frac{\eta  \lambda }{\gamma ^2}+\frac{4 e^{-\gamma  \eta }+\frac{\eta ^2 \lambda }{2}-4}{\gamma }+\gamma  \left(\frac{\eta ^4 \lambda }{24}-\eta ^2\right)+\eta  \left(e^{-\gamma  \eta }+3\right)-\frac{\eta ^3 \lambda }{6}\Bigg)M_0\\
    &+ \Bigg(\frac{\lambda  \left(e^{-\gamma  \eta }-1\right)}{\gamma ^4}+\frac{\gamma ^3 \eta ^5}{120}+\frac{\eta  \lambda }{\gamma ^3}+\frac{-5 e^{-\gamma  \eta }-\frac{\eta ^2 \lambda }{2}+5}{\gamma ^2}+\frac{\eta  \left(-e^{-\gamma  \eta }-4\right)+\frac{\eta ^3 \lambda }{6}}{\gamma }\\
    &+\frac{1}{120} \gamma  \eta ^3 \left(\eta ^2 \lambda -40\right)-\frac{1}{24} \eta ^2 \left(\eta ^2 \lambda -36\right)\Bigg)M_1\odot Xv_1\\
    & + \Bigg(\frac{\lambda -\lambda  e^{-\gamma  \eta }}{\gamma ^5}-\frac{\eta  \lambda }{\gamma ^4}+\frac{\gamma ^3 \eta ^6}{720}+\frac{6 e^{-\gamma  \eta }+\frac{\eta ^2 \lambda }{2}-6}{\gamma ^3}+\frac{\eta  \left(e^{-\gamma  \eta }+5\right)-\frac{\eta ^3 \lambda }{6}}{\gamma ^2}\\
    &+\frac{1}{720} \gamma  \eta ^4 \left(\eta ^2 \lambda -60\right)+\frac{\frac{\eta ^4 \lambda }{24}-2 \eta ^2}{\gamma }-\frac{1}{120} \eta ^3 \left(\eta ^2 \lambda -60\right)\Bigg)M_2 \odot Xv_1\odot Xv_1\\
    &+ \Bigg(\frac{\lambda  \left(e^{-\gamma  \eta }-1\right)}{\gamma ^6}+\frac{\eta  \lambda }{\gamma ^5}+\frac{-7 e^{-\gamma  \eta }-\frac{\eta ^2 \lambda }{2}+7}{\gamma ^4}+\frac{\gamma ^3 \eta ^7}{5040}+\frac{\eta  \left(-e^{-\gamma  \eta }-6\right)+\frac{\eta ^3 \lambda }{6}}{\gamma ^3}
    .+\frac{60 \eta ^2-\eta ^4 \lambda }{24 \gamma ^2}\\
    &+\frac{\gamma  \eta ^5 \left(\eta ^2 \lambda -84\right)}{5040}+\frac{\eta ^3 \left(\eta ^2 \lambda -80\right)}{120 \gamma }-\frac{1}{720} \eta ^4 \left(\eta ^2 \lambda -90\right)\Bigg)M_3\odot Xv_1\odot Xv_1\odot Xv_1.
\end{align*}
\begin{lemma}
    \label{dis:grad-log}
    Define the potential (loss) function with a penalty term (i.e., $L_2$ or Ridge regularization):
\begin{align*}
    U (\theta) &= -\sum_{i=0}^n\left[y_i\log (\hat{y}_i)+ (1-y_i)\log (1-\hat{y}_i)\right]  + \frac{\lambda}{2}|\theta|^2\\
    & = -\sum_{i=0}^n\left[y_i\log (\sigma(z_i))+ (1-y_i)\log (1-\sigma(z_i))\right]  + \frac{\lambda}{2}|\theta|^2.
\end{align*}
Then the gradient of the regularized loss function is given as
\begin{align*}
   \nabla U(\theta) = X^\top(\hat{y}-y)+\lambda\theta= X^\top(\Vec{\sigma}(X\theta)-y) +\lambda \theta.
\end{align*}
\begin{proof}
    To make the calculation easier, we consider the non-regularized elementwise gradient of the loss. Additionally, we use the recursive property of the \textit{sigmoid} function $\frac{d\sigma(z_i)}{d\theta}=\sigma(z_i)(1-\sigma(z_i))X_i \in \mathbb{R}^d$. Thus,
\begin{align*}
    \nabla U(\theta) 
    &= -\sum_{i=1}^{n} \left[ y_i  \nabla_\theta \log(\sigma(z_i)) + (1 - y_i) \nabla_\theta \log(1 - \sigma(z_i)) \right]\\
    &= -\sum_{i=1}^{n} \left[ y_i \frac{1}{\sigma(z_i)}  \frac{d}{d\theta} \sigma(z_i) - (1 - y_i) \cdot \frac{1}{\sigma(z_i)} \frac{d}{d\theta} \sigma(z_i) \right]\\
    &=-\sum_{i=1}^{n} \left[ y_i  \frac{1}{\sigma(z_i)} \sigma(z_i)(1 - \sigma(z_i)) X_i - (1 - y_i) \frac{1}{\sigma(z_i)}  \sigma(z_i)(1 - \sigma(z_i)) X_i \right]\\
    &=-\sum_{i=1}^{n} \left[ y_i  (1 - \sigma(z_i)) X_i - (1 - y_i) \cdot (1 - \sigma(z_i)) X_i \right]\\
    &=\sum_{i=1}^{n} \left[ -y_i  (1 - \sigma(z_i)) X_i + (1 - y_i) \cdot (1 - \sigma(z_i)) X_i \right]\\
    &=\sum_{i=1}^{n} \left[ (\sigma(z_i) - y_i) X_i \right]\\
    & = (\sigma(z_1)-y_1)X_1+(\sigma(z_2)-y_2)X_2+\cdots +(\sigma(z_n)-y_n)X_n.
\end{align*}
Therefore, the gradient of the regularized loss function in vector form would be 
\begin{align*}
   \nabla U(\theta) & = X^\top(\hat{y}-y)+\lambda\theta= X^\top(\Vec{\sigma}(X\theta)-y) +\lambda \theta.
\end{align*}
\end{proof}
\end{lemma}

\end{document}